\documentclass[11pt]{article} 
\usepackage{latexsym,amssymb,enumerate,amsmath,epsfig,amsthm,dsfont}
\usepackage{soul}
\usepackage{url}
\usepackage{smile}
\usepackage{booktabs}

\usepackage{algorithm}
\usepackage{algorithmic}
\usepackage{todonotes}
\usepackage{epstopdf}
\usepackage{wrapfig}

\usepackage[hyperfootnotes=false]{hyperref}
\usepackage[margin=1in]{geometry}
\usepackage[normalem]{ulem}
\usepackage[export]{adjustbox}
\usepackage{mathtools, cuted}
\usepackage{mathrsfs}
\usepackage{natbib}
\usepackage{enumerate}
\usepackage{lineno}
\setcitestyle{citesep={;},square}

\newtheorem{setting}{Setting}

\newtheorem{theorem}{Theorem}
\newtheorem{lemma}{Lemma}

\newtheorem{definition}{Definition}
\newtheorem{assumption}{Assumption}
\newtheorem{remark}{Remark}

\newcommand{\rhox}{\rho_{\xb}}
\newcommand{\rhof}{\rho_f}
\newcommand{\fs}{\mathfrak{s}}
\newcommand{\fS}{\mathfrak{S}}

\newcommand{\rmT}{{{\rm T}}}

\newcommand{\Rad}{{{\rm Rad}}}

\allowdisplaybreaks

\newcommand{\commentout}[1]{}

\date{\today}
	
	\title{Transformers as Adaptive Kernel Learners across Tasks for In-Context Learning}
\author{
Zhaiming Shen\thanks{zhaiming.shen@usu.edu.
Department of Mathematics and Statistics, Utah State University, Logan, UT 84322}
\and
Alexander Cloninger\thanks{acloninger@ucsd.edu. Department of Mathematics and Halicio{\u g}lu Data Science Institute, University of California, San Diego, La Jolla, CA 92093}
\and 
Rongjie Lai\thanks{lairj@purdue.edu. Department of Mathematics, Purdue University, West Lafayette, IN 47907} 
\and
Wenjing Liao\thanks{wliao60@gatech.edu.
School of Mathematics,
Georgia Institute of Technology, Atlanta, GA 30332}
}

\begin{document}

	\maketitle

\begin{abstract}
\noindent
Transformers have emerged as a dominant architecture in modern machine learning, with in-context learning (ICL) enabling models to adapt to new tasks by conditioning on prompts, without any parameter updates. To understand the attention mechanism in transformer blocks, a growing line of theoretical work interprets ICL through the lens of classical statistical estimators.
Recent works have established a connection between the attention mechanism and kernel methods, while the kernel function is usually assumed to be fixed. In practice, the query and key matrices in the attention mechanism are trained, which allows the attention mechanism to serve as an adaptive kernel method.
In this work, we demonstrate that transformers can act as adaptive  anisotropic kernel learners. Specifically, we study ICL for low-rank linear  and multi-index nonlinear regressions, and conduct numerical simulations to show that the learned attention scores closely resemble the oracle kernels for these function classes. We then develop a general framework to analyze the generalization error of ICL by connecting transformers to the oracle kernel estimator for each function class.
Our theory gives an explicit and constructive implementation of the oracle kernel estimators for these function classes using a wide and shallow transformer architecture, illustrating how the prediction is carried out by the interaction between the query token and the context tokens. Our work complements existing studies by providing numerical and theoretical evidence of how the attention mechanism adapts to anisotropic kernels in both linear and nonlinear settings.
\end{abstract}

\section{Introduction}

Introduced by \citet{Vaswani17}, the transformer architecture has become a cornerstone of modern machine learning, powering significant progress in natural language processing, computer vision, and related fields. In contrast to traditional feedforward and convolutional neural networks, transformers are built upon an attention mechanism that enables each token to attend to all others, dynamically weighting their contributions based on learned relevance. This approach yields representations that are both highly adaptive and sensitive to context. Consequently, transformers underpin many state-of-the-art generative models, including GPT \citep{achiam2023gpt}, BERT \citep{devlin2018bert}, SORA \citep{videoworldsimulators2024}, as well as their subsequent developments.

Beyond their empirical success, transformers exhibit a striking form of in-context learning (ICL), also often refer to as few-shot learning, in which the model effectively adapts to new tasks by conditioning on a sequence of example pairs, referred to as a prompt, at inference time, without any explicit parameter updates \citep{brown2020language,radford2019language,liu2023pre,garg2022can}. This emergent capability has prompted a growing line of work aimed at uncovering its underlying principles. Recent theoretical studies \citep{bai2023transformers,zhang2024trained,von2023transformers,akyurek2022learning,cole2024provable,lu2025transformer} demonstrate that transformers can implicitly implement classical estimation procedures, including least squares, ridge regression, Lasso, generalized linear models, and linear inverse problems.  Notably, problems resembling ICL were studied even before the advent of transformers, often referred to as meta-learning or multi-task learning in the literature \citep{baxter2000model,pontil2013excess,maurer2016benefit,tripuraneni2021provable,du2021few}.

Beyond linear regimes, a growing body of work has investigated
the expressive power, approximation capabilities, and statistical
properties of transformers for nonlinear functions and
sequence-to-sequence mappings
\citep{Yun19,Takakura23,TransformerClassifier,
Jiao2026CoD,He2026}.
For instance, \citet{Yun19} established the universal
approximation property of transformers for continuous
sequence-to-sequence functions with compact support, albeit
with network complexity that can grow exponentially with
the input dimension.
Extending this line of research, \citet{Takakura23}
derived approximation and estimation guarantees for
sequence-to-sequence mappings on infinite-dimensional
domains under anisotropic smoothness assumptions, while
\citet{TransformerClassifier} studied the statistical
properties of transformer classifiers under hierarchical
composition and H\"older regularity assumptions.
More recently, \citet{Jiao2026CoD} investigated the ability
of transformers to overcome the curse of dimensionality
in nonlinear function approximation, while
\citet{He2026} established quantitative
approximation bounds for smooth function classes,
further characterizing the approximation efficiency
and limitations of transformer architectures.
Complementary studies
\citep{Havrilla24,shen2025transformers,shi2026learning}
have demonstrated that transformers can exploit
low-dimensional structure, yielding approximation or
generalization guarantees governed by intrinsic rather
than ambient dimensionality.
These works have focused
on approximating or learning a fixed target function or
mapping, rather than characterizing the ability of a
single transformer to learn previously unseen nonlinear
functions from in-context examples.

In contrast, in-context learning (ICL) places transformers in a
fundamentally different regime, where a single model must
implicitly encode and execute a family of learning procedures,
adapting to previously unseen tasks through contextual examples
without updating its parameters.
This paradigm introduces a distinct form of task-level
generalization, requiring the model to transfer its learned
inference mechanism across different target functions rather
than merely generalize to unseen samples from a fixed task.
While substantial theoretical progress has been made in
understanding linear ICL, extending these results to nonlinear
function classes remains an active area of research.
Recent studies have demonstrated that transformers can
implement  
learning algorithms, including functional
gradient descent, kernel regression, and polynomial regression
\citep{li2023transformers,bai2023transformers,Cheng2024,
Sun2026,shen2025understanding,hsu2026understanding}.
Complementary works have investigated nonlinear feature
learning and the training dynamics of transformers
\citep{KimSuzuki2024,Li2024Nonlinear}, while
\citet{Kim2024Minimax,Ching2026,shen2025understanding} established
statistical guarantees for nonparametric ICL over
smooth function classes. Beyond finite-dimensional regression and classification,
ICL has also been investigated in the context of operator learning
\citep{chiu2025context}, further broadening its scope to function-space learning problems. 

Another important yet less explored direction is understanding how shared structure across tasks can be exploited to reduce learning complexity and improve generalization in the task space \citep{lu2025transformer,cole2026incontext,shi2026transformers}.  In inverse linear regression, \citet{lu2025transformer} showed that a linear transformer can learn a cross-task prior and an effective regularization strategy. \citet{cole2026incontext} studied in-context operator learning for optimal transport on probability measure spaces and established generalization bounds that quantify the effect of the intrinsic task dimension. For nonlinear task families, \citet{shi2026transformers} characterized cross-task complexity through covering numbers and constructed transformers that exploit this structure. A generalization error bound is derived in \citet{shi2026transformers} to demonstrate the context-sample efficiency with low-dimensional cross-task structures. Despite of these theoretical progresses, how transformers learn and adapt to shared cross-task structure remains incompletely understood, both theoretically and computationally.

In this work, we investigate how the attention mechanism adapts to shared cross-task structure in linear and nonlinear regression through the lens of adaptive anisotropic kernels.
We demonstrate this kernel perspective through  numerical experiments and theoretical analysis in two representative settings: low-rank linear regression and anisotropic kernel regression with noisy observations. Our main contributions are summarized as follows.

\begin{itemize}
\item We investigate the in-context learning behavior of transformers
    in low-rank linear and multi-index nonlinear regression models. Numerical experiments
    reveal a strong correlation between the attention scores in the
    final layer of trained transformers and the underlying
    ground-truth algorithmic kernels in both settings, providing
    empirical evidence that transformers can implicitly learn
    task-specific prediction algorithms.

    \item Motivated by these empirical observations, we establish the 
    generalization error bounds for transformer-based ICL in both
    settings, demonstrating that transformers can implicitly implement
    low-rank linear regression and anisotropic kernel regression
    respectively.

    \item We provide explicit, step-by-step constructions of
    wide and shallow transformer architectures that realize
    the corresponding in-context learning algorithms in both
    settings. These constructions elucidate how transformers
    leverage contextual examples to perform task-specific
    inference without updating their parameters.
    
\end{itemize}
The rest of paper is structured as follows. In Section \ref{sec:prelim}, we introduce the notations and preliminary definitions. In Section \ref{sec:setup}, we introduce the problem setup for In-Context Learning (ICL) and transformer for ICL. In Section \ref{sec:adaptive_kernel} and \ref{sec:statistical_error}, we illustrate how transformer adapt to the cases such as low-rank linear regression and anisotropic kernel regression, as well as discuss their statistical estimation error for functions in these two classes. We also refer the two regression tasks as low-rank linear model and multi-index model respectively. In Section \ref{sec:LinearFunctions} and Section \ref{sec:multiindex}, we rigorously show the generalization error bound for low-rank linear model and multi-index model respectively. Finally, Section~\ref{sec:conclusion} concludes the paper by discussing its limitations and potential directions for future research.

\section{Preliminary and Notation} \label{sec:prelim}

\begin{definition} [$k$-th order truncated pseudo-inverse] \label{def:pseudo-inverse}
    Suppose $A\in \mathbb{R}^{r \times r}$ has singular value decomposition $A=U\Sigma V^{\top}$ where $\Sigma$ is the diagonal matrix with at least $k$ nonzero singular values $\sigma_1(A)\geq\cdots\geq \sigma_r(A)$. We define the \emph{$k$-th order truncation pseudo-inverse} of $A$, denoted by $[A]_k^{\dagger} \in \mathbb{R}^{r \times r}$, as $[A]_k^{\dagger}:=V[\Sigma]_k^{\dagger}U^{\top}$, where $[\Sigma]_k^{\dagger}$ is the diagonal matrix with singular values $\sigma_1(A)^{-1}, \sigma_{2}(A)^{-1},\cdots, \sigma_k(A)^{-1},0,0,\cdots,0$.
\end{definition}

\begin{definition}[Covering Number] \label{def:covering}

     Let $(\mathcal{F},\rho)$ be a metric space. For a given \( \epsilon > 0 \), the \emph{covering number} \( \mathcal{N}(\epsilon, \mathcal{F}, \rho) \) is the smallest number of balls of radius \( \epsilon \)  needed to cover \( \mathcal{F} \). More precisely,
    \begin{align*}
        &\mathcal{N}(\epsilon, \mathcal{F}, \rho) := \min \{ N \in \mathbb{N}\text{ }|\text{ } \exists \{f_1, f_2, \dots, f_N\} \subseteq \mathcal{F}, \ \forall f \in \mathcal{F} ,\ \exists f_i \text{ such that } \rho(f,f_i) \leq \epsilon \}.
    \end{align*}

\end{definition}

\begin{definition}[Attention and Multi-head Attention]
Attention with the Query, Key, Value matrices $Q,K,V\in\mathbb{R}^{d_{embed}\times d_{embed}}$ is defined as
\begin{equation}
\textstyle 
   {\rm A}_{K,Q,V}(H)=VH\sigma((KH)^\top QH).
\end{equation}
The multi-head attention (MHA) with $m$ heads is the linear superposition of attentions given by
\begin{equation}
  \textstyle  {\rm MHA}(H)=\sum_{j=1}^m V_jH\sigma((K_jH)^\top Q_jH)
\end{equation}
with $\sigma$ be some activation function to be specified.

A transformer block is a residual composition of the
form
\begin{equation}
       {\rm B}(H) = {\rm FFN}({\rm MHA}(H) + H) + {\rm MHA}(H) + H.
\end{equation}
where {\rm FFN} is a feed-forward neural network operating tokenwise on the input.
\end{definition}

\begin{definition}[Embedding Layer and Positional Encoding]
    Given $\xb_i\in\mathbb{R}^D$ and $y_i\in\mathbb{R}$, the \emph{embedding layer} $E$ takes an input $\fs=\{\xb_1,y_1,\xb_2,y_2,\ldots,\xb_n,y_n;\xb_{n+1}\}$ and maps it to
    \begin{equation} 
    {\rm E}\left(\fs \right) =
\begin{bmatrix}
    \xb_{1} & \cdots & \xb_{n} & \xb_{n+1} & \mathbf{0}  \\
    y_1 & \cdots & y_n & 0 & \mathbf{0} \\ 
    0 & \cdots & \cdots & \cdots & 0 \\
    \vdots & \vdots & \vdots & \vdots & \vdots \\
    0 & \cdots & \cdots & \cdots & 0 \\
    1 & \cdots & \cdots & \cdots & 1
\end{bmatrix}\in\mathbb{R}^{d_{embed}\times\ell}.
\end{equation}
The \emph{sinusoidal positional encoding} takes an input $\fs$, maps it to $\mathcal{I}_j=(\cos(\frac{j\pi}{2\ell}),\sin(\frac{j\pi}{2\ell}))^\top$ and put those $\mathcal{I}_j$, $j=1,\cdots,\ell$, into the third and second last row in the embedding matrix, i.e.,
\begin{equation} 
     {\rm PE}\left(\fs \right) =
\begin{bmatrix}
   0 & \cdots & \cdots & \cdots & 0  \\
    \vdots & \vdots & \vdots & \vdots & \vdots \\ 
    0 & \cdots & \cdots & \cdots & 0 \\
    \mathcal{I}_1 & \cdots & \cdots & \cdots & \mathcal{I}_{\ell} \\
    0 & \cdots & \cdots & \cdots & 0
\end{bmatrix}
\in\mathbb{R}^{d_{embed}\times \ell}.
\end{equation}
With these, we have
\begin{equation} 
    {\rm PE}+{\rm E}\left(\fs \right) =
\begin{bmatrix}
    \xb_{1} & \cdots & \xb_{n} & \xb_{n+1} & \mathbf{0}  \\
    y_1 & \cdots & y_n & 0 & \mathbf{0} \\ 
    0 & \cdots & \cdots & \cdots & 0 \\
    \vdots & \vdots & \vdots & \vdots & \vdots \\ 
    0 & \cdots & \cdots & \cdots & 0 \\
    \mathcal{I}_1 & \cdots & \cdots & \cdots & \mathcal{I}_{\ell} \\
    1 & \cdots & \cdots & \cdots & 1
\end{bmatrix}
\in\mathbb{R}^{d_{embed}\times \ell}.
\end{equation}
\end{definition}

{\bf Notation.}  We use boldface letters to indicate vectors, and use regular letters to indicate scalar and matrices. Define function $L^\infty$ norm. $\|\theta\|_\infty$ for neural network parameter. Throughout the discussion, we use $\|A\|_{\infty,\infty}$ to denote the largest entry-wise absolute value in matrix $A$. The remaining notation used throughout the paper is summarized in Table~\ref{tab:notation}.

\section{Problem Setup}
\label{sec:setup}

\subsection{In-Context Learning (ICL)}

Let $\xb \in \RR^D$ be a random sample from a distribution $\rhox$ supported on a bounded domain $\Omega\subset \RR^D$. A function  $f:\Omega \to\RR$ maps a point on $\Omega$ to a real number. Suppose $\rho_f$ is a probability distribution in the target function space $\{f:\ \Omega\to\RR\}$ and $\rho_f$ is supported on a function class $\cF \subset \{f:\ \Omega\to\RR\}$.

\begin{setting}
\label{setting:icl}
We consider in-context regression where a prompt/task is given as
\begin{equation}
\fs=\{\xb_1,y_1,\xb_2,y_2,\ldots,\xb_n,y_n;\xb_{n+1},y_{n+1}\}  \ \text{ with } y_i=f(\xb_i)+\epsilon_i,
\label{eq:promp}\end{equation}
where $\xb_i$'s are i.i.d. samples from $\rho_{\xb}$, $f$ is sampled from $\rho_f$, and   $\epsilon_i$'s are i.i.d. Gaussian noise: $\epsilon_i\sim N(0,\sigma^2)$. The goal is to predict $f(\xb_{n+1})$ from $\{\xb_1,y_1,\xb_2,y_2,\ldots,\xb_n,y_n;\xb_{n+1}\}$. The training data contains $\Gamma$ tasks: $\fS:=\{\fs^\gamma\}_{\gamma=1}^\Gamma$ where
\[\fs^\gamma=\{\xb_1^\gamma,y^\gamma_1,\ldots,\xb_n^\gamma,y_n^\gamma;\xb_{n+1}^\gamma,y_{n+1}^\gamma\}  \ \text{ with } y_i^\gamma=f^\gamma(\xb^\gamma_i)+\epsilon^{\gamma}_i,\]
where  $\{\xb_i^\gamma\}_{i=1,\ldots,n+1}^{\gamma=1,\ldots,\Gamma}$ are i.i.d. samples from the $\rhox$, $\{f^\gamma\}_{\gamma=1,\ldots,\Gamma}$ are i.i.d. samples from the $\rhof$, and $\{\epsilon^{\gamma}_i\}_{i=1,\ldots,n+1}^{\gamma =1,\ldots,\Gamma}$ are i.i.d. Gaussian noise.

\end{setting}

In Setting \ref{setting:icl},  $n$ is the prompt length and $\Gamma$ is the training sample/task size.
For a random sample $\fs$ in \eqref{eq:promp} and an in-context regression algorithm denoted by $\rmT$, we define the  the risk as the squared prediction error: $\cR_{\fs}(\rmT):= \left({\rmT}(\{\xb_i,y_i\}_{i=1}^n\};\xb_{n+1}) - y_{n+1}\right)^2$.
 Given the training set $\fS$, we minimize the following empirical risk:
\begin{align}
    \hat{\cR}_{\fS}({\rmT}_\theta) :=  \frac{1}{\Gamma}\sum_{\gamma=1}^\Gamma \cR_{\fs^\gamma}(\rmT_\theta)&=\frac{1}{\Gamma}\sum_{\gamma=1}^\Gamma  \Big({\rmT}_\theta(\{\xb^\gamma_i,y^\gamma_i\}_{i=1}^n\};\xb^\gamma_{n+1}) - y^\gamma_{n+1}\Big)^2    \label{eq:RnT}
\end{align}
where ${\rmT}_\theta$ is a transformer neural network  parameterized by $\theta$. The global minimizer is denoted by  $\hat{\rmT}$ such that
\begin{equation}
\hat{\rmT} \in \argmin_{{\rmT}_\theta \in \cT} \hat\cR_{\fS}({\rmT}_\theta),
\label{eq:hatT}
\end{equation}
where $\cT$ is a transformer neural network class with the architecture to be specified.

The squared generalization error of $\hat{\rmT}$ on a random test sample $\fs_f$ (independent of training data) 
\begin{equation} \label{prompt:sf}
    \fs_f=\{\xb_1,y_1,\ldots,\xb_n,y_n;\xb_{n+1},f(\xb_{n+1})\},
\end{equation}
is $\cR_{\fs_f}(\hat{\rmT})= \left(\hat{\rmT}(\{\xb_i,y_i\}_{i=1}^n\};\xb_{n+1}) - f(\xb_{n+1})\right)^2$.
The mean squared generalization of $\hat{\rmT}$ is
\begin{equation}
\cR(\hat{\rmT}) = \EE_{\fS} \EE_{\fs_f} \left[\cR_{\fs_f}(\hat{\rmT})\right]
\label{eq:mse1}
\end{equation}
where  the expectation $\EE_{\fs_f}$ is taken for the test sample $\fs_f$ and the expectation $\EE_{\fS}$ is  taken for the joint distribution of the training samples.

\subsection{Transformer Network Class for In-Context Learning (ICL)}
After describing the ICL problem above,  we next specify the transformer network class considered in this paper. For an input  $\fs$ defined in (\ref{eq:promp}), we define a transformer neural network ${\rm T}_{\theta}(\fs)$ with weights parametrized by $\theta$ as
    \begin{equation} \label{Trep}
        {\rm T}_{\theta}(\fs):={\rm DE}\circ {\rm B}_{L_{\rm T}}\circ\cdots\circ {\rm B}_1\circ ({\rm PE}+{\rm E}(\fs)),
    \end{equation}
    where ${\rm E}$ is an embedding layer, ${\rm PE}$ is a positional encoding module, ${\rm B}_1,\cdots,{\rm B}_{L_{\rm T}}$ is a sequence of transformer blocks, and ${\rm DE}$ is a decoding layer which outputs the desired element. 

    We next define the transformer network class $\mathcal{T}(d_{embed},\ell,L_{\rm TF}, m_{\rm TF},L_{\rm FFN},w_{\rm FFN},R,\kappa)$ for our ICL problem. When there is no ambiguity throughout the discussion, we will shorten the notation as $\cT$.

\paragraph{Transformer Network Class.} The transformer network class with weights parametrized by $\theta$ is given by
    \begin{align}
    \nonumber
        &\mathcal{T}(d_{embed},\ell,L_{\rm T}, m_{\rm T},L_{\rm FFN},w_{\rm FFN},R,\kappa)
        \\
        \nonumber
        &=\Big\{{\rm T}_{\theta}(\cdot) \text{ }| \text{ } {\rm T}_{\theta}(\cdot)  \text{ in form }(\ref{Trep}) \text{ with embedded dimension } d_{embed},  \text{number of hidden tokens } \ell, \text{number} \\ 
        \nonumber
        & \quad\quad\quad\quad\quad 
        \text{ of transformer blocks } L_{\rm T}, \text{at most } 
        m_{\rm T} \text{ attention heads in each block}, \text{ and } L_{\rm FFN} \text{ layers}  \\  
        \nonumber
        &\quad\quad\quad\quad\quad
        \text{ of feed-forward components with}
        \text{ hidden width } w_{\rm FFN}, \text{output }    \|{\rm T}_{\theta}(\cdot)\|_{L^{\infty}(\Omega)}\leq R \text{ and}
         \\    &\quad\quad\quad\quad\quad \text{ weights magnitude } \|\theta\|_{\infty}\leq \kappa\Big\}.
         \label{def:TNC}
    \end{align} 
Given an input $\fs$, the transformer first apply ${\rm PE}+{\rm E}(\fs)$ to get an embedding matrix
    \begin{align} \label{eq:Hmatrix}
    H =  {\rm PE}+{\rm E}\left(\fs \right) = 
\begin{bmatrix}
    \xb_{1} & \cdots & \xb_{n} & \xb_{n+1}  \\
    y_1 & \cdots & y_n & 0 \\ 
    0 & \cdots & \cdots & 0 \\
    \vdots & \vdots & \vdots & \vdots & \\
    0 & \cdots & \cdots & 0 \\
    \mathcal{I}_1 & \cdots & \cdots & \cI_{n+1} \\
    1 & \cdots & \cdots & 1
\end{bmatrix}\in\mathbb{R}^{d_{embed}\times\ell}=\mathbb{R}^{d_{embed}\times(n+1)}.
\end{align}
\begin{remark}
   Since Gaussian noise is unbounded, the input embedding matrix $H$ and the transformer output ${\rm T}_{\theta}(\cdot)$ may be unbounded as well. We therefore truncate the network values so that both remain bounded with high probability. Without loss of generality, we assume that our subsequent analysis is carried out under this truncation for both the input and output.
\end{remark}
In \eqref{eq:Hmatrix}, we fix $\ell=n+1$, and we regard each column as a token, whose embedding dimension is $d_{embed}$. We use $H$
as the input for the first layer of transformer ${\rm B}_1$, then successively as input for ${\rm B}_2$, ${\rm B}_3$, $\cdots$, etc,.   In section \ref{sec:LinearFunctions} and section \ref{sec:multiindex}, we will show step-by-step how to construct a specific transformer so it can implement desirable output.

Our goal next is to illustrate via both mathematical formulation and numerical experiments how transformer can adaptively learn anisotropic kernel functions by taking $H$ as an input for the cases of low-rank linear regression and anisotropic kernel regression with noisy responses.

\section{Transformer as an Adaptive Kernel Learner} 
\label{sec:adaptive_kernel}

The attention mechanism in transformers bears a strong resemblance to kernel methods. Both rely on pairwise interactions to adaptively aggregate information across inputs: kernel methods weight inputs through a kernel function, while transformers do so through attention scores. Motivated by this connection, we study transformer based ICL through the lens of kernel methods. 

Given a random  sample $\fs$ in \eqref{eq:promp}, 
the empirical oracle kernel estimator in its most general form is 
\begin{equation} \label{eq:Kgeneral}
    \cK(\fs): = {\sum_{i=1}^n 
    K(\xb_{n+1},\xb_i)y_i}
\end{equation}
where $ K(\xb_{n+1},\xb_i)$ represents an interaction kernel between the query $\xb_{n+1}$ and the prompt point $\xb_i$. 
In practice, the interaction kernel can take different forms depending on the function class $\cF$.

In this paper, we establish a connection between transformers and adaptive kernel methods and provide a general framework for analyzing the generalization error of transformer-based ICL. For a function class $\cF$ of interest, if an oracle kernel estimator $\cK$ exists and can be well approximated by a transformer neural network, our framework enables the development of corresponding generalization guarantees. In particular, we summarize several classical function classes and their corresponding oracle interaction kernels in Table \ref{tab:Kernelsummary}. The notation $[\cdot]^{\dagger}$ denotes matrix pseudo-inverse, which coincides with the inverse whenever the matrix is invertible.

For these classical function classes, we can construct a transformer $\rmT^*$ to approximate the associated oracle kernel estimator $\cK$. Specifically, we focus on illustrating the cases for low-rank linear model and multi-index model in Section \ref{subsec:low_rank_linear} and Section \ref{subsec:multi-index} respectively. Then we analyze the generalization error of transformer-based ICL via a bias–variance decomposition in Section \ref{sec:statistical_error}, and discuss the statistical error in Sections \ref{subsec:Stats} and \ref{subsec:StatsNoise}.

\begin{table*}[t]
\centering
\caption{Function classes and their oracle interaction kernels (matrix pseudo-inverse $[\cdot]^{\dagger}$ is applied when the matrix is not invertible). } \label{tab:Kernelsummary}
\vspace{2mm}
\begin{tabular}{l|l}
\toprule
Function classes $\cF$ & Interaction kernels $K(\xb_{n+1},\xb_i)$  \cr
\midrule
$\cF_{{\rm Linear},W}$ \eqref{eq:funlinearW} & \small{$(W^\top \xb_{n+1})^\top (W^\top X^\top X W)^{\dagger} (W^\top \xb_i)$} \eqref{eq:linear_low_rank_kernel} \cr
\midrule
$\cF_{{\rm Linear}}$ \eqref{eq:funlinear} (special case of $\cF_{{\rm Linear},W}$ with $W=I$) & \small{$\xb_{n+1}^\top (X^{\top}X)^{\dagger}\xb_{i}$} \eqref{eq:linear_kernel} \cr
\midrule
$\cF_{{\rm Linear},\phi,W}$ \eqref{eq:funlinearW_phi} (general case of $\cF_{{\rm Linear},W}$, studied  \\ \hspace{20mm} in \citet{hsu2026understanding} with polynomial  & \small{$(W^\top \phi(\xb_{n+1}))^\top (W^\top \Phi^\top \Phi W)^{\dagger} (W^\top \phi(\xb_i))$} \eqref{eq:linear_feature_kernel} \cr
\hspace{20mm} and spline feature maps $\phi$) & \cr
\midrule
$\cF_{{\rm MI},U}$ \eqref{multi-index} & {$\frac{\exp({-\|U^{\top}(\xb_{n+1}-\xb_i)\|^2/h^2})}{\sum_{i=1}^n \exp({-\|U^{\top}(\xb_{n+1}-\xb_i)\|^2/h^2})}$} \eqref{eq:anisotropic_kernelF7} \cr
\midrule
$\cF_{\cH^\alpha}$ \eqref{eq:funholder}  (special case of $\cF_{{\rm MI},U}$ with $U=I$, studied \\ \hspace{18mm} in \citet{shen2025understanding}) & $\frac{\exp({-\|(\xb_{n+1}-\xb_i)\|^2/h^2})}{\sum_{i=1}^n \exp({-\|(\xb_{n+1}-\xb_i)\|^2/h^2})}$ \eqref{eq:isotropic_kernel} \cr
\bottomrule
\end{tabular}
\end{table*}

\subsection{Low-rank Linear Model} \label{subsec:low_rank_linear}

We  first consider linear functions that share a low-dimensional linear structure across different tasks. Let  $W \in \RR^{D \times r}$ be a fixed matrix with orthonormal columns,  which represent a low-dimensional task subspace.
The low-rank linear function class 
\begin{align}
\cF_{{\rm Linear},W} := \{f: \text{ }f(\xb) =\wb^\top \xb \text{ for some } \wb = W \bbeta, \text{ with } \bbeta \sim \rho_{\bbeta} \text{ in } \RR^r \text{ satisfying } \|\bbeta\|_{2}\leq\beta \}.
\label{eq:funlinearW}
\end{align}
exhibits a low-dimensional task structure characterized by the subspace $W$.
In $\cF_{{\rm Linear},W}$, the task subspace $W$ exists, but is unknown as a prior information. The information about $W$ is supposed to be learned cross all training tasks. See Figure~\ref{fig:low_rank_illustration} for an illustration of this function class.

\begin{figure}[t]
    \centering
    \includegraphics[width=0.9\linewidth]{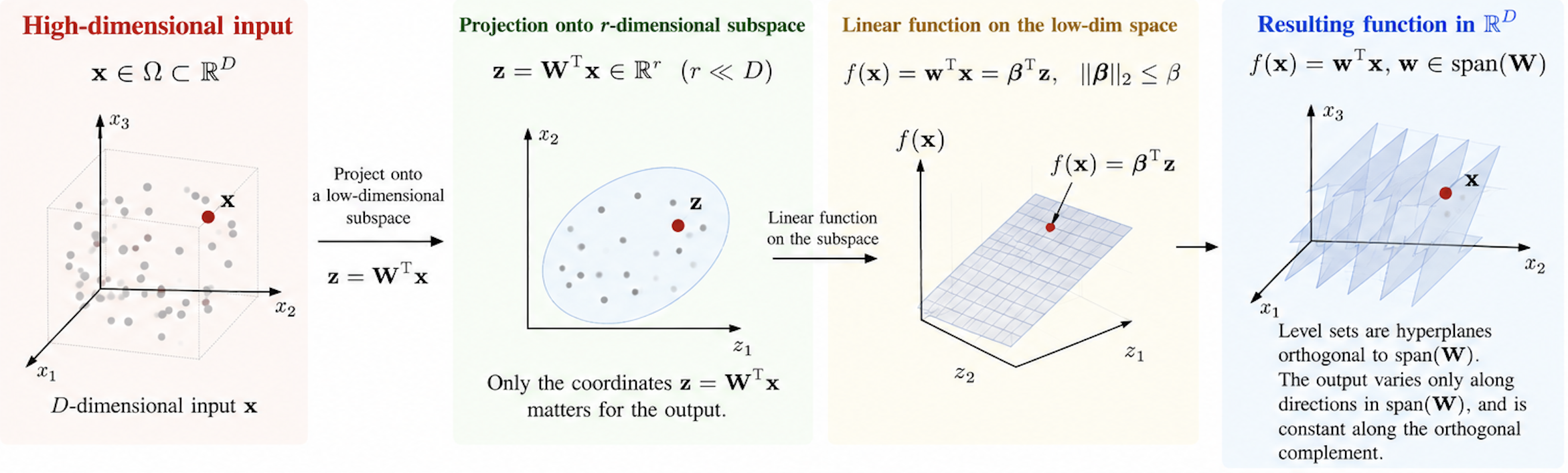}
    \vspace{-2mm}\caption{Illustration of the low-rank linear function class $\cF_{{\rm Linear},W}$.}
    \label{fig:low_rank_illustration}
\end{figure}

Suppose $X=[\xb_1,\cdots,\xb_n]^{\top} \in \RR^{n\times D}$.
For convenience, we introduce $Z=[\zb_1,\cdots,\zb_n]^{\top} \in \RR^{n \times r}$ with $\zb_i:=W^{\top}\xb_i \in \RR^{r}$ for $i=1,\cdots,n$, and $Z=XW\in\mathbb{R}^{n\times r}$.  The uncentered population covariance matrix of the random vector $\zb:=W^{\top}\xb$ when $\xb\sim\rho_{\xb}$ is $\Sigma_{\zb}:=\mathbb{E}[\zb\zb^{\top}]=\mathbb{E}_{\xb\sim\rho_{\xb}}[W^{\top}\xb\xb^{\top}W] \in \RR^{r \times r}$. 

The population covariance matrix $\Sigma_{\zb}$ positive semidefinite but may not be full rank. In the case that $\Sigma_{\zb}$ is not full rank, we use $\Sigma_{\zb}^{\dagger}$ to indicate its pseudo-inverse. Suppose the population level covariance $\Sigma_{\zb}$ has rank $k$, our oracle kernel estimator \eqref{eq:linear_low_rank_Ks} will be in terms of the $k$-th order truncated pseudo-inverse $[\cdot]^{\dagger}_k$ as defined in Definition \ref{def:pseudo-inverse}.

Let $\yb=[y_1,\cdots,y_n]^{\top}\in \RR^n$. The function $f(\xb) = \wb^\top\xb = (W\bbeta)^\top \xb = \xb^\top W\bbeta$ can be estimated from the linear system $XW\bbeta= \yb$  with the least squares solution $\hat\bbeta = [W^\top X^\top X W]_k^{\dagger} W^\top X^\top \yb$ so that $f(\xb_{n+1})$ can be approximated by the oracle kernel estimator
\begin{align} \label{eq:linear_low_rank_Ks}
\nonumber
    \cK(\fs)=\cK(\fs_f)&= \xb_{n+1}^\top W[W^\top X^\top X W]_k^{\dagger} W^\top X^\top \yb \\
    &=\sum_{i=1}^n  (W^\top \xb_{n+1})^\top \underbrace{ [W^\top X^\top X W]_k^{\dagger}}_{r \times r} (W^\top \xb_i) y_i={\sum_{i=1}^n 
    K(\xb_{n+1},\xb_i)y_i}
\end{align}
which can be expressed in the  kernel form of \eqref{eq:Kgeneral} with the linear interaction kernel
\begin{equation} \label{eq:linear_low_rank_kernel}
K(\xb_{n+1},\xb_i) = (W^\top \xb_{n+1})^\top \underbrace{ [W^\top X^\top X W]_k^{\dagger}}_{r \times r} (W^\top \xb_i).
\end{equation}
The low-rank linear model directly connects to the ordinary linear model as well as low-rank linear feature model, where we briefly showcase their connections in Appendix \ref{sec:other_kernels}. It is worth pointing out the previous work in \citet{hsu2026understanding} studies ICL in the low-rank linear feature model setup when the feature map is taken to be polynomial or spline basis.

We next show that transformer does adaptively learn the low-rank structure in their kernel representation via simulated experiments.

  \begin{table*}[t]
\centering
\caption{Average Pearson correlation coefficients ($\pm$ standard deviation) and the corresponding $p$-values ($\pm$ standard deviation) for function class $\cF_{{\rm Linear},W}$ with $k=r=2$. The low-rank linear kernel is computed by using \eqref{eq:linear_low_rank_kernel}, and the anisotropic Gaussian kernel is computed by using \eqref{eq:anisotropic_kernelF7} with $U=W$.}
\label{tab:Pearson_F2}
\vspace{2mm}
	\begin{tabular}{ccc}
	\toprule
           Kernel function  & Pearson correlation coefficient & $p$-value  \\
    \midrule
    anisotropic Gaussian  & 0.64 $\pm$ 0.24  & 0.09 $\pm$ 0.19 \\
	isotropic Gaussian  & 0.44 $\pm$ 0.24 & 0.21 $\pm$ 0.27 \\ 
{\bf low-rank linear} & 0.66 $\pm$ 0.26 & 0.09 $\pm$ 0.20  \\
linear & 0.41 $\pm$ 0.23 & 0.24 $\pm$ 0.27 \\
	\bottomrule
	\end{tabular}
\end{table*} 

\begin{figure}[h!]
    \centering
    \begin{tabular}{c}
         \includegraphics[width=0.9\linewidth]{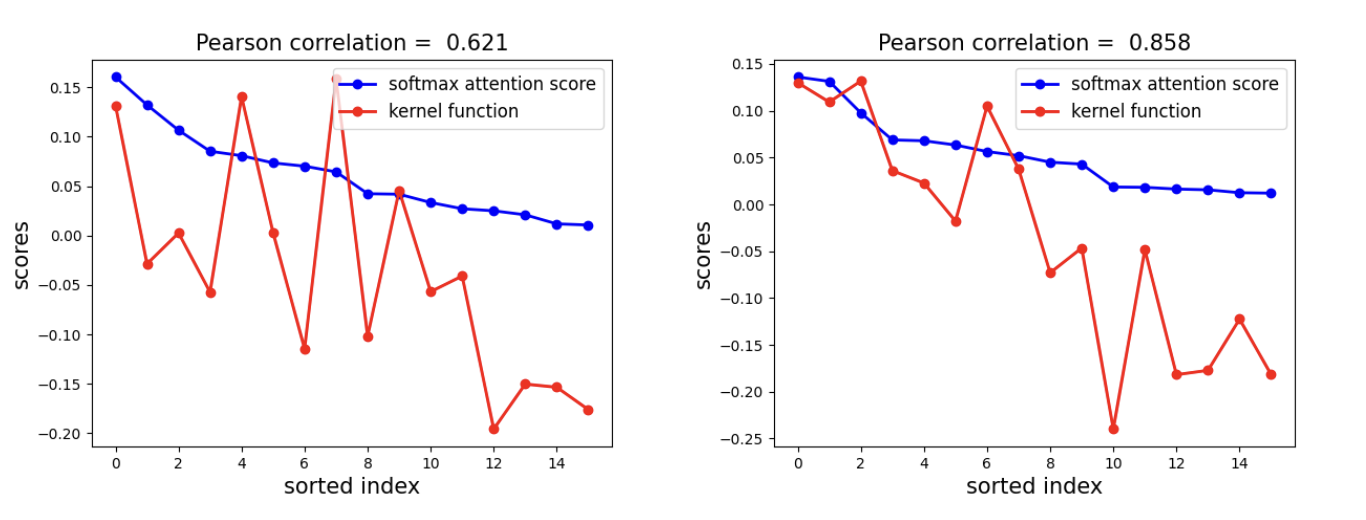} \\
     \includegraphics[width=1.0\linewidth]{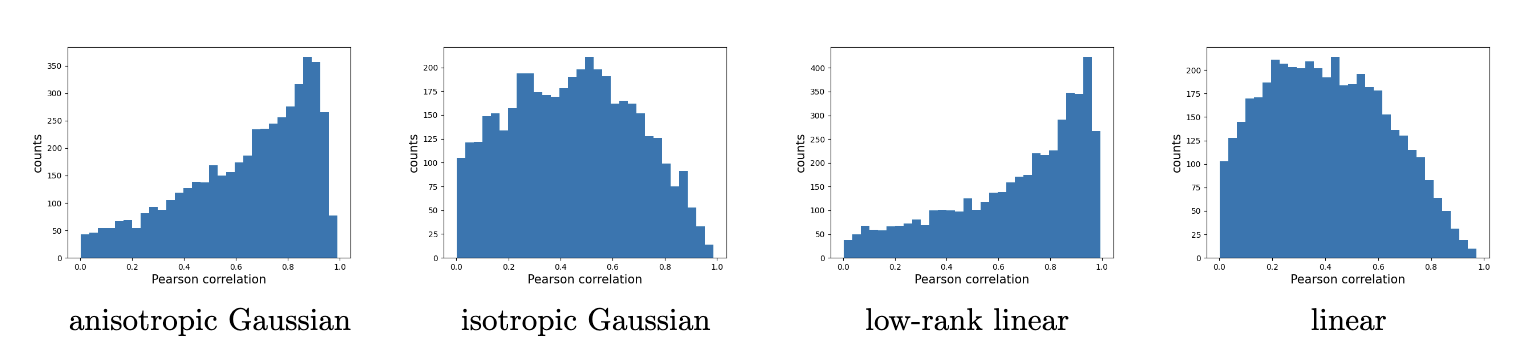} 
    \end{tabular}
    \caption{\emph{Top row}: two examples of  attention scores and low-rank linear kernel function \eqref{eq:linear_low_rank_kernel} with context length $n=16$ for $\cF_{{\rm Linear},W}$.  This subfigure shows  a strong correlation between attention scores and low-rank linear kernel. \emph{Bottom row}: histograms of the Pearson correlation for anisotropic kernel, isotropic kernel, low-rank linear kernel, linear kernel, respectively, for $\cF_{{\rm Linear},W}$. The ones with negative correlation are not included in this plot, while they only account for a small amount. 
       The total counts for positive correlation are  $4737,4268,4563,4367$ out of a total of $5000$ samples in each case respectively.}
    \label{fig:AttenScore_F2}
\end{figure}

\paragraph{Experiments for low-rank linear model.}
For low-rank linear function class $\cF_{{\rm Linear},W}$, we consider $f(\xb)=\wb^{\top}\xb$ for $\xb\in[0,1]^5$ with $\wb=W\beta\in\mathbb{R}^5$. Here we uniform randomly sample $\beta$ from $[0,1]^2$, and each realization of $\beta$ is a learning task. $W\in\mathbb{R}^{5\times 2}$ is chosen as an orthonormal matrix across all the tasks.

Next, we train the transformer to learn this function class. Top row of Figure \ref{fig:AttenScore_F2}
plots the attention scores (after sorting from the highest to the lowest value) in the last layer of the trained transformer and compares it against the low-rank linear kernel function \eqref{eq:linear_low_rank_kernel} (sorted according to the sorting index for attention scores), which demonstrates a strong correlation between the two quantities. Bottom row of Figure~\ref{fig:AttenScore_F2} plots the distribution of the Pearson correlation values between the last layer of the trained transformer and different kernel functions. 

Table~\ref{tab:Pearson_F2} further confirms that the low-rank linear kernel \eqref{eq:linear_low_rank_kernel} has the largest correlation with the transformer model final layer attention score.   Note that the anisotropic Gaussian kernel achieves a comparable Pearson correlation coefficient, due to both the low-rank linear and anisotropic Gaussian kernels can capture the same underlying directional structure. Further experimental details are provided in Appendix \ref{sec:experiment_details}.

\subsection{Multi-index Model} \label{subsec:multi-index}
We next consider an anisotropic regular function model, i.e., multi-index model, a class of nonlinear function models which share some low-dimensional task features. The multi-index model for a fixed subspace $U \in \RR^{D \times r}$:
\begin{align} \label{multi-index}
\nonumber
    \cF_{{\rm MI},U} := \{f: \Omega\to\mathbb{R}\text{ }|&\text{ } f(\xb) =g(U^\top \xb) \text{ for some } \alpha\text{-H\"older } g:\mathbb{R}^r\to\mathbb{R} \text{ with H\"older constant} \\
&\text{ no larger than } L, \text{ and } \|f\|_{L^{\infty}(\Omega)}\le R\}
\end{align}
Here all tasks share the same low-dimensional task space $U$. Without loss of generality, we assume the columns of $U$ are orthonormal. See Figure~\ref{fig:multi_index_illustration} for an illustration of this function class.

\begin{figure}[t]
    \centering
    \includegraphics[width=0.9\linewidth]{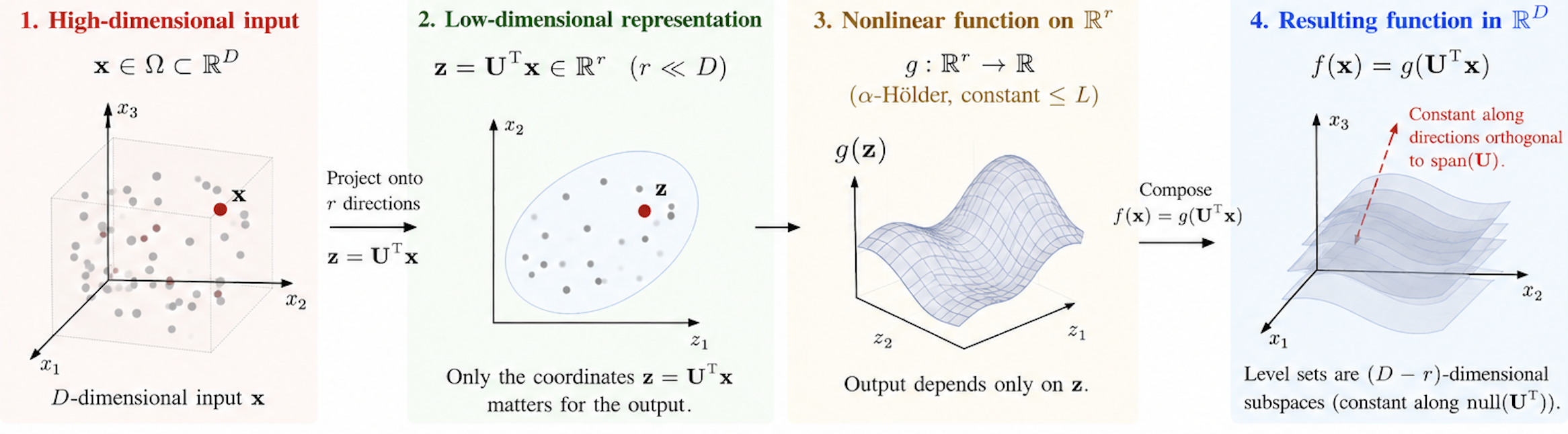}
    \vspace{-2mm}\caption{Illustration of the multi-index function class $\cF_{{\rm MI},U}$.}
    \label{fig:multi_index_illustration}
\end{figure}

If we use the classical Nadaraya–Watson kernel estimator \citep{nadaraya1964estimating,watson1964smooth} to estimate $f$ through $g$, we get
the corresponding oracle kernel estimator for $\cF_{{\rm MI},U}$ is given by
\begin{equation} \label{eq:NWkernelF7}
    \cK(\fs)= \cK(\fs_f)=\frac{\sum_{i=1}^ny_iK_h(\xb_{n+1}-\xb_i)}{\sum_{i=1}^n K_h(\xb_{n+1}-\xb_i)}={\sum_{i=1}^n 
    K(\xb_{n+1},\xb_i)y_i}.
\end{equation}
One common choice is $K_h(\xb) = e^{-\|\xb\|_2^2/h^2}$, with a proper choice of the kernel bandwidth $h$. With this choice $K_h$ for the estimation of $g$, the corresponding interaction kernel for the estimation of $f$ can be written as
\begin{equation} \label{eq:anisotropic_kernelF7}
K(\xb_{n+1},\xb_i)=\frac{e^{-\|U^{\top}(\xb_{n+1}-\xb_i)\|^2/h^2}}{\sum_{i=1}^n e^{-\|U^{\top}(\xb_{n+1}-\xb_i)\|^2/h^2}},
\end{equation}
when the subspace $U$ is known. In practice, the subspace $U$ is unknown, and needs to be learned from training data.

The multi-index model naturally works as a generalization of the single-index model as well as the class of isotropic regular function, we briefly mention their connection in the Appendix \ref{sec:other_kernels}. It is worth pointing out the  work in \citet{shen2025understanding} studies ICL when the Nadaraya–Watson kernel is set to be isotropic when $U=I$.

We next demonstrate via simulated experiments that transformer does adaptively learn the anisotropic structure in their kernel representation.

\paragraph{Experiments for multi-index models}

For function class $\cF_{{\rm MI},U}$, we consider $f(\xb) =g(U^\top \xb)$, for $\xb\in [0,1]^5$, where $g$ is a Legendre polynomial on $\mathbb{R}^2$ with six  random coefficients, i.e., $g(x,y)=c_0 P_0(x,y)+c_1P_1(x,y)+c_2P_2(x,y)+c_3P_3(x,y)+c_4P_4(x,y)+c_5P_5(x,y)$, where $c_0,c_1,c_2,c_3,c_4,c_5$ are uniformly random sampled from $[0,1]$. Each realization of $c_0,c_1,c_2,c_3,c_4,c_5$ is a learning task. In $\cF_{{\rm MI},U}$, we keep $U\in\mathbb{R}^{5\times 2}$ as a random but fixed orthonormal matrix across all tasks.

Next, we train the transformer to learn this function class. Top row of Figure \ref{fig:AttenScore_F5}
plots the attention scores (after sorting from the highest to the lowest value) in the last layer of the trained transformer and compares it against the anisotropic Gaussian kernel function \eqref{eq:anisotropic_kernelF7} (sorted according to the sorting index for attention scores), which demonstrates a strong correlation between the two quantities. Bottom row of Figure~\ref{fig:AttenScore_F5} plots the distribution of the Pearson correlation values between the last layer of the trained transformer and different kernel functions. 

Table~\ref{tab:Pearson_F5} further confirms that the anisotropic Gaussian kernel \eqref{eq:anisotropic_kernelF7} has the largest correlation with the transformer model final layer attention score.   Further experimental details are provided in Appendix \ref{sec:experiment_details}.  

 \begin{table}[t]
\centering
\caption{Average Pearson correlation coefficients ($\pm$ standard deviation) and the corresponding $p$-values ($\pm$ standard deviation) for function class $\cF_{{\rm MI},U}$. The anisotropic Gaussian kernel is computed by using \eqref{eq:anisotropic_kernelF7}, and the low-rank linear kernel is computed by using \eqref{eq:linear_low_rank_kernel} with $W=U$.}
\label{tab:Pearson_F5}
\vspace{2mm}
	\begin{tabular}{ccc}
	\toprule
           Kernel function  & Pearson correlation coefficient & $p$-value  \\
    \midrule
    {\bf anisotropic Gaussian}  & 0.76 $\pm$ 0.11  & 0.004 $\pm$ 0.013 \\
	isotropic Gaussian  & 0.39 $\pm$ 0.24 & 0.27 $\pm$ 0.29 \\ 
low-rank linear & 0.51 $\pm$ 0.24 & 0.17 $\pm$ 0.25  \\
linear & 0.33 $\pm$ 0.19 & 0.31 $\pm$ 0.29 \\
	\bottomrule
	\end{tabular}
\end{table} 

\begin{figure}[h!]
    \centering
    \begin{tabular}{c}
         \includegraphics[width=0.9\linewidth]{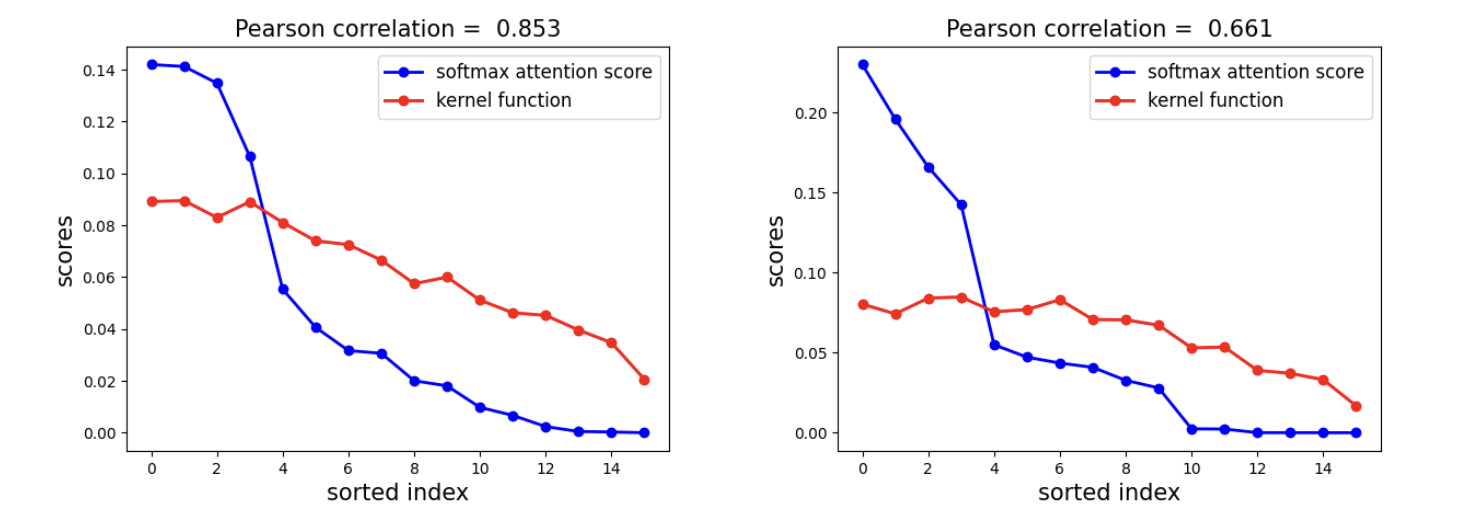} \\
     \includegraphics[width=1.0\linewidth]{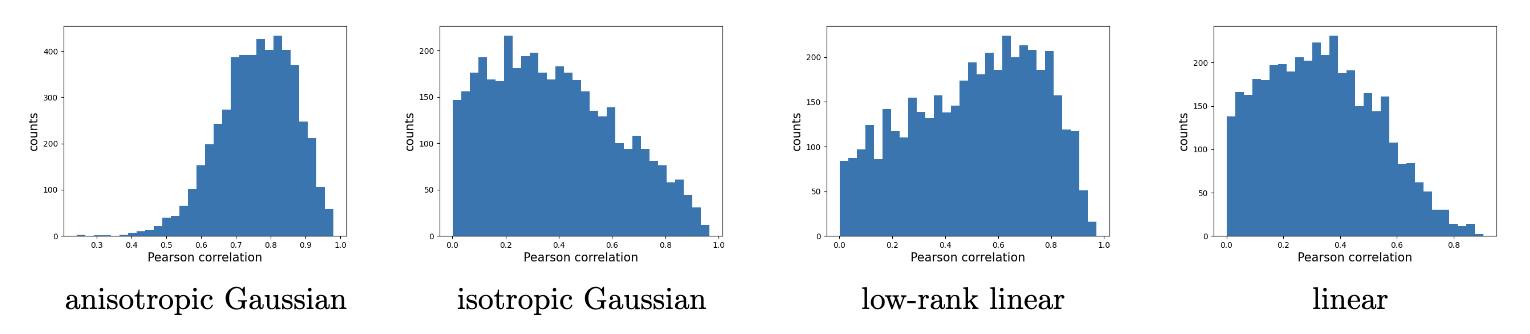} 
    \end{tabular}
    \caption{\emph{Top row}: two examples of attention scores and anisotropic Gaussian kernel function \eqref{eq:anisotropic_kernelF7} with context length $n=16$ for $\cF_{{\rm MI},U}$.  This subfigure shows  a strong correlation between attention scores and anisotropic Gaussian kernel. \emph{Bottom row}: histograms of the Pearson correlation for anisotropic kernel, isotropic kernel, low-rank linear kernel, linear kernel, respectively, for $\cF_{{\rm MI},U}$. The ones with negative correlation are not included in this plot, while they only account for a small amount. 
       The total counts for positive correlation are  $5000,3952,4352,3972$ out of a total of $5000$ samples in each case respectively.}
    \label{fig:AttenScore_F5}
\end{figure}

\section{Error Analysis Connecting Transformers and Kernel Algorithms}
\label{sec:statistical_error}

To establish the generalization error for both the low-rank linear model and the multi-index model using transformer, let us demonstrate the generalization error decomposition as well as the statistical error estimation in Section \ref{subsec:Stats} and Section \ref{subsec:StatsNoise}.

We first decompose the mean squared generalization error in \eqref{eq:mse1} via a bias–variance decomposition.
For the sake of error analysis, we define $\fS_f:=\{\fs_f^\gamma\}_{\gamma=1}^\Gamma$ where 
\[\fs^\gamma_f=\{\xb_1^\gamma,y^\gamma_1,\ldots,\xb_n^\gamma,y_n^\gamma;\xb_{n+1}^\gamma,f^\gamma(\xb^\gamma_{n+1})\},\]
and the empirical risk without noise
\begin{align}
    \hat{\cR}_{\fS_f}({\rmT}_\theta) :=  \frac{1}{\Gamma}\sum_{\gamma=1}^\Gamma \cR_{\fs^\gamma_f}(\rmT_\theta)=\frac{1}{\Gamma}\sum_{\gamma=1}^\Gamma  \Big({\rmT}_\theta(\{\xb^\gamma_i,y^\gamma_i\}_{i=1}^n\};\xb^\gamma_{n+1}) - f^\gamma(\xb^\gamma_{n+1})\Big)^2. 
\end{align}
We can then decompose the squared generalization error of $\hat\rmT$ on a random test sample $\fs_f$ as follows:
\begin{align}
\nonumber
\cR_{\fs_f}(\hat\rmT) &= 
\underbrace{{\cR_{\fs_f}(\hat\rmT)}-\hat\cR_{\fS_f}(\hat\rmT)}_{\rm III}+\underbrace{\hat\cR_{\fS_f}(\hat\rmT) - \hat\cR_{\fS}(\hat\rmT)}_{\rm V} 
 +\underbrace{{\hat\cR_{\fS}(\hat\rmT)  - \hat\cR_{\fS}(\rmT^*)}}_{\rm \leq 0} \\
 \label{eq:errde}
&+ \underbrace{\hat\cR_{\fS}(\rmT^*) -\hat\cR_{\fS_f}(\rmT^*)}_{\rm VI} + \underbrace{\hat\cR_{\fS_f}(\rmT^*)- \cR_{\fs_f}(\rmT^*)}_{\rm IV} 
+
\underbrace{{\cR_{\fs_f}(\rmT^*)} - 2{\cR_{\fs_f}(\cK)}}_{\rm II} + \underbrace{2{\cR_{\fs_f}(\cK)}}_{\rm I},
\end{align}
where $\rmT^*$ is a transformer network which approximates the oracle  kernel algorithm ${\cK}$ in \eqref{eq:Kgeneral}. Notice that ${\hat\cR_{\fS}(\hat\rmT) - \hat\cR_{\fS}(\rmT^*)}\le 0$ since $\hat\rmT$ is the minimizer of $\hat\cR_{\fS}$ given in \eqref{eq:RnT}.
After taking expectations testing and training samples, we can decompose the mean squared generalization error in \eqref{eq:mse1} as: 
\begin{align}
\cR(\hat\rmT) =& 
\EE_{\fS} \EE_{\fs_f} \left[ \cR_{\fs_f}(\hat\rmT)\right]
\leq 
\underbrace{ \EE_{\fs_f}\left[2\cR_{\fs_f}(\cK)\right]}_{{\rm I}} 
\label{eq:errde1}
\\
&+\underbrace{\EE_{\fs_f}\left[\cR_{\fs_f}(\rmT^*)-2\cR_{\fs_f}(\cK)\right]}_{\rm II} 
\label{eq:errde2}
\\
&+\underbrace{\EE_{\fS}\left(\EE_{\fs_f}\left[\cR_{\fs_f}(\hat\rmT)\right] - \hat\cR_{\fS_f}(\hat\rmT) \right)}_{{\rm III}}
+ \underbrace{\EE_{\fS} \Big(\hat\cR_{\fS_f}(\rmT^*) -  \EE_{\fs_f} \left[ \cR_{\fs_f}(\rmT^*)\right]\Big)}_{{\rm IV}} 
\label{eq:errde4}
\\
& + \underbrace{\EE_{\fS}\left[\hat\cR_{\fS_f}(\hat\rmT) -\hat\cR_{\fS}(\hat\rmT)\right]}_{\rm V} 
+ \underbrace{\EE_{\fS}\left[\hat\cR_{\fS}(\rmT^*) -\hat\cR_{\fS_f}(\rmT^*)\right]}_{\rm VI}.
\label{eq:errde6}
\end{align}

In this error decomposition, the term ${\rm I}$ represents the error given by the oracle kernel algorithm, the term ${\rm II}$ represents the transformer network approximation error for the oracle kernel algorithm. The terms ${\rm III}$ and ${\rm IV}$ correspond to statistical estimation error by transformers, which depends on the transformer network complexity, and the terms ${\rm V}$ and ${\rm VI}$ correspond to the statistical error induced by the noise on the output.
In Sections \ref{subsec:Stats} and \ref{subsec:StatsNoise}, we will analyze the statistical errors  in the terms {\rm III - VI}, which depend on the transformer network class. The approximation errors in the terms {\rm I - II} depend on the function class $\cF$, which will be analyzed in Section \ref{sec:LinearFunctions} for low-rank linear functions and in Section \ref{sec:multiindex} for nonlinear multi-index functions.

\subsection{Statistical Error of Transformers}
\label{subsec:Stats}

To bound the statistical error terms in ${\rm III}$ and ${\rm IV}$ in \eqref{eq:errde4}, we invoke the standard statistical error analysis techniques such as Rademacher complexity and Dudley entropy integral \citep{dudley1967sizes}.

For each sample, the loss function we considered is the squared loss 
\[\cL_{\fs_f}(\rmT) = ( \rmT(\{\xb_i,y_i\}_{i=1}^n;\xb_{n+1})- f(\xb_{n+1}))^2.\] 
and the Rademacher complexity of $\cL\circ\cT$ with respect to $\fS_f$
is
\begin{equation} \label{Ram}
\Rad_{\fS_f}(\cL\circ\cT):=\frac{1}{\Gamma} \EE_{\bxi \sim \{\pm 1\}^\Gamma} \left[\sup_{\rmT \in \cT} \sum_{\gamma=1}^\Gamma \xi_\gamma \Big(\rmT(\{\xb^\gamma_i,y^\gamma_i\}_{i=1}^n;\xb^\gamma_{n+1}) - f^{\gamma}(\xb^{\gamma}_{n+1})\Big)^2\right].
\end{equation}
Lemma \ref{lemma:III} below establishes the bound for statistical errors ${\rm III}$ and ${\rm IV}$. We defer its proof to Appendix \ref{subsec:proof_stats1}.

\begin{lemma} \label{lemma:III}
Let $\cT$ be the transformer network class defined in \eqref{def:TNC} and $\mathcal{N}(\delta,\cL\circ\mathcal{T},\|\cdot\|_{L^\infty})$ be the covering number of the function class $\cL\circ\cT$ under the $L^\infty$ norm. Then
    \begin{equation}
        {\rm III}+{\rm IV} \leq 
   4\inf_{ \epsilon>0}\left(2\epsilon+\frac{12}{\sqrt{\Gamma}}\int_{\epsilon}^{4R^2}\sqrt{\log\mathcal{N}(\delta,\cL\circ\mathcal{T},\|\cdot\|_{L^\infty})}\text{ }d\delta\right).
    \end{equation}
\end{lemma}

\subsection{Statistical  Error Induced by Noise}
\label{subsec:StatsNoise}

For the statistical errors induced by noise in ${\rm V}$ and ${\rm VI}$, we can show that
\begin{align} \label{eq:V_VI}
    {\rm V}+{\rm VI}\leq 2\EE_{\fS}[\cG_{\fS}(\cL'\circ\cT)],
\end{align}
where $\cL'(\rmT\circ\fs^{\gamma}):=\rmT(\{\xb^\gamma_i,y^\gamma_i\}_{i=1}^n;\xb^\gamma_{n+1})$, and $\cG_{\fS}(\cL'\circ\cT)$ is the Gaussian complexity defined as 
\begin{equation}
\cG_{\fS}(\cL'\circ\cT):=\frac{1}{\Gamma} \EE_{\{\epsilon_{n+1}^\gamma\}_{\gamma=1}^\Gamma}\sup_{{\rm T}\in\cT}\left[ \sum_{\gamma=1}^\Gamma \epsilon_{n+1}^{\gamma}\Big(\rmT(\{\xb^\gamma_i,y^\gamma_i\}_{i=1}^n;\xb^\gamma_{n+1})\Big)\right].
\end{equation}

This leads to
Lemma \ref{lemma:V} which establishes the bound for the statistical errors in ${\rm V}$ and ${\rm VI}$. We defer the derivation of \eqref{eq:V_VI} and the proof of Lemma \ref{lemma:V} to Appendix \ref{subsec:proof_stats2}.

\begin{lemma} \label{lemma:V}
Let $\cT$ be the transformer network class defined in \eqref{def:TNC} and $\mathcal{N}(\delta,\cL'\circ\mathcal{T},\|\cdot\|_{L^\infty})$ be the covering number of the function class $\cL'\circ\cT$ under the $L^\infty$ norm. Then
    \begin{equation}
        {\rm V}+{\rm VI} \leq   4\sigma\sqrt{\log(\Gamma)}\cdot\inf_{ \epsilon>0}\left(2\epsilon+\frac{12}{\sqrt{\Gamma}}\int_{\epsilon}^{R}\sqrt{\log\mathcal{N}(\delta,\cL'\circ\mathcal{T},\|\cdot\|_{L^\infty})}\text{ }d\delta\right).
    \end{equation}
\end{lemma}

\paragraph{Transformer Network Covering Number.}
To further bound the covering numbers appearing in Lemmas~\ref{lemma:III} and~\ref{lemma:V}, we establish a bound on the covering number of the transformer network class $\cT$ in Lemma~\ref{lemmacovering}. The proof follows the argument of Lemma~4 in \citet{shen2025understanding} and we omitted the proof details here.

\begin{lemma} \label{lemmacovering}
    For a transformer class $\mathcal{T}(L_{\rm T}, m_{\rm T},d_{embed},\ell,L_{\rm FFN},w_{\rm FFN},R,\kappa)$ with input $\|H\|_{\infty,\infty}\leq u$. Let $\delta>0$, then the covering number of
    $\cL\circ\cT$ satisfies
    \begin{align*}
    \mathcal{N}(\delta,\cL\circ\cT,\|\cdot\|_{L^\infty})\leq \left(\frac{ 2^{L^2_{\rm T}+4}L_{\rm FFN}u^{3L_{\rm T}}d_{embed}^{18L_{\rm T}^2L_{\rm FFN}}\kappa^{6L_{\rm T}^2 L_{\rm FFN}+1}m_{\rm T}^{L_{\rm T}^2}\ell^{L_{\rm T}^2}R}{\delta}\right)^{P_{\rm T}},
\end{align*}
and the covering number of $\cL'\circ\cT$ satisfies 
\begin{align*}
    \mathcal{N}(\delta,\cL'\circ\cT,\|\cdot\|_{L^\infty})\leq \left(\frac{ 2^{L^2_{\rm T}+2}L_{\rm FFN}u^{3L_{\rm T}}d_{embed}^{18L_{\rm T}^2L_{\rm FFN}}\kappa^{6L_{\rm T}^2 L_{\rm FFN}+1}m_{\rm T}^{L_{\rm T}^2}\ell^{L_{\rm T}^2}}{\delta}\right)^{P_{\rm T}},
\end{align*}
where $P_{\rm T}=d_{embed}(D+1)+L_{\rm T}(3d_{embed}^2m_{\rm T}+L_{\rm FFN}w_{\rm FFN}^2)$.
\end{lemma}

Our next goal  is to study the approximation errors {\rm I} and {\rm II} for different function classes $\cF$. In particular, we will consider low-rank linear functions in Section \ref{sec:LinearFunctions} and nonlinear multi-index functions in Section \ref{sec:multiindex}, with a focus on leveraging a low-dimensional task space which is shared for all prompts.

\section{ICL Generalization Error for Low-rank Linear Model}
\label{sec:LinearFunctions}

In this section, we study the approximation error ${\rm I}$ and ${\rm II}$ for the low-rank linear function class $\cF_{{\rm Linear},W}$ in the ICL setup.
We first show that the transformer can approximate the low-rank linear kernel estimator $\cK(\fs)$ in \eqref{eq:linear_low_rank_Ks}, and then we bound its mean squared generalization error $\cR(\hat\rmT)$.

Our theory is built on the following assumptions. The first assumption is about the function class.

\begin{assumption}[Function Class]
\label{assumpf_low_rank_linear} Let $\beta>0$, $W\in\mathbb{R}^{D\times r}$ be a fixed  matrix  such that $W^\top W = I$, and $\rho_{\bbeta}$ be a distribution in $\RR^r$ supported within $\{\bbeta: \|\bbeta\|_{2}\leq\beta\}$. 
Suppose the function class $\cF$ is the low-rank linear function class in \eqref{eq:funlinearW} such that
\[\cF = \cF_{{\rm Linear},W} := \{f: \text{ }f(\xb) =\wb^\top \xb \text{ for some } \wb = W \bbeta, \text{ with } \bbeta \sim \rho_{\bbeta} \text{ in } \RR^r \text{ satisfying } \|\bbeta\|_{2}\leq\beta \},\]
and $\rho_f$ is the distribution in $\cF$ induced by the distribution $\rho_{\bbeta}$.  
\end{assumption}

For convenience, we introduce $Z=[\zb_1,\cdots,\zb_n]^{\top} \in \RR^{n \times r}$ with $\zb_i:=W^{\top}\xb_i \in \RR^{r}$ for $i=1,\cdots,n$. Then $Z=XW\in\mathbb{R}^{n\times r}$.  The uncentered population covariance matrix of the random vector $\zb:=W^{\top}\xb$ when $\xb\sim\rho_{\xb}$ is $\Sigma_{\zb}:=\mathbb{E}[\zb\zb^{\top}]=\mathbb{E}_{\xb\sim\rho_{\xb}}[W^{\top}\xb\xb^{\top}W] \in \RR^{r \times r}$. 

Notice that $\Sigma_{\zb}$ positive semidefinite. In case that $\Sigma_{\zb}$ is not full rank, we use $\Sigma_{\zb}^{\dagger}$ to indicate its pesudo-inverse. Our second assumption is about the boundedness of input $\xb$, and low-rankness of the covariance matrix over transformed input $\zb$.

\begin{assumption}[Input Distribution]
\label{assumpx_low_rank_linear}
Let $b>0$ and $W\in\mathbb{R}^{D\times r}$ be a fixed  matrix  such that $W^\top W = I$. The input distribution $\rho_{\xb}$ is supported within the ball $\Omega=\{\xb:\|\xb\|_2\leq b\}$. The convariance matrix  $\Sigma_{\zb}$ has rank $k$: $\lambda_1(\Sigma_{\zb}) \ge \ldots \ge \lambda_k(\Sigma_{\zb})>\lambda_{k+1}(\Sigma_{\zb}) = \cdots=\lambda_r(\Sigma_{\zb})=0$ and the top $k$ nonzero-eigenvalues are bounded away from $0$: there exists $\tau>0$ such that
\begin{equation} 
\|\Sigma_{\zb}^{\dagger}\|_2={\lambda_{1}(\Sigma_\zb^{\dagger})}= \frac{1}{\lambda_{k}(\Sigma_\zb)}\leq\tau.
\end{equation}
\end{assumption}

It is worth noting that existing studies have primarily focused
on either full-rank linear models
\citep{bai2023transformers,zhang2024trained}
or low-rank models in which the shared structure is characterized
solely by the rank of $W$
\citep{maurer2016benefit,tripuraneni2021provable}. Our framework extends these settings by simultaneously
accounting for the rank $r$ of the shared representation $W$
and the effective rank $k$ of its interaction with the
input data $X$. In particular, our analysis naturally recovers the cases
$r=D$ and $k=r$ as special cases.

Intuitively, Assumption \ref{assumpx_low_rank_linear} assumes that among all the $k$ identifiable directions in $\mathbb{R}^{r}$, there is no direction that is negligible. 
Under Assumption  \ref{assumpx_low_rank_linear} with $\|\xb\|_2\leq b$ and Assumption  \ref{assumpf_low_rank_linear} with $\|\bbeta\|_2\leq \beta$,  it is easy to see that all functions in $\cF$ are bounded by $\|f\|_{L^{\infty}(\Omega)} = \sup_{\|\bbeta\|_2\leq \beta}|(W\bbeta)^\top \xb|\le b\beta$. 

Further, let us denote $\hat{\Sigma}_{\zb}:=\frac{Z^{\top}Z}{n}$ as the empirical covariance matrix. When the sample size $n$ is large, it is often the case $\hat{\Sigma}_{\zb}\approx\Sigma_{\zb}$. We consider a good event $\mathcal{E}$ when the difference between $\hat{\Sigma}_{\zb}$ and $\Sigma_{\zb}$ is bounded by $\frac{1}{4\tau}$:
\begin{equation} \label{eq:good_event}
  \mathcal{E}:=\{\|\hat{\Sigma}_{\zb}-\Sigma_{\zb}\|_2\leq \frac{1}{4\tau}\}.   
\end{equation}
Our goal next is to bound the errors ${\rm I}$ and ${\rm II}$ in Section \ref{subsec:linear_kernel_alg} and Section \ref{subsec:error_linear_kernel_alg} on the good event $\mathcal{E}$. Then we combine all the errors to get an error bound for $\cR(\hat{\rm T})$ in Section \ref{subsec:total_error}.

\subsection{Transformer Approximates the Low-rank Linear Kernel} \label{subsec:linear_kernel_alg}
Our goal in this subsection is to demonstrate that, under the good event $\mathcal{E}$, a transformer
can approximate the low-rank linear kernel oracle estimator $\cK(\fs)$ defined in \eqref{eq:linear_low_rank_Ks}.
Specifically, in Lemma~\ref{lemma:approx_low_rank_linear}, we provide an explicit step-by-step construction of a linear transformer that uses the embedding matrix \eqref{eq:Hmatrix} to approximate the low-rank linear kernel estimator $\cK(\fs)$ to arbitrary prescribed accuracy. Note that such connection between low-rank linear kernel and transformer is also validated by the numerical experiment shown in Section \ref{subsec:low_rank_linear}.

\begin{lemma} \label{lemma:approx_low_rank_linear}
Suppose the input satisfies Assumption \ref{assumpx_low_rank_linear} and the target function class satisfies Assumption \ref{assumpf_low_rank_linear}. Then under Setting \ref{setting:icl}, there exists a
transformer network ${\rm T}^*\in\mathcal{T}$ with parameters 
\begin{align*}
&d_{embed}=2D+6, \text{ } \ell=n+1, \text{ } L_{\rm T}=2, \text{ } m_{\rm T}=1,\\
&L_{\rm FFN}=O(1), \text{ } w_{\rm FFN}=2D+6, \text{ } \kappa=O\left(\tau\right),
\end{align*}
such that on the good event $\mathcal{E}$ defined in \eqref{eq:good_event}, it satisfies for the low-rank linear oracle kernel $\cK$ defined in \eqref{eq:linear_low_rank_Ks},
\begin{equation*}
    \EE_{\fs_f}\left[\cR_{\fs_f}(\rmT^*)-2\cR_{\fs_f}(\cK)|\mathcal{E}\right]\leq \frac{36C_1k\tau^2\beta^2b^6}{n}+\frac{4C_2kr\sigma^2\tau^4b^8}{n^2},
\end{equation*}
where the good event $\mathcal{E}$ occurs with probability no less than $1-8k\exp\left(
-\frac{3n}{16b^2\tau(6b^2\tau+1)}
\right)$.
The constants $C_1,C_2$ hide the dependency on some absolute constants. 
\end{lemma}

The proof of Lemma \ref{lemma:approx_low_rank_linear} is deferred to Appendix \ref{subsec:proof_approx_low_rank_linear}.

\subsection{Approximation Error of the Low-rank Linear Kernel Algorithm} \label{subsec:error_linear_kernel_alg}

Our goal in this subsection is to show in Lemma \ref{lemma:low_rank_lsq}, under the good event $\mathcal{E}$, the low-rank linear kernel oracle $\cK(\fs)$ is accurate.

\begin{lemma} \label{lemma:low_rank_lsq}
Suppose the input satisfies Assumption \ref{assumpx_low_rank_linear} and the target function class satisfies Assumption \ref{assumpf_low_rank_linear}. Then under Setting \ref{setting:icl} and for good event $\mathcal{E}$ defined in \eqref{eq:good_event}, it satisfies
\begin{equation*}
    \EE_{\fs_f}\left[\cR_{\fs_f}(\cK)|\mathcal{E}\right]\leq \frac{18C_1k\tau^2\beta^2b^6}{n}+\frac{8r\sigma^2\tau^2b^4}{n}.
\end{equation*}
The constant $C_1$ hides the dependency on some absolute constants.
\end{lemma}

The proof of Lemma \ref{lemma:low_rank_lsq} is deferred to Appendix \ref{subsec:proof_low_rank_lsq}.

\subsection{Transformer-based ICL Generalization Error for Low-rank Linear Functions} \label{subsec:total_error}

Building on the error bounds established in Sections~\ref{subsec:linear_kernel_alg} and~\ref{subsec:error_linear_kernel_alg}, we derive a bound on the mean squared generalization error for the low-rank linear model \eqref{eq:funlinearW} in the ICL setting.

\begin{theorem} \label{Thm:low_rank_linear}
Suppose the target function class satisfies Assumption \ref{assumpf_low_rank_linear} and the input satisfies Assumption \ref{assumpx_low_rank_linear}. Then under Setting \ref{setting:icl}, if the transformer network class $\mathcal{T}$ is chosen with parameters 
\begin{align} 
\label{parameter_low_rank_linear}
&d_{embed}=2D+6, \text{ } \ell=n+1, \text{ } L_{\rm T}=2, \text{ } m_{\rm T}=1, \\
\nonumber
&L_{\rm FFN}=O(1), \text{ } w_{\rm FFN}=2D+6, \text{ } \kappa=O\left(\tau\right),
\end{align}
where $O(\cdot)$ hides the dependency on the absolute constants.
For $n\geq\Theta(\tau^2b^4)$ up to a logarithmic factor, the empirical risk minimizer $\hat\rmT$ defined in \eqref{eq:hatT} satisfies the mean squared generalization error bound:
 \begin{align}
 \label{eq:mainerrbound_low_rank_linear} \cR(\hat\rmT)\leq 
     \left(\frac{72C_1k\tau^2\beta^2b^6}{n}+\frac{16C_2r\sigma^2\tau^2b^4}{n}\right)+C_3\left(\frac{D\log(nD\Gamma)}{\sqrt{\Gamma}}\right),
 \end{align}
The constants $C_1,C_2$ hide the absolute constants, and $C_3$ hides the dependency on $b,\beta$.
\end{theorem}

Our generalization error bound in Theorem~\ref{Thm:low_rank_linear} also connects to the multi-task learning literature. The approximation rate obtained in \citet{bai2023transformers} and \citet{zhang2024trained} is $O(D/n)$, which scales with the ambient dimension $D$. Exploiting a shared low-dimensional structure in the function space, \citet{tripuraneni2021provable} and \citet{du2021few} improve this to $O(r/n)$. The dominating term in our bound is $O(k/n)$ with $k \le r$, reflecting the additional gain from the interplay between the shared representation $W$ and the input data $X$.

\begin{proof} [Proof of Theorem \ref{Thm:low_rank_linear}]
Our proof techniques of Theorem \ref{Thm:low_rank_linear} proceeds with a 
\emph{bias-variance error decomposition} on a random test sample $\fs_f$ as in \eqref{eq:errde}, then taking the expectation gives the error terms ${\rm I}$, ${\rm II}$, ${\rm III}$, ${\rm IV}$, ${\rm V}$, ${\rm VI}$ in  \eqref{eq:errde1}, \eqref{eq:errde2}, \eqref{eq:errde4},
\eqref{eq:errde6}
respectively. 
\vspace{2mm}
\\
$\bullet$ {\bf Error ${\rm I}$ and ${\rm II}$. } 
We use Lemma \ref{lemma:approx_low_rank_linear} and Lemma \ref{lemma:low_rank_lsq} to bound the mean squared error on the good event $\mathcal{E}$. Then combining with the tail probability happening on the bad event $\mathcal{E}^c$, we get for large $n$
\begin{align*}
    {\rm I}+{\rm II}&=\EE_{\fs_f}\left[\cR_{\fs_f}(\rmT^*)-2\cR_{\fs_f}(\cK)\right]+\EE_{\fs_f}[2\cR_{\fs_f}(\cK)] \\
    &\leq \left(\frac{36C_1k\tau^2\beta^2b^6}{n}+\frac{4C_2kr\sigma^2\tau^4b^8}{n^2}+2\cdot\frac{18C_1k\tau^2\beta^2b^6}{n}+2\cdot\frac{8r\sigma^2\tau^2b^4}{n}\right)\cdot\left(1-8ke^{-\frac{3n(1/4\tau)^2}{6b^4+4b^2(1/4\tau)}}\right) \\
    &+4b^2\beta^2\cdot8ke^{-\frac{3n(1/4\tau)^2}{6b^4+4b^2(1/4\tau)}} \\
    & \leq \left(\frac{72C_1k\tau^2\beta^2b^6}{n}+\frac{4C_2kr\sigma^2\tau^4b^8}{n^2}+\frac{16r\sigma^2\tau^2b^4}{n}\right)+4b^2\beta^2\cdot8ke^{-\frac{3n(1/4\tau)^2}{6b^4+4b^2(1/4\tau)}} \\
    &\leq C_3\left(\frac{72C_1k\tau^2\beta^2b^6}{n}+\frac{16r\sigma^2\tau^2b^4}{n}\right).
\end{align*}
The term with $n$ in the denominator will dominate both the term with $n^2$ in the denominator and the term with exponential decay when $n\geq \Theta(\tau^2b^4)$ up to a logarithmic factor.
\vspace{2mm}
\\
$\bullet$ {\bf Error ${\rm III}$ and ${\rm IV}$. } 
Since $\|f\|_{L^{\infty}(\Omega)} \le b\beta$, by Lemma \ref{lemma:III},  we get
\begin{align*}
   {\rm III}+{\rm IV}\leq
   4\inf_{ \epsilon>0}\left(2\epsilon+\frac{12}{\sqrt{\Gamma}}\int_{\epsilon}^{4b^2\beta^2}\sqrt{\log\mathcal{N}(\delta,\cL\circ\mathcal{T},\|\cdot\|_{L^\infty})}\text{ }d\delta\right).
\end{align*}
Notice that the covering number satisfies
 \begin{align*}
    \mathcal{N}(\delta,\cL\circ\cT,\|\cdot\|_{L^\infty})\leq \left(\frac{ 2^{L^2_{\rm T}+4}L_{\rm FFN}u^{3L_{\rm T}}d_{embed}^{18L_{\rm T}^2L_{\rm FFN}}\kappa^{6L_{\rm T}^2 L_{\rm FFN}+1}m_{\rm T}^{L_{\rm T}^2}\ell^{L_{\rm T}^2}b\beta}{\delta}\right)^{P_{\rm T}},
\end{align*}
where $P_{\rm T}=d_{embed}(D+1)+L_{\rm T}(3d_{embed}^2m_{\rm T}+L_{\rm FFN}w_{\rm FFN}^2)$.

By choosing $\epsilon=1/\sqrt{\Gamma}$ and plug in the network parameters in \eqref{parameter_low_rank_linear}, we get 
\begin{equation*}
    {\rm III}+{\rm IV}\leq O\left(\frac{D\sqrt{\log(nD\Gamma)}}{\sqrt{\Gamma}}\right),
\end{equation*}
where $O(\cdot)$ hides the dependency on $b,\beta$.
\vspace{2mm}
\\
$\bullet$ {\bf Error ${\rm V}$ and ${\rm VI}$. } 
By Lemma \ref{lemma:V}, we get
\begin{align*}
    {\rm V} + {\rm VI}\leq   4\sigma\sqrt{\log(\Gamma)}\cdot\inf_{ \epsilon>0}\left(2\epsilon+\frac{12}{\sqrt{\Gamma}}\int_{\epsilon}^{b\beta}\sqrt{\log\mathcal{N}(\delta,\cL'\circ\mathcal{T},\|\cdot\|_{L^\infty})}\text{ }d\delta\right).
\end{align*}
Notice that the covering number in this case satisfies
 \begin{align*}
    \mathcal{N}(\delta,\cL'\circ\cT,\|\cdot\|_{L^\infty})\leq \left(\frac{ 2^{L^2_{\rm T}+2}L_{\rm FFN}u^{3L_{\rm T}}d_{embed}^{18L_{\rm T}^2L_{\rm FFN}}\kappa^{6L_{\rm T}^2 L_{\rm FFN}+1}m_{\rm T}^{L_{\rm T}^2}\ell^{L_{\rm T}^2}}{\delta}\right)^{P_{\rm T}},
\end{align*}
where $P_{\rm T}=d_{embed}(D+1)+L_{\rm T}(3d_{embed}^2m_{\rm T}+L_{\rm FFN}w_{\rm FFN}^2)$.

By choosing $\epsilon=1/\sqrt{\Gamma}$ and plug in the network parameters in \eqref{parameter_low_rank_linear}, we get 
\begin{equation*}
    {\rm V}+{\rm VI}\leq O\left(\frac{D\log(nD\Gamma)}{\sqrt{\Gamma}}\right),
\end{equation*}
where $O(\cdot)$ hides the dependency on $b,\beta$.
\vspace{2mm}
\\
$\bullet$ {\bf Putting errors together. } 
Putting all the errors together, we get 
\begin{align*}
\cR(\hat\rmT)&\leq {\rm I}+{\rm II}+{\rm III}+{\rm IV}+{\rm V}+{\rm VI} \\
& \leq C_3\left(\frac{72C_1k\tau^2\beta^2b^6}{n}+\frac{16r\sigma^2\tau^2b^4}{n}\right)+C_4\left(\frac{D\log(nD\Gamma)}{\sqrt{\Gamma}}\right).
\end{align*}
By abusing the constants, we get 
\begin{equation*}
    \cR(\hat\rmT)\leq\left(\frac{72C_1k\tau^2\beta^2b^6}{n}+\frac{16C_2r\sigma^2\tau^2b^4}{n}\right)+C_3\left(\frac{D\log(nD\Gamma)}{\sqrt{\Gamma}}\right).
\end{equation*}
The constants $C_1,C_2$ hide the absolute constants, and $C_3$ hides the constants depending on $b,\beta$.

\end{proof}

\section{ICL Generalization Error for Multi-index Model} 
\label{sec:multiindex}

In this section, we study the approximation error ${\rm I}$ and ${\rm II}$ for the multi-index model $\cF_{{\rm MI},U}$ in the ICL setup. 
We first show that the transformer can implement the anisotropic kernel estimator $\cK(\fs)$ in \eqref{eq:NWkernelF7}, and then we bound its mean squared generalization error $\cR(\hat\rmT)$.

We make the following assumptions on the function (task) space and the distribution of the input.

\begin{assumption} [Function Class]
\label{assumpf_multi_index} Let  $\alpha\in (0,1]$, $R,L,b>0$, $\Omega= [-b,b]^D$. For a fixed $U\in\mathbb{R}^{D\times r}$ with orthonormal columns, we consider the multi-index function (task) space
\begin{align*}
\mathcal{F}=\cF_{{\rm MI},U}: = \{f: \Omega\to\mathbb{R}\text{ }|&\text{ } f(\xb) =g(U^\top \xb) \text{ for some } \alpha\text{-H\"older } g:\mathbb{R}^r\to\mathbb{R} \text{ with H\"older constant} \\
&\text{ no larger than } L, \text{ and } \|f\|_{L^{\infty}(\Omega)}\le R\}
\end{align*}
We let $\rho_f$ be a probability distribution on $\cF$. 
\end{assumption}

\begin{assumption}[Input Distribution]
\label{assumpx_multi_index}
Let $\Omega=[-b,b]^D$ with $b>0$, and let
$U\in\mathbb R^{D\times r}$ have orthonormal columns.
For any $\xb\in\Omega$ and $h>0$, define
$B_h(\xb):=
\left\{
\zb\in\Omega:
\|U^\top\zb-U^\top\xb\|_2\le h
\right\}$.
We assume that the input distribution $\rho_{\xb}$
is supported within $\Omega$ and that there exist uniform
constants $C_{\rho,1},C_{\rho,2},h_0>0$
such that
\[
C_{\rho,1}h^r
\le
\rho_{\xb}(B_h(\xb))
\le
C_{\rho,2}h^r
\]
for all $\xb\in\operatorname{supp}(\rho_{\xb})$
and $0<h\le h_0$.
\end{assumption}

Our goal next is to bound the errors ${\rm I}$ and ${\rm II}$ in Section \ref{subsec:multi_index_kernel_alg} and Section \ref{subsec:error_multi_index_kernel_alg} on the good event $\mathcal{E}$. Then we combine all the errors to get an error bound for $\cR(\hat{\rm T})$ in Section \ref{subsec:total_error_multi_index}.

\subsection{Transformer Implements Anisotropic Kernel Estimator} \label{subsec:multi_index_kernel_alg}

We first establish that the error term ${\rm II}$ in \eqref{eq:errde2} is zero by showing that the transformer can exactly implement anisotropic kernel estimator.

Mainly, we follow the idea in \citet{shen2025understanding}, but extend it to the anisotropic case. We first utilize the interaction mechanism of attention to construct several layers of Transformer. These transformer layers (from the first layer to the penultimate layer) take the input matrix
$H$ from \eqref{eq:Hmatrix} and produce the output matrix $H'$ as
\begin{align*}
H':= 
\begin{bmatrix}
    \xb_{1} & \cdots & \xb_{n} & \xb_{n+1}  \\
    y_1 & \cdots & y_n & 0 \\
    U^{\top}(\xb_{n+1}-\xb_1) & \cdots & U^{\top}(\xb_{n+1}-\xb_n) & 0  \\
   0 & \cdots & 0 & 0  \\
     \vdots & \vdots & \vdots & \vdots  \\
    0 & \cdots & 0 & 0  \\
    -\frac{\|U^{\top}(\xb_{n+1}-\xb_1)\|^2}{h^2} & \cdots & -\frac{\|U^{\top}(\xb_{n+1}-\xb_n)\|^2}{h^2} & 0 \\ 
    0 & \cdots & 0 & 0 \\
    \mathcal{I}_1 & \cdots & \cI_{n} & \cI_{n+1} \\
    1 & \cdots & 1 & 1
\end{bmatrix}\in\mathbb{R}^{(2D+6)\times (n+1)}.
\end{align*}

In the last layer of transformer, we use softmax activation (with a mask of size $n$). Then, we can construct $V=\eb_{D+1}\eb^{\top}_{D+1}$, and data kernels  
\begin{equation*}
Q^{data} = 
\left[\begin{array}{ccccccc|ccc} 
	0 & & & &  & & &  0 & 0 & 0 \\ 
     & \ddots & & & & & & \vdots & \vdots & \vdots \\
     & & 0 & & & & & 0 & 0 & 1 \\ 
	 & & & \ddots & & & & 0 & 0 & 0 \\ 
     & & &  & 0 & & & \vdots & \vdots & \vdots\\
     & & & & & 0 & & 0 & 0 & 0 \\
      & & & &  & & 0 & 0 & 0 & 0
\end{array}\right], \quad 
K^{data} = 
\left[\begin{array}{ccccccc|ccc} 
	0 & & & &  & & &  0 & 0 & 0 \\ 
     & \ddots & & & & & & \vdots & \vdots & \vdots \\
     & & 1 & & & & & 0 & 0 & 0 \\ 
	 & & & \ddots & & & & 0 & 0 & 0 \\ 
     & & &  & 0 & & & \vdots & \vdots & \vdots\\
     & & & & & 0 & & 0 & 0 & 0 \\
      & & & &  & & 0 & 0 & 0 & 0
\end{array}\right], 
\end{equation*}
where $Q^{data}\in\mathbb{R}^{(2D+3)\times(2D+6)}$ has $(2D+2)$-th position in the last column equals to $1$ and all the other entries are $0$, and $K^{data}\in\mathbb{R}^{(2D+3)\times(2D+6)}$ has $(2D+2,2D+2)$-th position equals to $1$ and all the other entries are $0$,
such that 
\begin{align*}
    [A(H')]_{n+1} &= \sum_{j=1}^{n}{\rm softmax}\left(\langle Q^{data}h'_{n+1}, K^{data}H'\rangle\right)_{j} Vh_{j} \\ 
    &= \sum_{j=1}^n\frac{y_je^{-\|U^{\top}(\mathbf{x}_{n+1}-\mathbf{x}_j)\|^2/h^2}}{\sum_{j=1}^n e^{-\|U^{\top}(\mathbf{x}_{n+1}-\mathbf{x}_j)\|^2/h^2}}\cdot \eb_{D+1}=\cK(\{\xb_i,y_i\}_{i=1}^n;\xb_{n+1})\cdot \eb_{D+1}= \cK(\fs)\cdot \eb_{D+1},
\end{align*}
where $\eb_{j}$ is the one-hot vector with entry $1$ at the $j$-th component.
Therefore, the residual attention outputs
\begin{align*}
    {\rm A}(H')+H':=    \begin{bmatrix}
    \xb_{1} & \cdots & \xb_{n} & \xb_{n+1}  \\
    y_1 & \cdots & y_n & \cK(\fs) \\
    U^{\top}(\xb_{n+1}-\xb_1) & \cdots & U^{\top}(\xb_{n+1}-\xb_n) & 0  \\
   0 & \cdots & 0 & 0  \\
     \vdots & \vdots & \vdots & \vdots  \\
    0 & \cdots & 0 & 0  \\
    -\frac{\|U^{\top}(\xb_{n+1}-\xb_1)\|^2}{h^2} & \cdots & -\frac{\|U^{\top}(\xb_{n+1}-\xb_n)\|^2}{h^2} & 0 \\ 
    0 & \cdots & 0 & 0 \\
    \mathcal{I}_1 & \cdots & \cI_{n} & \cI_{n+1} \\
    1 & \cdots & 1 & 1
\end{bmatrix}.
\end{align*}
Finally, the decoding layer outputs the desired element $\cK(\fs)$ in the $(D+1)$-th row and $(n+1)$-th column.

This connection between the attention mechanism and kernel estimator can be rigorously established. More precisely, we show that transformers can implement the anisotropic kernel estimator exactly, without any approximation error. We state this result as in Lemma \ref{lemma:approx}.  Such connection between anisotropic kernel and transformer can also be validated by numerical experiments in Section \ref{subsec:multi-index}. The complete proof of Lemma \ref{lemma:approx} is deferred to Appendix \ref{subsec:AppendixLemma1}. 

\begin{lemma} \label{lemma:approx}
Suppose the target function class satisfies Assumption \ref{assumpf_multi_index} and the input satisfies Assumption \ref{assumpx_multi_index}. Then under Setting \ref{setting:icl},  there exists a
transformer network ${\rm T}^*\in\mathcal{T}$ with parameters 
\begin{align*}
&d_{embed}=2D+6, \text{ } \ell=n+1, \text{ } L_{\rm T}=4, \text{ } m_{\rm T}=n,\\
&L_{\rm FFN}=O(1), \text{ } w_{\rm FFN}=2D+6, \text{ } \kappa=O\left({D^8n^2b^8R^4}/{h^8}\right)
\end{align*}
such that for any prompt $\fs_f$, it satisfies for the anisotropic oracle kernel $\cK$ defined in \eqref{eq:NWkernelF7},
\begin{equation}
    {\rm T}^*(\fs_f)=\cK(\fs_f).
\end{equation}
The notation $O(\cdot)$ hides absolute constants. 
\end{lemma}

\begin{remark} [Universality]
In Lemma \ref{lemma:approx_low_rank_linear} and Lemma \ref{lemma:approx}, the network architecture and weight parameters of the constructed transformer $\rmT^*$ are independent of different target function $f$ and input $\fs$. This construction demonstrates that the transformers can universally implement or approximate the corresponding kernel algorithm. 
\end{remark}

Lemma \ref{lemma:approx} establishes that the error term ${\rm II}$ in \eqref{eq:errde2} is zero. It is worthwhile to note that although our analysis focuses only on the multi-index model, the result extends naturally to a broader class of models, e.g., single-index model or isotropic regular function model, with little additional effort.

\subsection{Approximation Error of the Anisotropic Kernel Algorithm} \label{subsec:error_multi_index_kernel_alg}

As demonstrated in Lemma \ref{lemma:approx}, transformer can exactly implement the anisotropic kernel estimator with zero approximation error, i.e., error term ${\rm II}$ equal to zero. Our next result in Lemma \ref{lemma:kernel} shows the bound on error term ${\rm I}$, which is the error given by oracle kernel algorithm. We defer the proof of Lemma \ref{lemma:kernel} to Appendix \ref{subsec:AppendixProposition1}.

\begin{lemma}
\label{lemma:kernel}
Suppose the target function class satisfies Assumption \ref{assumpf_multi_index} and the input satisfies Assumption \ref{assumpx_multi_index}. Then under Setting \ref{setting:icl}, it satisfies 
\begin{equation}
\EE_{\fs_f}\left[\cR_{\fs_f}(\cK)\right]\leq C_3\left(\frac{\left[\log\left(h^{-1}\right)\right]^{1+\frac{3r}{2}}}{nh^{r}}+h^{2\alpha}[\log(h^{-1})]^2\right).
\label{eq:ThmApprox}
\end{equation}
The constant $C_3$ hides the constants depending on $r,b,D,L,R,C_{\rho,1},C_{\rho,2}$. 
\end{lemma}

\subsection{Transformer-based ICL Generalization Error Bound for Multi-index Model} \label{subsec:total_error_multi_index}

Building on the error bounds established in Sections~\ref{subsec:multi_index_kernel_alg} and~\ref{subsec:error_multi_index_kernel_alg}, we derive a bound on the mean squared generalization error for the multi-index model \eqref{multi-index} in the ICL setting.

\begin{theorem}
\label{Thm}
Suppose the target function class satisfies Assumption \ref{assumpf_multi_index} and the input satisfies Assumption \ref{assumpx_multi_index}. Then under Setting \ref{setting:icl}, if the transformer network class $\mathcal{T}$ is chosen with parameters 
\begin{align} \label{parameter_multi_index}
&d_{embed}=2D+6, \text{ } \ell=n+1, \text{ } L_{\rm T}=4, \text{ } m_{\rm T}=n, \\
\nonumber
&L_{\rm FFN}=O(1), \text{ } w_{\rm FFN}=2D+6, \text{ } \kappa=O\left(D^8n^{\frac{4\alpha+2r+8}{2\alpha+r}}b^8R^4\right),
\end{align}
where $O(\cdot)$ hides the dependency on the absolute constants.
Then the empirical risk minimizer $\hat\rmT$ defined in \eqref{eq:hatT} satisfies the mean squared generalization error bound:
 \begin{align}
 \label{eq:mainerrbound} \cR(\hat\rmT)\leq 
     C_1\left(n^{-\frac{2\alpha}{2\alpha+r}}\log^{1+\frac{3r}{2}} n\right)+C_2\left(\frac{nD^2\log(nD\Gamma)}{\sqrt{\Gamma}}\right),
 \end{align}
where the constant $C_1$ depends on $\alpha,r,b,D,L,R,C_{\rho,1},C_{\rho,2}$, and $C_2$ depends on $b,R$. 
\end{theorem}

\begin{proof} [Proof of Theorem \ref{Thm}]
Our proof techniques of Theorem \ref{Thm} proceeds with a 
\emph{bias-variance error decomposition} on a random test sample $\fs_f$, then taking the expectation gives the error terms ${\rm I}$, ${\rm II}$, ${\rm III}$, ${\rm IV}$, ${\rm V}$, ${\rm VI}$ in  \eqref{eq:errde1}, \eqref{eq:errde2}, \eqref{eq:errde4},
\eqref{eq:errde6}
respectively. 
\vspace{2mm}
\\
$\bullet$ {\bf Error ${\rm I}$ and ${\rm II}$. } 
By Lemma \ref{lemma:kernel}, we obtain 
\begin{align*}
    &{\rm I}+{\rm II}\leq O\left(\frac{\left[\log\left(h^{-1}\right)\right]^{1+\frac{3r}{2}}}{nh^{r}}+h^{2\alpha}[\log(h^{-1})]^2\right).
\end{align*}
\vspace{2mm}
\\
$\bullet$ {\bf Error ${\rm III}$ and ${\rm IV}$. } 
By Lemma \ref{lemma:III}, we get
\begin{align*}
   {\rm III}+{\rm IV}\leq
   4\inf_{ \epsilon>0}\left(2\epsilon+\frac{12}{\sqrt{\Gamma}}\int_{\epsilon}^{4R^2}\sqrt{\log\mathcal{N}(\delta,\cL\circ\mathcal{T},\|\cdot\|_{L^\infty})}\text{ }d\delta\right).
\end{align*}
Notice that the covering number satisfies
 \begin{align*}
    \mathcal{N}(\delta,\cL\circ\cT,\|\cdot\|_{L^\infty})\leq \left(\frac{ 2^{L^2_{\rm T}+4}L_{\rm FFN}u^{3L_{\rm T}}d_{embed}^{18L_{\rm T}^2L_{\rm FFN}}\kappa^{6L_{\rm T}^2 L_{\rm FFN}+1}m_{\rm T}^{L_{\rm T}^2}\ell^{L_{\rm T}^2}R}{\delta}\right)^{P_{\rm T}},
\end{align*}
where $P_{\rm T}=d_{embed}(D+1)+L_{\rm T}(3d_{embed}^2m_{\rm T}+L_{\rm FFN}w_{\rm FFN}^2)$.

By choosing $\epsilon=1/\sqrt{\Gamma}$ and plug into \eqref{parameter_multi_index}, we obtain
\begin{align*}
   {\rm III}+{\rm IV}\leq O\left(\frac{nD^2\sqrt{\log(nD\Gamma/h)}}{\sqrt{\Gamma}}\right),
\end{align*}
where the $O(\cdot)$ hides the dependency on $b,R$.
\vspace{2mm}
\\
$\bullet$ {\bf Error ${\rm V}$ and ${\rm VI}$. } 
By Lemma \ref{lemma:V}, we get
\begin{align*}
    {\rm V} + {\rm VI}
    \leq    4\sigma\sqrt{\log(\Gamma)}\cdot\inf_{ \epsilon>0}\left(2\epsilon+\frac{12}{\sqrt{\Gamma}}\int_{\epsilon}^{R}\sqrt{\log\mathcal{N}(\delta,\cL'\circ\mathcal{T},\|\cdot\|_{L^\infty})}\text{ }d\delta\right).
\end{align*}
Notice that the covering number in this case satisfies
 \begin{align*}
    \mathcal{N}(\delta,\cL'\circ\cT,\|\cdot\|_{L^\infty})\leq \left(\frac{ 2^{L^2_{\rm T}+2}L_{\rm FFN}u^{3L_{\rm T}}d_{embed}^{18L_{\rm T}^2L_{\rm FFN}}\kappa^{6L_{\rm T}^2 L_{\rm FFN}+1}m_{\rm T}^{L_{\rm T}^2}\ell^{L_{\rm T}^2}}{\delta}\right)^{P_{\rm T}},
\end{align*}
where $P_{\rm T}=d_{embed}(D+1)+L_{\rm T}(3d_{embed}^2m_{\rm T}+L_{\rm FFN}w_{\rm FFN}^2)$.

By choosing $\epsilon=1/\sqrt{\Gamma}$, we get 
\begin{equation*}
    {\rm V}+{\rm VI} \leq O\left(\frac{nD^2\log(nD\Gamma/h)}{\sqrt{\Gamma}}\right),
\end{equation*}
where $O(\cdot)$ hides the dependency on $b,R$.
\vspace{2mm}
\\
$\bullet$ {\bf Putting errors together. } 
Putting these errors together, we get
\begin{align*}
\cR(\hat\rmT)&\leq {\rm I}+{\rm II}+{\rm III}+{\rm IV}+{\rm V}+{\rm VI}
 \leq C_1\left(\frac{\left[\log\left(h^{-1}\right)\right]^{1+\frac{3r}{2}}}{nh^{r}}+h^{2\alpha}[\log(h^{-1})]^2\right)+C_2\left(\frac{nD^2\log(nD\Gamma/h)}{\sqrt{\Gamma}}\right).
\end{align*}
Finally, choosing $h=n^{-\frac{1}{2\alpha+r}}$ gives rise to \eqref{eq:mainerrbound} in Theorem \ref{Thm}.

\end{proof}

\section{Conclusion and Discussions} \label{sec:conclusion}

In this work, we demonstrated that transformers can act as
adaptive and anisotropic algorithmic learners in context,
extending the kernel perspective on ICL beyond the isotropic
setting considered in prior work. Our numerical experiments
revealed a strong alignment between the attention patterns
of trained transformers and the underlying anisotropic kernels,
suggesting that attention can adapt to the low-dimensional
and direction-dependent structure shared across tasks. Motivated by these empirical observations, we developed a theoretical framework for both low-rank linear regression
and anisotropic kernel regression. We established
generalization error bounds demonstrating that transformers
can implicitly implement these learning algorithms even
in the presence of noisy responses. We further provided explicit constructions of wide and shallow transformer
architectures that perform inference solely through prompt conditioning without parameter updates.
Together, our empirical and theoretical results suggest that attention goes beyond measuring similarities in the
raw input space: it can capture task-dependent geometric structure and exploit intrinsic low-dimensional information
that isotropic kernels may fail to capture.

\paragraph{Limitations and future directions.}
While our analysis establishes representational and generalization
guarantees for transformer-based ICL, it does not address whether gradient-based training can provably recover the proposed constructions.
Future research could investigate the optimization dynamics of
transformers, characterize the trade-off between depth and width, and extend the framework to settings with local task-dependent anisotropic structures.  
Such extensions would provide a more comprehensive understanding
of the adaptability and robustness of transformers as
in-context algorithmic learners.

\section*{Acknowledgments}
Zhaiming Shen and Wenjing Liao acknowledge support from the National Science Foundation under the NSF DMS 2145167 and the U.S. Department of Energy under the DOE SC0024348. Zhaiming Shen also acknowledges the support by SIAM Postdoctoral Support Program.  Alexander Cloninger acknowledges support from NSF CISE 2403452, NSF DMS 2608292, DARPA HR0011263E145, and a fellowship from Simons Foundation. Rongjie Lai acknowledges support from NSF DMS-2401297.
\vspace{0.01mm}
\\
The authors also acknowledge the use of AI-assisted tools
for language editing, figure refinement, literature searches,
and code development.
All mathematical results and proofs were independently
derived and verified by the authors, who take full
responsibility for the content of this manuscript.

\bibliographystyle{plainnat}
\bibliography{references}

@inproceedings{Havrilla24,
  author    = {Alex Havrilla and Wenjing Liao},
  title     = {Predicting Scaling Laws with Statistical and Approximation Theory for Transformer Neural Networks on Intrinsically Low-Dimensional Data},
  booktitle = {Advances in Neural Information Processing Systems},
  year      = {2024}
}

@book{shalev2014understanding,
  title={Understanding machine learning: From theory to algorithms},
  author={Shalev-Shwartz, Shai and Ben-David, Shai},
  year={2014},
  publisher={Cambridge university press}
}

@article{dudley1967sizes,
  title={The sizes of compact subsets of Hilbert space and continuity of Gaussian processes},
  author={Dudley, Richard M},
  journal={Journal of Functional Analysis},
  volume={1},
  number={3},
  pages={290--330},
  year={1967},
  publisher={Elsevier}
}

@article{hoeffding1994probability,
  title={Probability inequalities for sums of bounded random variables},
  author={Hoeffding, Wassily},
  journal={The collected works of Wassily Hoeffding},
  pages={409--426},
  year={1994},
  publisher={Springer}
}

@article{radford2019language,
  title={Language models are unsupervised multitask learners},
  author={Radford, Alec and Wu, Jeffrey and Child, Rewon and Luan, David and Amodei, Dario and Sutskever, Ilya and others},
  journal={OpenAI blog},
  volume={1},
  number={8},
  pages={9},
  year={2019}
}

@inproceedings{li2023transformers,
  title={Transformers as algorithms: Generalization and stability in in-context learning},
  author={Li, Yingcong and Ildiz, Muhammed Emrullah and Papailiopoulos, Dimitris and Oymak, Samet},
  booktitle={International conference on machine learning},
  pages={19565--19594},
  year={2023},
  organization={PMLR}
}

@inproceedings{bai2023transformers,
  title={Transformers as Statisticians: Provable In-Context Learning
         with In-Context Algorithm Selection},
  author={Bai, Yu and Chen, Fan and Wang, Huan
          and Xiong, Caiming and Mei, Song},
  booktitle={Advances in Neural Information Processing Systems},
  volume={36},
  pages={57125--57211},
  year={2023}
}

@article{garg2022can,
  title={What can transformers learn in-context? a case study of simple function classes},
  author={Garg, Shivam and Tsipras, Dimitris and Liang, Percy S and Valiant, Gregory},
  journal={Advances in Neural Information Processing Systems},
  volume={35},
  pages={30583--30598},
  year={2022}
}

@article{zhang2024trained,
  title={Trained transformers learn linear models in-context},
  author={Zhang, Ruiqi and Frei, Spencer and Bartlett, Peter L},
  journal={Journal of Machine Learning Research},
  volume={25},
  number={49},
  pages={1--55},
  year={2024}
}

@article{brown2020language,
  title={Language models are few-shot learners},
  author={Brown, Tom and Mann, Benjamin and Ryder, Nick and Subbiah, Melanie and Kaplan, Jared D and Dhariwal, Prafulla and Neelakantan, Arvind and Shyam, Pranav and Sastry, Girish and Askell, Amanda and others},
  journal={Advances in neural information processing systems},
  volume={33},
  pages={1877--1901},
  year={2020}
}

@article{liu2023pre,
  title={Pre-train, prompt, and predict: A systematic survey of prompting methods in natural language processing},
  author={Liu, Pengfei and Yuan, Weizhe and Fu, Jinlan and Jiang, Zhengbao and Hayashi, Hiroaki and Neubig, Graham},
  journal={ACM computing surveys},
  volume={55},
  number={9},
  pages={1--35},
  year={2023},
  publisher={ACM New York, NY}
}

@inproceedings{von2023transformers,
  title={Transformers learn in-context by gradient descent},
  author={Von Oswald, Johannes and Niklasson, Eyvind and Randazzo, Ettore and Sacramento, Jo{\~a}o and Mordvintsev, Alexander and Zhmoginov, Andrey and Vladymyrov, Max},
  booktitle={International Conference on Machine Learning},
  pages={35151--35174},
  year={2023},
  organization={PMLR}
}

@article{shen2025transformers,
  title={Transformers for Learning on Noisy and Task-Level Manifolds: Approximation and Generalization Insights},
  author={Shen, Zhaiming and Havrilla, Alex and Lai, Rongjie and Cloninger, Alexander and Liao, Wenjing},
  journal={arXiv preprint arXiv:2505.03205},
  year={2025}
}

@article{shen2025understanding,
  title={Understanding in-context learning on structured manifolds: Bridging attention to kernel methods},
  author={Shen, Zhaiming and Hsu, Alexander and Lai, Rongjie and Liao, Wenjing},
  journal={arXiv preprint arXiv:2506.10959},
  year={2025}
}

@inproceedings{akyurek2022learning,
  title={What Learning Algorithm Is In-Context Learning?
         Investigations with Linear Models},
  author={Aky{\"u}rek, Ekin and Schuurmans, Dale
          and Andreas, Jacob and Ma, Tengyu
          and Zhou, Denny},
  booktitle={International Conference on Learning Representations},
  year={2023},
  url={https://openreview.net/forum?id=0g0X4H8yN4I}
}

@article{cole2024provable,
  title={Provable in-context learning of linear systems and linear elliptic pdes with transformers},
  author={Cole, Frank and Lu, Yulong and O'Neill, Riley and Zhang, Tianhao},
  journal={arXiv preprint arXiv:2409.12293},
  year={2024}
}

@article{liao2019adaptive,
  title={Adaptive geometric multiscale approximations for intrinsically low-dimensional data},
  author={Liao, Wenjing and Maggioni, Mauro},
  journal={Journal of machine learning research},
  volume={20},
  number={98},
  pages={1--63},
  year={2019}
}

@ARTICLE{TransformerClassifier,

  author={Gurevych, Iryna and Kohler, Michael and Şahin, Gözde Gül},

  journal={IEEE Transactions on Information Theory}, 

  title={On the Rate of Convergence of a Classifier Based on a Transformer Encoder}, 

  year={2022},

  volume={68},

  number={12},

  pages={8139-8155}}

@article{videoworldsimulators2024,
  title={Video generation models as world simulators},
  author={Tim Brooks and Bill Peebles and Connor Holmes and Will DePue and Yufei Guo and Li Jing and David Schnurr and Joe Taylor and Troy Luhman and Eric Luhman and Clarence Ng and Ricky Wang and Aditya Ramesh},
  year={2024}
}

@article{achiam2023gpt,
  title={Gpt-4 technical report},
  author={Achiam, Josh and Adler, Steven and Agarwal, Sandhini and Ahmad, Lama and Akkaya, Ilge and Aleman, Florencia Leoni and Almeida, Diogo and Altenschmidt, Janko and Altman, Sam and Anadkat, Shyamal and others},
  journal={arXiv preprint arXiv:2303.08774},
  year={2023}
}

@article{devlin2018bert,
  title={Bert: Pre-training of deep bidirectional transformers for language understanding},
  author={Devlin, Jacob},
  journal={arXiv preprint arXiv:1810.04805},
  year={2018}
}

@inproceedings{Takakura23,
  author    = {Shokichi Takakura and Taiji Suzuki},
  title     = {Approximation and Estimation Ability of Transformers for Sequence-to-Sequence Functions with Infinite Dimensional Input},
  booktitle = {International Conference on Machine Learning},
  pages     = {33416--33447},
  year      = {2023},
  publisher = {PMLR}
}

@inproceedings{Vaswani17,
  author    = {Ashish Vaswani and Noam Shazeer and Niki Parmar and Jakob Uszkoreit and Llion Jones and Aidan N. Gomez and Lukasz Kaiser and Illia Polosukhin},
  title     = {Attention Is All You Need},
  booktitle = {Advances in Neural Information Processing Systems},
  year      = {2017}
}

@inproceedings{Yun19,
  author    = {Chulhee Yun and Srinadh Bhojanapalli and Ankit Singh Rawat and Sashank Reddi and Sanjiv Kumar},
  title     = {Are Transformers Universal Approximators of Sequence-to-Sequence Functions?},
  booktitle = {International Conference on Learning Representations},
  year      = {2019}
}

@article{nadaraya1964estimating,
  author  = {Nadaraya, E. A.},
  title   = {On Estimating Regression},
  journal = {Theory of Probability and Its Applications},
  volume  = {9},
  number  = {1},
  pages   = {141--142},
  year    = {1964},
  doi     = {10.1137/1109020}
}

@article{watson1964smooth,
  author  = {Watson, G. S.},
  title   = {Smooth Regression Analysis},
  journal = {Sankhyā: The Indian Journal of Statistics, Series A},
  volume  = {26},
  number  = {4},
  pages   = {359--372},
  year    = {1964},
  jstor   = {25049340}
}

@article{lu2025transformer,
  title={Transformer learns the cross-task prior and regularization for in-context learning},
  author={Lu, Fei and Yu, Yue},
  journal={arXiv preprint arXiv:2505.12138},
  year={2025}
}

@article{tropp2015matrix,
  author  = {Tropp, Joel A.},
  title   = {An Introduction to Matrix Concentration Inequalities},
  journal = {Foundations and Trends in Machine Learning},
  volume  = {8},
  number  = {1--2},
  pages   = {1--230},
  year    = {2015},
  doi     = {10.1561/2200000048}
}

@article{hsu2026understanding,
  title={Understanding In-Context Learning for Nonlinear Regression with Transformers: Attention as Featurizer},
  author={Hsu, Alexander and Shen, Zhaiming and Liao, Wenjing and Lai, Rongjie},
  journal={arXiv preprint arXiv:2605.05176},
  year={2026}
}

@article{cole2026incontext,
  author  = {Cole, Frank and Wang, Dixi and Chen, Yineng and Lu, Yulong and Lai, Rongjie},
  title   = {In-Context Operator Learning on the Space of Probability Measures},
  journal = {arXiv preprint arXiv:2601.09979},
  year    = {2026},
  eprint  = {2601.09979},
  archivePrefix = {arXiv},
  primaryClass  = {cs.LG}
}

@article{chiu2025context,
  title={In-Context Multi-Operator Learning with DeepOSets},
  author={Chiu, Shao-Ting and Nambiar, Aditya and Syed, Ali and Siegel, Jonathan W and Braga-Neto, Ulisses},
  journal={arXiv preprint arXiv:2512.16074},
  year={2025}
}

@article{yu2015useful,
  title={A useful variant of the {D}avis--{K}ahan theorem for statisticians},
  author={Yu, Yi and Wang, Tengyao and Samworth, Richard J.},
  journal={Biometrika},
  volume={102},
  number={2},
  pages={315--323},
  year={2015},
  publisher={Oxford University Press}
}

@article{davis1970rotation,
  title={The rotation of eigenvectors by a perturbation. {III}},
  author={Davis, Chandler and Kahan, William Morton},
  journal={SIAM Journal on Numerical Analysis},
  volume={7},
  number={1},
  pages={1--46},
  year={1970},
  publisher={SIAM}
}

@article{shi2026learning,
  title={Learning theory of transformers: Local-to-global approximation via softmax partition of unity},
  author={Shi, Zhongjie and Liao, Wenjing},
  journal={arXiv preprint arXiv:2605.08811},
  year={2026}
}

@article{baxter2000model,
  title={A model of inductive bias learning},
  author={Baxter, Jonathan},
  journal={Journal of Artificial Intelligence Research},
  volume={12},
  pages={149--198},
  year={2000}
}

@inproceedings{pontil2013excess,
  title={Excess risk bounds for multitask learning with trace norm regularization},
  author={Pontil, Massimiliano and Maurer, Andreas},
  booktitle={Proceedings of the 26th Annual Conference on Learning Theory (COLT)},
  pages={55--76},
  year={2013},
  publisher={PMLR}
}

@article{maurer2016benefit,
  title={The benefit of multitask representation learning},
  author={Maurer, Andreas and Pontil, Massimiliano and Romera-Paredes, Bernardino},
  journal={Journal of Machine Learning Research},
  volume={17},
  number={81},
  pages={1--32},
  year={2016}
}

@inproceedings{tripuraneni2021provable,
  title={Provable meta-learning of linear representations},
  author={Tripuraneni, Nilesh and Jin, Chi and Jordan, Michael},
  booktitle={Proceedings of the 38th International Conference on Machine Learning (ICML)},
  pages={10434--10443},
  year={2021},
  publisher={PMLR}
}

@inproceedings{du2021few,
  title={Few-shot learning via learning the representation, provably},
  author={Du, Simon S and Hu, Wei and Kakade, Sham M and Lee, Jason D and Lei, Qi},
  booktitle={Proceedings of the 9th International Conference on Learning Representations (ICLR)},
  year={2021}
}

@article{Jiao2026CoD,
  title={Transformers Can Overcome the Curse of Dimensionality:
         A Theoretical Study from an Approximation Perspective},
  author={Jiao, Yuling and Lai, Yanming and Wang, Yang
          and Yan, Bokai},
  journal={Journal of Machine Learning Research},
  volume={27},
  number={50},
  pages={1--34},
  year={2026}
}

@inproceedings{He2026,
  title={Approximation Error Upper and Lower Bounds
         for H{\"o}lder Class with Transformers},
  author={He, Xin and Jiao, Yuling and Lu, Xiliang
          and Yang, Jerry Zhijian},
  booktitle={Proceedings of the 43rd International
             Conference on Machine Learning},
  series={Proceedings of Machine Learning Research},
  volume={306},
  pages={41506--41536},
  year={2026}
}

@inproceedings{Cheng2024,
  title={Transformers Implement Functional Gradient Descent
         to Learn Non-Linear Functions In Context},
  author={Cheng, Xiang and Chen, Yuxin and Sra, Suvrit},
  booktitle={Proceedings of the 41st International
             Conference on Machine Learning},
  series={Proceedings of Machine Learning Research},
  volume={235},
  pages={8002--8037},
  year={2024}
}

@inproceedings{Kim2024Minimax,
  title={Transformers Are Minimax Optimal
         Nonparametric In-Context Learners},
  author={Kim, Juno and Nakamaki, Tai
          and Suzuki, Taiji},
  booktitle={Advances in Neural Information
             Processing Systems},
  year={2024}
}

@inproceedings{Sun2026,
  title={On the Role of Transformer Feed-Forward
         Layers in Nonlinear In-Context Learning},
  author={Sun, Haoyuan and Jadbabaie, Ali
          and Azizan, Navid},
  booktitle={Proceedings of the 37th International
             Conference on Algorithmic Learning Theory},
  series={Proceedings of Machine Learning Research},
  volume={313},
  year={2026}
}

@inproceedings{Ching2026,
  title={Efficient and Minimax Optimal In-context
         Nonparametric Regression with Transformers},
  author={Ching, Michelle and Popescu, Ioana
          and Smith, Nico and Ma, Tianyi
          and Underwood, William G.
          and Samworth, Richard J.},
  booktitle={Proceedings of the 43rd International
             Conference on Machine Learning},
  series={Proceedings of Machine Learning Research},
  volume={306},
  pages={19581--19610},
  year={2026}
}

@inproceedings{KimSuzuki2024,
  title={Transformers Learn Nonlinear Features In Context:
         Nonconvex Mean-field Dynamics on the Attention Landscape},
  author={Kim, Juno and Suzuki, Taiji},
  booktitle={Proceedings of the 41st International
             Conference on Machine Learning},
  series={Proceedings of Machine Learning Research},
  volume={235},
  pages={24527--24561},
  year={2024},
  publisher={PMLR},
  url={https://proceedings.mlr.press/v235/kim24af.html}
}

@inproceedings{Li2024Nonlinear,
  title={How Do Nonlinear Transformers Learn and
         Generalize in In-Context Learning?},
  author={Li, Hongkang and Wang, Meng and Lu, Songtao
          and Cui, Xiaodong and Chen, Pin-Yu},
  booktitle={Proceedings of the 41st International
             Conference on Machine Learning},
  series={Proceedings of Machine Learning Research},
  volume={235},
  pages={28734--28783},
  year={2024},
  publisher={PMLR},
  url={https://proceedings.mlr.press/v235/li24bn.html}
}

@article{shi2026transformers,
  title={Transformers as Cross-Task Learners:
         Shared Structure Drives Sample Efficiency
         in In-Context Learning},
  author={Shi, Zhongjie and Lai, Rongjie
          and Cloninger, Alexander and Liao, Wenjing},
  journal={arXiv preprint arXiv:2609.29060},
  year={2026},
  url={https://arxiv.org/abs/2609.29060}
}

\appendix 

\section{Notations} \label{sec:notation}

\begin{table}[h]
    \centering
    \caption{Summary of Notations}
    \vspace{2mm}
    \label{tab:notation}
    \begin{tabular}{cl}
        \toprule
        \textbf{Symbol} & \textbf{Description} \\
        \midrule
        $D$ & ambient dimension of the input \\
        $b$ & $L^2$/$L^{\infty}$ norm of the input $\xb\in\Omega$ for the low-rank linear model / multi-index model\\
        $\beta$ & $L^2$ norm of the parameter $\bbeta$ for the low-rank linear model \\ 
        $W$ & subspace in the low-rank linear model \\
        $U$ & subspace in the multi-index model \\
        $r$ & dimension of the subspace in the low-rank linear model and multi-index model \\
        $\tau$ & upper bound on the spectral norm of the pseudo-inverse of expected covariance matrix \\
        $\sigma$ & standard deviation of the Gaussian noise on the response \\
        $n$ & context length of the prompt $\fs$ or $\fs_f$ \\
        $\Gamma$ & number of tasks \\
        $A^{\dagger}$ & standard Moore-Penrose pseudo-inverse \\
    $[A]_k^{\dagger}$ & $k$-th order truncation pseudo-inverse \\
        \bottomrule
    \end{tabular}
\end{table}

\section{Connection to Other Algorithmic Kernel Functions} \label{sec:other_kernels}

   \paragraph{Ordinary linear model.}
It is immediate to see but still worth to mention that $\cF_{{\rm Linear},W}$ \eqref{eq:funlinearW} generalizes the ordinary linear function class (when $W=I\in\RR^{D\times D}$):
\begin{align}
\cF_{{\rm Linear}} := \{f: \text{ }f(\xb) =\wb^\top \xb \text{ for some } \wb \sim \rho_{\wb} \text{ in } \RR^D \text{ satisfying } \|\wb\|_{2}\leq w \}.
\label{eq:funlinear}
\end{align}
The interaction kernel for $\cF_{{\rm Linear}}$ can be represented as 
\begin{equation} \label{eq:linear_kernel}
    K(\xb_{n+1},\xb_i)=\xb_{n+1}^\top (X^{\top}X)^{\dagger}\xb_{i}.
\end{equation}
\paragraph{Low-rank linear feature model.}
Moreover, $\cF_{{\rm Linear},W}$ \eqref{eq:funlinearW} can be generalized to low-rank linear feature model associated with the feature map $\phi:\RR^D\to\RR^J$ on the input:
\begin{align}
\cF_{{\rm Linear},\phi,W} := \{f: \text{ }f(\xb) =\wb^\top \phi(\xb) \text{ for some } \wb = W \bbeta, \text{ with } \bbeta \sim \rho_{\bbeta} \text{ in } \RR^r \text{ satisfying } \|\bbeta\|_{2}\leq\beta \}.
\label{eq:funlinearW_phi}
\end{align}
The interaction kernel for $\cF_{{\rm Linear},\phi,W}$ can be represented as 
\begin{equation} \label{eq:linear_feature_kernel}
K(\xb_{n+1},\xb_i) = (W^\top \phi(\xb_{n+1}))^\top \underbrace{ (W^\top \Phi^\top \Phi W)^{\dagger}}_{r \times r} (W^\top \phi(\xb_i)),
\end{equation}
where the task subspace $W\in\RR^{J\times r}$ exists, but is unknown as a prior information. Here $\Phi$ is obtained by solving the linear system 
\begin{equation*}
\Phi \wb = \yb, \text{ where } \Phi_{ij} = \phi_j(\xb_i).
\label{eq:lsfeature}
\end{equation*}
We denote the vector $\phi_j(\{\xb_i\}_{i=1}^n) = [\phi_j(\xb_1),\ldots,\phi_j(\xb_n)]^\top \in \RR^{n}$  and then the matrix $\Phi$
 is \[\Phi =  [\phi_1(\{\xb_i\}_{i=1}^n),\ \ldots, \ \phi_J(\{\xb_i\}_{i=1}^n)] = \begin{bmatrix}
 \phi(\xb_1)^\top
 \\
 \vdots
 \\
  \phi(\xb_n)^\top
 \end{bmatrix}\in \RR^{n \times J}.\]

Note that the analysis of the low-rank linear feature model is essentially the same as that of the low-rank linear model, with the feature coordinates replacing the original input coordinates. In this paper, we focus on the study of the latter. In \citet{hsu2026understanding}, the generalization error analysis of using transformer to approximate the function class $\cF_{{\rm Linear},\phi,W}$ when the feature map $\phi$ is taken to be polynomial or spline are investigated.

\paragraph{Isotropic regular function.} 
It is straightforward to see that $\cF_{{\rm MI},U}$ \eqref{multi-index} generalize the H\"older function class \begin{align}
\cF_{\cH^\alpha} :=\{f:\Omega \rightarrow \RR \text{ }|\text{ } f \text{ is }\alpha\text{-H\"older with H\"older } \text{constant } L_{\cF} \text{ and } \|f\|_{L^\infty(\Omega)} \le R
\}. 
\label{eq:funholder}
\end{align}
For this function class, the empirical oracle kernel estimator becomes
\begin{equation} \label{eq:NWkernel}
    \cK(\fs) = \frac{\sum_{i=1}^n y_ie^{-\|\xb_{n+1}-\xb_i\|^2/h^2}}{\sum_{i=1}^n e^{-\|\xb_{n+1}-\xb_i\|^2/h^2}},
\end{equation}
and the corresponding interaction kernel is 
\begin{equation} \label{eq:isotropic_kernel}
K(\xb_{n+1},\xb_i)=\frac{e^{-\|\xb_{n+1}-\xb_i\|^2/h^2}}{\sum_{i=1}^n e^{-\|\xb_{n+1}-\xb_i\|^2/h^2}},
\end{equation}
with a proper choice of the kernel bandwidth $h$. In \citet{shen2025understanding}, the generalization error analysis of using transformer to approximate the function class $\cF_{\cH^\alpha}$ is investigated.

\section{Further Details on the Experiments} \label{sec:experiment_details}

For the transformer architecture used in the experiments for $\cF_{{\rm Linear},W}$ in Section \ref{subsec:low_rank_linear}, we fix a linear transformer with hyperparameters $d_{embed} = 8$,
$L_{\rm T}=2$, $L_{\rm FNN}=6$. The length of the prompt is chosen with $n=16$ and the number of attention heads is $m=1$ for each layer. We generate $\Gamma=50000$ for both training and testing. The model is trained with batch size $100$, using Adam with learning rate $0.001$ for $100$ epochs. The testing MSE for model in $\cF_{{\rm Linear},W}$ is $8.3\times 10^{-5}$, and relative error is $2.0\times 10^{-2}$.

For the transformer architecture used in the experiments for $\cF_{{\rm MI},U}$ in Section \ref{subsec:multi-index}, we fix a transformer (last layer with softmax activation and all previous layers with ReLU activations) with hyperparameters $d_{embed} = 8$,
$L_{\rm T}=4$, $L_{\rm FNN}=6$. The length of the prompt is chosen with $n=16$ and the number of attention heads is $m=1$ for each layer. We generate $\Gamma=50000$ for both training and testing. The model is trained with batch size $100$, using Adam with learning rate $0.001$ for $100$ epochs. The kernel bandwidth is fixed at $h=1$.
The testing MSE for the model in $\cF_{{\rm MI},U}$ is $2.3\times 10^{-3}$, and relative error is $8.2\times 10^{-2}$.

\section{Deferred Proofs} \label{sec:deferred_proofs}

\subsection{Proof of Lemma \ref{lemma:III}} \label{subsec:proof_stats1}

\begin{proof} [Proof of Lemma \ref{lemma:III}]
Notice that $|\cL_{\fs_f}(\rmT)|\leq 4R^2$. By \citet[Lemma 26.2]{shalev2014understanding} and Dudley entropy integral \citep{dudley1967sizes}  (Lemma \ref{lemmaRAD} in Appendix \ref{sec:Appendixfundlemma}), we get
\begin{align*}
   {\rm III}&\leq \EE_{\fS}\sup_{\rmT \in \cT}  \left[ \EE_{\fs_f}\left[\cR_{\fs_f}(\rmT)\right] - \hat{\cR}_{\fS_f}(\rmT) \right] = \EE_{\fS_f}\sup_{\rmT \in \cT}  \left[ \EE_{\fs_f}\left[\cR_{\fs_f}(\rmT)\right] - \hat{\cR}_{\fS_f}(\rmT) \right] \\
   &\leq 2\EE_{\fS_f}\left[ \Rad_{\fS_f}(\cL\circ\cT)\right] \leq 
   2\inf_{ \epsilon>0}\left(2\epsilon+\frac{12}{\sqrt{\Gamma}}\int_{\epsilon}^{4R^2}\sqrt{\log\mathcal{N}(\delta,\cL\circ\mathcal{T},\|\cdot\|_{L^\infty})}\text{ }d\delta\right).
\end{align*}
and
\begin{align*}
    {\rm IV}&\leq \EE_{\fS}\sup_{\rmT \in \cT}  \left[ \hat{\cR}_{\fS_f}(\rmT)-\EE_{\fs_f}\left[\cR_{\fs_f}(\rmT)\right] \right] = \EE_{\fS_f}\sup_{\rmT \in \cT}  \left[\hat{\cR}_{\fS_f}(\rmT)-\EE_{\fs_f}\left[\cR_{\fs_f}(\rmT)\right]  \right] \\
   &\leq 2\EE_{\fS_f}\left[ \Rad_{\fS_f}(\cL\circ\cT)\right] \leq 
   2\inf_{ \epsilon>0}\left(2\epsilon+\frac{12}{\sqrt{\Gamma}}\int_{\epsilon}^{4R^2}\sqrt{\log\mathcal{N}(\delta,\cL\circ\mathcal{T},\|\cdot\|_{L^\infty})}\text{ }d\delta\right).
\end{align*}
Therefore,
\begin{equation*}
     {\rm III}+{\rm IV} \leq 
   4\inf_{ \epsilon>0}\left(2\epsilon+\frac{12}{\sqrt{\Gamma}}\int_{\epsilon}^{4R^2}\sqrt{\log\mathcal{N}(\delta,\cL\circ\mathcal{T},\|\cdot\|_{L^\infty})}\text{ }d\delta\right)
\end{equation*}
as desired.
\end{proof}

\subsection{Proof of Lemma \ref{lemma:V}}
\label{subsec:proof_stats2}

\begin{proof} [Proof of Lemma \ref{lemma:V}]
For the statistical errors induced by noise in ${\rm V}$ and ${\rm VI}$, we have
\begin{align*}
    &\hat\cR_{\fS_f}(\hat\rmT) -\hat\cR_{\fS}(\hat\rmT)+\hat\cR_{\fS}(\rmT^*) -\hat\cR_{\fS_f}(\rmT^*) \\
    &= \frac{1}{\Gamma} \left[ \sum_{\gamma=1}^\Gamma \Big(\hat\rmT(\{\xb^\gamma_i,y^\gamma_i\}_{i=1}^n;\xb^\gamma_{n+1}) - f^\gamma(\xb^\gamma_{n+1})\Big)^2\right] -\frac{1}{\Gamma} \left[ \sum_{\gamma=1}^\Gamma \Big(\hat\rmT(\{\xb^\gamma_i,y^\gamma_i\}_{i=1}^n;\xb^\gamma_{n+1}) - f^\gamma(\xb^\gamma_{n+1})-\epsilon_{n+1}^{\gamma}\Big)^2\right] \\
    & + \frac{1}{\Gamma} \left[ \sum_{\gamma=1}^\Gamma \Big(\rmT^*(\{\xb^\gamma_i,y^\gamma_i\}_{i=1}^n;\xb^\gamma_{n+1}) - f^\gamma(\xb^\gamma_{n+1})-\epsilon_{n+1}^{\gamma}\Big)^2\right] -\frac{1}{\Gamma} \left[ \sum_{\gamma=1}^\Gamma \Big(\rmT^*(\{\xb^\gamma_i,y^\gamma_i\}_{i=1}^n;\xb^\gamma_{n+1}) - f^\gamma(\xb^\gamma_{n+1})\Big)^2\right] \\
    &=\frac{1}{\Gamma} \left[ \sum_{\gamma=1}^\Gamma 2\epsilon_{n+1}^{\gamma}\Big(\hat\rmT(\{\xb^\gamma_i,y^\gamma_i\}_{i=1}^n;\xb^\gamma_{n+1}) - f^\gamma(\xb^\gamma_{n+1})\Big)\right]-\frac{1}{\Gamma}\sum_{\gamma=1}^{\Gamma}(\epsilon_{n+1}^{\gamma})^2 \\
    &-\frac{1}{\Gamma} \left[ \sum_{\gamma=1}^\Gamma 2\epsilon_{n+1}^{\gamma}\Big(\rmT^*(\{\xb^\gamma_i,y^\gamma_i\}_{i=1}^n;\xb^\gamma_{n+1}) - f^\gamma(\xb^\gamma_{n+1})\Big)\right]+\frac{1}{\Gamma}\sum_{\gamma=1}^{\Gamma}(\epsilon_{n+1}^{\gamma})^2 \\
    & = \frac{1}{\Gamma} \left[ \sum_{\gamma=1}^\Gamma 2\epsilon_{n+1}^{\gamma}\Big(\hat\rmT(\{\xb^\gamma_i,y^\gamma_i\}_{i=1}^n;\xb^\gamma_{n+1}) - \rmT^*(\{\xb^\gamma_i,y^\gamma_i\}_{i=1}^n;\xb^\gamma_{n+1})\Big)\right].
\end{align*}
Taking expectation $\EE_{\fS}$ on both sides gives rise to 
\begin{align*}
    {\rm V}+{\rm VI}&=\EE_{\fS}\left[\hat\cR_{\fS_f}(\hat\rmT) -\hat\cR_{\fS}(\hat\rmT)+\hat\cR_{\fS}(\rmT^*) -\hat\cR_{\fS_f}(\rmT^*)\right] \\
    & = \frac{1}{\Gamma}\EE_{\fS} \left[ \sum_{\gamma=1}^\Gamma 2\epsilon_{n+1}^{\gamma}\Big(\hat\rmT(\{\xb^\gamma_i,y^\gamma_i\}_{i=1}^n;\xb^\gamma_{n+1}) - \rmT^*(\{\xb^\gamma_i,y^\gamma_i\}_{i=1}^n;\xb^\gamma_{n+1})\Big)\right] \\
    & = \frac{1}{\Gamma} \EE_{\fS}\left[ \sum_{\gamma=1}^\Gamma 2\epsilon_{n+1}^{\gamma}\Big(\hat\rmT(\{\xb^\gamma_i,y^\gamma_i\}_{i=1}^n;\xb^\gamma_{n+1})\Big)\right] \\
    &\leq \frac{1}{\Gamma} \EE_{\fS}\sup_{{\rm T}\in\cT}\left[ \sum_{\gamma=1}^\Gamma 2\epsilon_{n+1}^{\gamma}\Big(\rmT(\{\xb^\gamma_i,y^\gamma_i\}_{i=1}^n;\xb^\gamma_{n+1})\Big)\right]  \\
    &=\EE_{\fS} \frac{2}{\Gamma} \EE_{\{\epsilon_{n+1}^\gamma\}_{\gamma=1}^\Gamma}\sup_{{\rm T}\in\cT}\left[ \sum_{\gamma=1}^\Gamma \epsilon_{n+1}^{\gamma}\Big(\rmT(\{\xb^\gamma_i,y^\gamma_i\}_{i=1}^n;\xb^\gamma_{n+1})\Big)\right]  \\
    &=2\EE_{\fS}[\cG_{\fS}(\cL'\circ\cT)],
\end{align*}
where $\cL'(\rmT\circ\fs^{\gamma}):=\rmT(\{\xb^\gamma_i,y^\gamma_i\}_{i=1}^n;\xb^\gamma_{n+1})$, and $\cG_{\fS}(\cL'\circ\cT)$ is the Gaussian complexity defined as 
\begin{equation}
\cG_{\fS}(\cL'\circ\cT):=\frac{1}{\Gamma} \EE_{\{\epsilon_{n+1}^\gamma\}_{\gamma=1}^\Gamma}\sup_{{\rm T}\in\cT}\left[ \sum_{\gamma=1}^\Gamma \epsilon_{n+1}^{\gamma}\Big(\rmT(\{\xb^\gamma_i,y^\gamma_i\}_{i=1}^n;\xb^\gamma_{n+1})\Big)\right].
\end{equation}
    From the computation above,
\begin{align*}
    {\rm V} + {\rm VI}&= \EE_{\fS}\left[\hat\cR_{\fS_f}(\hat\rmT) -\hat\cR_{\fS}(\hat\rmT)+\hat\cR_{\fS}(\rmT^*) -\hat\cR_{\fS_f}(\rmT^*)\right] \\
    & \leq 
   2\EE_{\fS}[\cG_{\fS}(\cL'\circ\cT)]\leq 4\sigma\sqrt{\log(\Gamma)}\cdot\EE_{\fS}\left[ \Rad_{\fS}(\cL'\circ\cT)\right], 
\end{align*}
where the last inequality follows from Lemma \ref{lemma:Gaussian_complexity} in Appendix \ref{sec:Appendixfundlemma}. Notice that $|\cL'(\rmT\circ\fs^{\gamma})|\leq R$, by Dudley entropy integral \citep{dudley1967sizes}  (Lemma \ref{lemmaRAD} in Appendix \ref{sec:Appendixfundlemma}), we have
\begin{align*}
    \EE_{\fS}\left[ \Rad_{\fS}(\cL'\circ\cT)\right] \leq 
   \inf_{ \epsilon>0}\left(2\epsilon+\frac{12}{\sqrt{\Gamma}}\int_{\epsilon}^{R}\sqrt{\log\mathcal{N}(\delta,\cL'\circ\mathcal{T},\|\cdot\|_{L^\infty})}\text{ }d\delta\right).
\end{align*}
Therefore,
\begin{align*}
    {\rm V}+{\rm VI} \leq   4\sigma\sqrt{\log(\Gamma)}\cdot\inf_{ \epsilon>0}\left(2\epsilon+\frac{12}{\sqrt{\Gamma}}\int_{\epsilon}^{R}\sqrt{\log\mathcal{N}(\delta,\cL'\circ\mathcal{T},\|\cdot\|_{L^\infty})}\text{ }d\delta\right)
\end{align*}
as desired.
\end{proof}

\subsection{Proof of Lemma \ref{lemma:approx_low_rank_linear}}
\label{subsec:proof_approx_low_rank_linear}

\begin{proof}[Proof of Lemma \ref{lemma:approx_low_rank_linear}]
First, the transformer embed the sample $\fs=\{(\xb_{i},y_i)_{i=1}^n;\xb_{n+1}\}$ in such a way that
\begin{equation*} 
    ({\rm PE}+{\rm E})\left(\fs \right) = H = 
\begin{bmatrix}
    \xb_{1} & \cdots & \xb_{n} & \xb_{n+1}  \\
    y_1 & \cdots & y_n & 0 \\ 
    0 & \cdots & \cdots & 0 \\
    \vdots & \vdots & \vdots & \vdots & \\
    0 & \cdots & \cdots & 0 \\
    \mathcal{I}_1 & \cdots & \cdots & \cI_{n+1} \\
    1 & \cdots & \cdots & 1
\end{bmatrix}
\in \mathbb{R}^{d_{embed}\times \ell}.
\end{equation*}
Throughout the proof, we use $u=\|H\|_{\infty,\infty}$ to denote the largest entry in matrix $H$. Let us show the step-by-step construction of the $(D+1,n+1)$-th entry.

\vspace{5mm}
\noindent
$\bullet$ Implementation of $W^{\top}\xb_{i}$ for $i=1,\cdots,n+1$.
\\
\noindent
Let us construct a linear transformer block with single attention head
\begin{equation}
  {\rm A}(H)=VH((KH)^\top QH)=VHH^{\top}KQH:=  VHH^{\top}PH
\end{equation}
and a trivial feed-forward component so that it implements $W^{\top}\xb_{i}$ in the desired positions. Let us use the parameterization
\begin{equation*}
    V=\left[\begin{array}{ccccccc} 
0 & \cdots & \cdots & \cdots & 0 & \cdots & 0 \\
\vdots &  &  & & \vdots &  & \vdots \\
0 & \cdots & \cdots & \cdots & 0 & \cdots & 0 \\ 
 & & W^{\top} & & \vdots & & \vdots \\ 
     0 & \cdots & \cdots & \cdots & 0 & \cdots & 0 \\
	 \vdots & & & & \vdots & & \vdots  \\ 
    0 & \cdots & \cdots & \cdots & 0 & \cdots & 0
 \end{array}\right], \quad P=\left[\begin{array}{cc}
 \mathbf{I}_{D\times D} & \mathbf{0}_{D\times (D+6)} \\
 \mathbf{0}_{(D+6)\times D} & \mathbf{0}_{(D+6)\times (D+6)}
 \end{array}\right]
\end{equation*}
where both $V,P\in\mathbb{R}^{(2D+6)\times(2D+6)}$, and  $W^{\top}\in\mathbb{R}^{r\times D}$ appears in $(D+2)$-th row to $(D+2+r)$-th row of $V$.

A direct calculation shows that  the residual attention gives
\begin{equation*}
    {\rm A}(H)+H=\begin{bmatrix}
    \xb_{1} & \cdots & \xb_{n} & \xb_{n+1}  \\
    y_1 & \cdots & y_n & 0 \\ 
    W^{\top}\xb_1 & \cdots & W^{\top}\xb_{n} & W^{\top}\xb_{n+1} \\
    \vdots & \vdots & \vdots & \vdots & \\
    0 & \cdots & \cdots & 0 \\
    \mathcal{I}_1 & \cdots & \cdots & \cI_{n+1} \\
    1 & \cdots & \cdots & 1
\end{bmatrix}.
\end{equation*}
Let the feed-forward component implement the identity transformation. Then
\begin{equation*}
    H_1:={\rm B}_1(H)=\begin{bmatrix}
    \xb_{1} & \cdots & \xb_{n} & \xb_{n+1}  \\
    y_1 & \cdots & y_n & 0 \\ 
    W^{\top}\xb_1 & \cdots & W^{\top}\xb_{n} & W^{\top}\xb_{n+1} \\
    \vdots & \vdots & \vdots & \vdots & \\
    0 & \cdots & \cdots & 0 \\
    \mathcal{I}_1 & \cdots & \cdots & \cI_{n+1} \\
    1 & \cdots & \cdots & 1
\end{bmatrix}.
\end{equation*}

\noindent
In total, this process needs ${\rm B}_1\in\mathcal{B}(1,1,2D+6)$. The upper bound of the weights parameter in ${\rm B}_1$ is $\|\theta_{{\rm B}_1}\|_{\infty}\leq O(1)$. 

\vspace{5mm}
\noindent
$\bullet$ Implementation of $\cK(\fs)$ in the $(D+1)$-th row, $(n+1)$-th column.
\\
\noindent
Let us construct a linear transformer block with single attention head
\begin{equation}
  {\rm A}(H_1)=VH_1((KH_1)^\top QH_1)=VH_1H_1^{\top}K^{\top}QH_1:=  VH_1H_1^{\top}PH_1
\end{equation}
and trivial feed-forward component so that it implements $\cK(\fs)$ in the desired position. 

We use the parameterization
\begin{equation*}
    V=\left[\begin{array}{ccccccc} 
0 & \cdots & \cdots & \cdots & 0 & \cdots & 0 \\
\vdots &  &  & & \vdots &  & \vdots \\
	0 & \cdots & \cdots & \cdots & 0 & \cdots & 0 \\ 
    0 & \cdots & 0 & 1 & \vdots & & \vdots \\ 
     0 & \cdots & \cdots & \cdots & 0 & \cdots & 0 \\
	 \vdots & & & & \vdots & & \vdots  \\ 
    0 & \cdots & \cdots & \cdots & 0 & \cdots & 0
\end{array}\right]\in\mathbb{R}^{(2D+6)\times (2D+6)},
\end{equation*}
where the entry $1$ appears in the $(D+1,D+1)$-th position, all the other entries are $0$.

The matrix $P$ is
\begin{equation*}
P=\frac{1}{n}\begin{bmatrix}
    \mathbf{0}_{(D+1)\times (D+1)} & \mathbf{0}_{(D+1)\times r} & \mathbf{0}_{(D+1)\times (D+5-r)} \\
  \mathbf{0}_{r\times (D+1)} & P_0 & \mathbf{0}_{r\times (D+5-r)} \\
   \mathbf{0}_{(D+5-r)\times (D+1)} & \mathbf{0}_{(D+5-r)\times r} & \mathbf{0}_{(D+5-r)\times (D+5-r)}
\end{bmatrix}
\end{equation*}
with $P_0\in\mathbb{R}^{r\times r}$ to be determined below. A direct calculation shows that the $(D+1,n+1)$-th entry of the residual attention is
\begin{equation}\label{eq:finaloutput} [{\rm A}(H_1)]_{D+1,n+1}+[H_1]_{D+1,n+1}=\left(\sum_{i=1}^ny_i\zb_i^{\top}\right)\frac{P_0}{n}\zb_{n+1}=\zb^{\top}_{n+1}\frac{P_0}{n}Z^{\top}\yb.
\end{equation}
Rewrite $\eqref{eq:linear_low_rank_Ks}$ using $Z=XW$, we get 
$\cK(\fs) = \zb_{n+1}^{\top}[Z^{\top}Z]_k^{\dagger}Z^{\top}\yb=\frac{1}{n}\zb_{n+1}^{\top}[\hat\Sigma_{\zb}]_k^{\dagger}Z^{\top}\yb$. 

Suppose the eigen-decomposition of the covariance matrix $\hat\Sigma_{\zb}:=\frac{Z^{\top}Z}{n}$ can be written as $\hat\Sigma_{\zb}=\hat U\hat\Sigma \hat U^{\top}$,  where $\hat U=[\hat U_1,\hat U_2]$ with $\hat U_1$ the singular vectors correspond to the nonzero singular values of $[\hat\Sigma]_k^{\dagger}$, and $\hat U_2$ the singular vectors correspond to the zero singular values of $[\hat\Sigma]_k^{\dagger}$. Then it is easy to see $[\hat\Sigma_{\zb}]_k^{\dagger}=\hat U[\hat\Sigma]_k^{\dagger}\hat U^{\top}$, and $[\hat\Sigma_{\zb}]_k^{\dagger}(\hat\Sigma_{\zb})=[Z^{\top}Z]_k^{\dagger}(Z^{\top}Z)=\hat U_1\hat U_1^{\top}$, Then, we have
\begin{align}
\nonumber
    [{\rm A}(H_1)]_{D+1,n+1}+[H_1]_{D+1,n+1}-\cK(\fs)&=\zb_{n+1}^{\top}\left[\frac{P_0}{n}-[Z^{\top}Z]_k^{\dagger}\right]Z^{\top}\yb \\
    \nonumber
    &=\zb_{n+1}^{\top}\left[\frac{P_0}{n}-[Z^{\top}Z]_k^{\dagger}\right]Z^{\top}(Z\bbeta +\bepsilon)\\
    \nonumber
    &=\zb_{n+1}^{\top}\left[P_0\hat{\Sigma}_{\zb}-\hat U_1 \hat U_1^{\top}\right]\bbeta+\zb_{n+1}^{\top}\left[\frac{P_0}{n}-[Z^{\top}Z]_k^{\dagger}\right]Z^{\top}\bepsilon, \\
    &=\underbrace{\xb_{n+1}^{\top}W\left[P_0\hat{\Sigma}_{\zb}-\hat U_1\hat U_1^{\top}\right]\bbeta}_{\rm A}+\underbrace{\xb_{n+1}^{\top}W\left[\frac{P_0}{n}-[Z^{\top}Z]_k^{\dagger}\right]Z^{\top}\bepsilon}_{\rm B}.
    \label{eq:errorab}
\end{align}
Our goal next is to control the mean squared errors of ${\rm A}$ and ${\rm B}$ on the good event $\mathcal{E}$.
\vspace{2mm}
\\
\noindent
$\bullet$ {\bf Bounding error ${\rm A}$ on the event $\mathcal{E}$}.
\vspace{1mm}
\\
\noindent
To control the quantity above, the key is to control the quantity $P_0\hat{\Sigma}_{\zb}-\hat U_1\hat U_1^{\top}$. We set $P_0:=\Sigma_{\zb}^{\dagger}$ and consider the good event $\mathcal{E}$ defined in \eqref{eq:good_event}. On the good event $\mathcal{E}$, it satisfies
\begin{align*}
    \left\|P_0\hat{\Sigma}_{\zb}-\hat U_1\hat U_1^{\top}\right\|_2&=\|\Sigma_{\zb}^{\dagger}\hat{\Sigma}_{\zb}-\hat U_1\hat U_1^{\top}\|_2 \leq\|\Sigma_{\zb}^{\dagger}\hat{\Sigma}_{\zb}-\Sigma_{\zb}^{\dagger}\Sigma_{\zb}\|_2+\|\Sigma_{\zb}^{\dagger}\Sigma_{\zb}-\hat U_1\hat U_1^{\top}\|_2 \\ & \leq  \|\Sigma_{\zb}^{\dagger}\|_2\|\hat{\Sigma}_{\zb}-\Sigma_{\zb}\|_2+\|U_1 U_1^{\top}-\hat U_1\hat U_1^{\top}\|_2 \\
    &\leq \|\Sigma_{\zb}^{\dagger}\|_2\|\hat{\Sigma}_{\zb}-\Sigma_{\zb}\|_2+\frac{\|\hat{\Sigma}_{\zb}-\Sigma_{\zb}\|_2}{\lambda_k(\Sigma_{\zb})-\lambda_{k+1}(\hat\Sigma_{\zb})} \\
    &\leq \|\Sigma_{\zb}^{\dagger}\|_2\|\hat{\Sigma}_{\zb}-\Sigma_{\zb}\|_2+\frac{\|\hat{\Sigma}_{\zb}-\Sigma_{\zb}\|_2}{\|\Sigma_{\zb}^{\dagger}\|_2^{-1}-\|\hat{\Sigma}_{\zb}-\Sigma_{\zb}\|_2} \\
    &\leq 3\|\Sigma_{\zb}^{\dagger}\|_2\|\hat{\Sigma}_{\zb}-\Sigma_{\zb}\|_2.
\end{align*}
The second part of the third equality is due to Davis-Kahan Theorem (Lemma \ref{lemma:DavisKahan}).

To bound $\|\hat{\Sigma}_{\zb}-\Sigma_{\zb}\|_2$, we make use of the Intrinsic Matrix Bernstein inequality in \citet[Theorem 7.3.1]{tropp2015matrix} (Lemma \ref{lemma:intrinsic_bernstein} in Appendix \ref{sec:Appendixfundlemma}). Let $S_i=\frac{1}{n}(\zb_{i}\zb_i^{\top}-\Sigma_{\zb})$, so that $\hat{\Sigma}_{\zb}-\Sigma_{\zb}=\sum_{i=1}^n S_i$. For each $S_i$, we have the bound
\begin{align*}
    \|S_i\|_2 &\leq \frac{1}{n}(\|\zb_i\zb_i^{\top}\|_2+\|\Sigma_{\zb}\|_2)=\frac{1}{n}(\|\zb_i\zb_i^{\top}\|_2+\|\EE[\zb\zb^{\top}]\|_2) \\
    & \leq \frac{1}{n}(\|\zb_i\zb_i^{\top}\|_2+\EE\|[\zb\zb^{\top}]\|_2)\leq\frac{1}{n}(\|\zb_i\|^2_2+\EE\|\zb\|^2_2)\leq \frac{2b^2}{n}.
\end{align*}
Notice that the random matrices $S_i$ are  identically and independently  distributed, and centered. A direct calculation shows that 
\begin{align*}
    \EE\left[(\hat{\Sigma}_{\zb}-\Sigma_{\zb})^2\right]&=\EE\left[\left(\sum_{i=1}^n S_i\right)^2\right]=\sum_{i=1}^n\EE \left[S^2_i\right] \\
    & =\frac{1}{n^2}\sum_{i=1}^n\EE\left[\|\zb_i\|^2_2\cdot\zb_i\zb_i^{\top}-(\zb_i\zb_i^{\top})\Sigma_{\zb}-\Sigma_{\zb}(\zb_i\zb_i^{\top})+\Sigma_{\zb}^2\right] \\
    & \preccurlyeq \frac{1}{n^2}\sum_{i=1}^n\left[ b^2\EE[\zb_i\zb_i^{\top}]-\Sigma_{\zb}^2-\Sigma_{\zb}^2+\Sigma_{\zb}^2\right]\preccurlyeq \frac{1}{n^2}\sum_{i=1}^n [b^2\Sigma_{\zb}]=\frac{b^2\Sigma_{\zb}}{n},
\end{align*}
where the last inequality is due to dropping the negative-semidefinite term $-\Sigma_{\zb}^2$. Furthermore, direct calculation shows that
$\|\EE[(\hat{\Sigma}_{\zb}-\Sigma_{\zb})^2]\|_2\leq \frac{b^2}{n}\|\Sigma_{\zb}\|_2\leq \frac{b^4}{n}$. The matrix intrinsic dimension of $V$ satisfies $d={\rm intdim}(V)=\frac{2{\rm tr}(\Sigma_{\zb})}{\|\Sigma_{\zb}\|_2}\leq 2k$. Here the matrix $V=\begin{bmatrix}
    \frac{b^2\Sigma_{\zb}}{n} & 0 \\
    0 & \frac{b^2\Sigma_{\zb}}{n}
\end{bmatrix}$ and the matrix intrinsic dimension for any matrix $A$ is defined as ${\rm intdim}(A):=\frac{{\rm tr}(A)}{\|A\|_2}$.   For more detailed discussion of the matrix intrinsic dimension, we refer to \citet[Chapter 7]{tropp2015matrix}.  

By \citet[Theorem 7.3.1]{tropp2015matrix} (Lemma \ref{lemma:intrinsic_bernstein}), we have for all $t\geq \frac{b^2}{\sqrt{n}}+\frac{2b^2}{3n}$,
\begin{equation*}
    \mathbb{P}(\|\hat{\Sigma}_{\zb}-\Sigma_{\zb}\|_2\geq t)\leq 8ke^{-\frac{3nt^2}{6b^4+4b^2t}}.
\end{equation*}
Notice that $\|\hat{\Sigma}_{\zb}-\Sigma_{\zb}\|$ is upper bounded as 
$\|\hat{\Sigma}_{\zb}-\Sigma_{\zb}\|_2=\left\|\frac{Z^{\top}Z}{n}-\Sigma_{\zb}\right\|_2\leq \left\|\frac{Z^{\top}Z}{n}\right\|_2+\|\Sigma_{\zb}\|_2\leq 2b^2$.
Using the identity $\mathbb{E}[X^2]=\int_0^{2b^2}2t\mathbb{P}(X\geq t)dt$, we get 
\begin{align*}
    \EE\|\hat{\Sigma}_{\zb}-\Sigma_{\zb}\|^2_2&=\int_0^{\frac{b^2}{\sqrt{n}}+\frac{2b^2}{3n}}2tdt+\int_{\frac{b^2}{\sqrt{n}}+\frac{2b^2}{3n}}^{2b^2}2t\cdot8ke^{-\frac{3nt^2}{6b^4+4b^2t}}dt \leq \left(\frac{b^2}{\sqrt{n}}+\frac{2b^2}{3n}\right)^2 +\int_{\frac{b^2}{\sqrt{n}}+\frac{2b^2}{3n}}^{2b^2}2t\cdot8ke^{-\frac{3nt^2}{14b^4}}dt \\
    &\leq \frac{b^4}{n}+\frac{4b^4}{3n^{3/2}}+\frac{4b^4}{9n^2}+\frac{112kb^4}{3n}\leq \frac{C_1kb^4}{n},
\end{align*}
for some $C_1$ only depends on the absolute constants.

Therefore, on the good event $\mathcal{E}$, it satisfies
\begin{align*}    \EE\left[\|P_0\hat{\Sigma}_{\zb}-\hat{U}_1\hat{U}_1^{\top}\|^2_2\rvert \mathcal{E}\right]&=\EE\left[\|\Sigma_{\zb}^{\dagger}\hat{\Sigma}_{\zb}-\hat{U}_1\hat{U}_1^{\top}\|^2_2\rvert \mathcal{E}\right] \\
&\leq 9\|\Sigma_{\zb}^{\dagger}\|^2_2\cdot\EE\left[\|\hat{\Sigma}_{\zb}-\Sigma_{\zb}\|^2_2 | \mathcal{E}\right]\leq 9\|\Sigma_{\zb}^{\dagger}\|^2_2\cdot\EE\left[\|\hat{\Sigma}_{\zb}-\Sigma_{\zb}\|^2_2 \right]\leq\frac{9C_1k\tau^2b^4}{n}.
\end{align*}
Therefore, on the good event $\mathcal{E}$, the error ${\rm A}$ satisfies
\begin{align*}
    \EE\left[\|\xb_{n+1}^{\top}W\left[P_0\hat{\Sigma}_{\zb}-\hat U_1\hat U_1^{\top}\right]\bbeta\|^2_2\rvert \mathcal{E}\right]&\leq \sup_{\xb_{n+1}}\|\xb_{n+1}^{\top}\|^2_2\cdot\sup_{W}\|W\|^2_2\cdot\EE_{\fs_f}\left[\|\Sigma_{\zb}^{\dagger}\hat{\Sigma}_{\zb}-U_1\hat U_1\|^2_2\rvert \mathcal{E}\right]\cdot\sup_{\bbeta}\|\bbeta\|^2_2 \\
    &\leq b^2\cdot\frac{9C_1k\tau^2b^4}{n}\cdot \beta^2=\frac{9C_1k\tau^2\beta^2b^6}{n}
\end{align*}
\vspace{2mm}
\\
\noindent
$\bullet$ {\bf Bounding error ${\rm B}$  in \eqref{eq:errorab} on the event $\mathcal{E}$}.
\vspace{1mm}
\\
\noindent
To bound $B$, the key is to bound $\frac{P_0}{n}-[Z^{\top}Z]_k^{\dagger}=\frac{1}{n}(\Sigma_{\zb}^{\dagger}-[\hat\Sigma_{\zb}]_k^{\dagger})$. Notice $[\hat\Sigma_{\zb}]_k$ is the best rank $k$ approximation of $\hat\Sigma_{\zb}$, so $\|\hat\Sigma_{\zb}-[\hat\Sigma_{\zb}]_k\|_2=\lambda_{k+1}(\hat\Sigma_{\zb})$, and by Weyl's inequality, 
\begin{equation*}
    |\lambda_{k+1}(\hat\Sigma_{\zb})-\lambda_{k+1}(\Sigma_{\zb})|=|\lambda_{k+1}(\hat\Sigma_{\zb})|=\|\hat\Sigma_{\zb}-[\hat\Sigma_{\zb}]_k\|_2\leq\|\hat\Sigma_{\zb}-\Sigma_{\zb}\|_2.
\end{equation*}
Therefore
\begin{equation*}
    \|\Sigma_{\zb}-[\hat\Sigma_{\zb}]_k\|_2\leq \|\Sigma_{\zb}-\hat\Sigma_{\zb}\|_2+\|\hat\Sigma_{\zb}-[\hat\Sigma_{\zb}]_k\|_2\leq 2\|\Sigma_{\zb}-\hat\Sigma_{\zb}\|_2.
\end{equation*}
By Lemma \ref{lemma:pesudo-inverse}, we can bound the perturbation of pseudo-inverse by the perturbation of the original matrix.
\begin{align} \label{eq:pesudo-inverse}
    \nonumber
    \|\Sigma_{\zb}^{\dagger}-[\hat\Sigma_{\zb}]_k^{\dagger}\|_2 &\leq 3\max\{\|\Sigma_{\zb}^{\dagger}\|^2_2, \|[\hat\Sigma_{\zb}]_k^{\dagger}\|^2_2\}\|\Sigma_{\zb}-[\hat\Sigma_{\zb}]_k\|_2 \\
    & \leq 6\max\{\|\Sigma_{\zb}^{\dagger}\|^2_2, \|[\hat\Sigma_{\zb}]_k^{\dagger}\|^2_2\}\|\Sigma_{\zb}-\hat\Sigma_{\zb}\|_2.
\end{align}
Notice $\|\Sigma_{\zb}-\hat\Sigma_{\zb}\|_2\leq t$ with probability at least $1-8ke^{-\frac{3nt^2}{6b^4+4b^2t}}$ for $t\in \left[\frac{b^2}{\sqrt{n}}+\frac{2b^2}{3n},\frac{1}{4\tau}\right]$. Then, for $\frac{b^2}{\sqrt{n}}+\frac{2b^2}{3n}\leq t\leq\frac{1}{4\tau}$, it satisfies $\lambda_k([\hat\Sigma_{\zb}]_k)\geq \lambda_k(\Sigma_{\zb})-2t\geq \frac{1}{\tau}-2t=\frac{1-2t\tau}{\tau}$ with probability at least $1-8ke^{-\frac{3nt^2}{6b^4+4b^2t}}$. Hence $\|[\hat\Sigma_{\zb}]_k^{\dagger}\|_2=\frac{1}{\lambda_k([\hat\Sigma_{\zb}]_k)}\leq\frac{\tau}{1-2t\tau}$ with probability at least $1-8ke^{-\frac{3nt^2}{6b^4+4b^2t}}$ for $t\in \left[\frac{b^2}{\sqrt{n}}+\frac{2b^2}{3n},\frac{1}{4\tau}\right]$.

Therefore, on the good event $\mathcal{E}$, it satisfies
\begin{align*}
    \EE\left[\left\|\frac{P_0}{n}-[Z^{\top}Z]_k^{\dagger}\right\|^2_2| \mathcal{E} \right]&= \frac{1}{n^2}\cdot\EE\left[\|\Sigma_{\zb}^{\dagger}-[\hat{\Sigma}_{\zb}]_k^{\dagger}\|^2_2|\mathcal{E}\right] \\
    & = \frac{1}{n^2}\cdot\EE\left[\left(6\max\{\|\Sigma_{\zb}^{\dagger}\|^2_2, \|[\hat\Sigma_{\zb}]_k^{\dagger}\|^2_2\}\|\Sigma_{\zb}-\hat\Sigma_{\zb}\|_2|\mathcal{E}\right)^2\right] \\
    &\leq \frac{1}{n^2}\cdot 36(2\tau)^4\cdot\EE\left[\left(\|\Sigma_{\zb}-\hat\Sigma_{\zb}\|_2|\mathcal{E}\right)^2\right] \\
    & \leq \frac{1}{n^2}\cdot 36\left(2\tau\right)^4 \frac{C_1kb^4}{n}=\frac{C_2k\tau^4b^4}{n^3},
\end{align*}
for some $C_2$ only depends on the absolute constants.

Therefore, on the good event $\mathcal{E}$, the error ${\rm B}$ satisfies
\begin{align*}
    \EE\left[\left\|\xb_{n+1}^{\top}W\left[\frac{P_0}{n}-[Z^{\top}Z]_k^{\dagger}\right]Z^{\top}\bepsilon\right\|_2^2\right]&\leq \sup_{\xb_{n+1}}\|\xb_{n+1}\|^2_2\cdot\sup_{W}\|W\|^2_2\cdot\EE\left[\left\|\frac{P_0}{n}-[Z^{\top}Z]_k^{\dagger}\right\|^2_2|\mathcal{E}\right]\EE\|Z^{\top}\bepsilon\|^2_2 \\
    & \leq b^2\cdot\frac{C_2k\tau^4b^4}{n^3}\cdot rnb^2\sigma^2=\frac{C_2kr\sigma^2\tau^4b^8}{n^2}.
\end{align*}
Here we have used the fact $\EE_{\fs_f}\|Z^{\top}\bepsilon\|_2^2\leq rnb^2\sigma^2$ for Gaussian noise vector $\bepsilon\sim N(0,\sigma^2)$.
\vspace{2mm}
\\
\noindent
$\bullet$ {\bf Combining ${\rm A}$ and ${\rm B}$ together on both $\mathcal{E}$.} 
\vspace{1mm}
\\
\noindent
Notice that
\begin{align} 
\nonumber
   \cR_{\fs_f}(\rmT^*)-2\cR_{\fs_f}(\cK) &= \left([{\rm A}(H_1)]_{D+1,n+1}+[H_1]_{D+1,n+1}-f(\xb_{n+1})\right)^2-2(\cK(\fs)-f(\xb_{n+1}))^2 \\
   \nonumber
   &\leq 2\left([{\rm A}(H_1)]_{D+1,n+1}+[H_1]_{D+1,n+1}-\cK(\fs)\right)^2 + 2(\cK(\fs)-f(\xb_{n+1}))^2 - 2(\cK(\fs)-f(\xb_{n+1}))^2 \\
   \nonumber
   &= 2\left([{\rm A}(H_1)]_{D+1,n+1}+[H_1]_{D+1,n+1}-\cK(\fs)\right)^2 \\
   \nonumber
   &\leq 4\left(\xb_{n+1}^{\top}W\left[P_0\hat{\Sigma}_{\zb}-U_1\hat U_1\right]\bbeta\right)^2+4\left(\xb_{n+1}^{\top}W\left[\frac{P_0}{n}-[Z^{\top}Z]_k^{\dagger}\right]Z^{\top}\bepsilon\right)^2.
\end{align}
Therefore, combining the error for ${\rm A}$ and ${\rm B}$ together, we get
\begin{align*}   \EE_{\fs_f}\left[\cR_{\fs_f}(\rmT^*)-2\cR_{\fs_f}(\cK)|\mathcal{E}\right]\leq \frac{36C_1k\tau^2\beta^2b^6}{n}+\frac{4C_2kr\sigma^2\tau^4b^8}{n^2}
\end{align*}
for some $C_1,C_2$ which only depend on the absolute constants.

Let the feed-forward component of this layer be trivial, then it needs ${\rm B}_2\in\mathcal{B}(1,1,2D+6)$. By Assumption \ref{assumpx_low_rank_linear}, we have $\|\Sigma_{\zb}^{\dagger}\|_{\infty,\infty}\leq\|\Sigma_{\zb}^{\dagger}\|_2\leq \tau$. Therefore, the upper bound of the weights parameter in ${\rm B}_2$ is $\|\theta_{{\rm B}_2}\|_{\infty}\leq O(\tau)$. 

\end{proof}

\subsection{Proof of Lemma \ref{lemma:low_rank_lsq}}
\label{subsec:proof_low_rank_lsq}

\begin{proof} [Proof of Lemma \ref{lemma:low_rank_lsq}]

On the good event $\mathcal{E}$, it satisfies 
        \begin{align*}
        \EE_{\fs_f}\left[\cR_{\fs_f}(\cK)|\mathcal{E}\right] &= \EE_{\fs_f}(\cK(\fs_f)-f(\xb_{n+1})|\mathcal{E})^2 = \EE_{\fs_f}\left[\zb_{n+1}^{\top}[Z^{\top}Z]_k^{\dagger}Z^{\top}\yb-\zb_{n+1}^{\top}\bbeta)|\mathcal{E}\right]^2 \\
        &=\EE_{\fs_f}\left[\zb_{n+1}^{\top}[Z^{\top}Z]_k^{\dagger}Z^{\top}(Z\bbeta+\bepsilon)-\zb_{n+1}^{\top}\bbeta)|\mathcal{E}\right]^2
        \\
        &\leq 2\EE_{\fs_f}\left[\|\zb_{n+1}^{\top}(\hat U_1\hat U_1^{\top}-I)\bbeta\|_2^2|\mathcal{E}\right]+2\EE_{\fs_f}\left[\|\zb_{n+1}^{\top}[Z^{\top}Z]_k^{\dagger}Z^{\top}\bepsilon\|^2_2|\mathcal{E}\right] \\
        &\leq 2\EE_{\fs_f}\left[\|\zb_{n+1}^{\top}(\hat U_1\hat U_1^{\top}-U_1U_1^{\top})\bbeta\|_2^2|\mathcal{E}\right]+2\EE_{\fs_f}\left[\|\zb_{n+1}^{\top}(U_1U_1^{\top}-I)\bbeta\|_2^2|\mathcal{E}\right] \\
        &+2\EE_{\fs_f}\left[\|\zb_{n+1}^{\top}[Z^{\top}Z]_k^{\dagger}Z^{\top}\bepsilon\|^2_2|\mathcal{E}\right] \\ 
        & \leq 2\EE_{\fs_f}\left[\|\zb_{n+1}^{\top}(\hat U_1\hat U_1^{\top}-U_1U_1^{\top})\bbeta\|_2^2|\mathcal{E}\right]+2\EE_{\fs_f}\left[\|\zb_{n+1}^{\top}(U_2U_2^{\top})\bbeta\|_2^2|\mathcal{E}\right] \\
        & +2\EE_{\fs_f}\left[\|\zb_{n+1}^{\top}[Z^{\top}Z]_k^{\dagger}Z^{\top}\bepsilon\|^2_2|\mathcal{E}\right] \\ 
        & = 2\EE_{\fs_f}\left[\|\zb_{n+1}^{\top}(\hat U_1\hat U_1^{\top}-U_1U_1^{\top})\bbeta\|_2^2|\mathcal{E}\right]+2\EE_{\fs_f}\left[\|\zb_{n+1}^{\top}[Z^{\top}Z]_k^{\dagger}Z^{\top}\bepsilon\|^2_2|\mathcal{E}\right] \\ 
        &\leq 2\sup_{\xb_{n+1}}\|\xb_{n+1}^{\top}\|^2_2\cdot\sup_{W}\|W\|_2^2 \cdot \EE_{\fs_f}\left[\|\hat U_1\hat U_1^{\top}-U_1U_1^{\top}\|_2^2|\mathcal{E}\right]\cdot\sup_{\bbeta}\|\bbeta\|_2^2 \\
        & +2\sup_{\xb_{n+1}}\|\xb_{n+1}^{\top}\|^2_2\cdot\sup_W\|W\|_2^2\cdot\frac{1}{n^2}\cdot\EE_{\fs_f}\left[\|[\hat\Sigma_{\zb}]_k^{\dagger}\|_2^2|\mathcal{E}\right]\cdot\EE_{\fs_f}\|Z^{\top}\bepsilon\|_2^2  \\
        &\leq 2b^2\cdot\frac{9C_1k\tau^2b^4}{n}\cdot\beta^2+2b^2\cdot\frac{1}{n^2}\cdot(2\tau)^2\cdot rnb^2\sigma^2 \\
        &\leq \frac{18C_1k\tau^2\beta^2b^6}{n}+\frac{8r\sigma^2\tau^2b^4}{n}.
    \end{align*}

    In the third inequality above, we have used the fact $I=U_1U_1^{\top}+U_2U_2^{\top}$. In the fourth equality above, we have used fact $\EE_{\fs_f}\left[\|\zb_{n+1}^{\top}(U_2U_2^{\top})\bbeta\|_2^2|\mathcal{E}\right]=\EE_{\fs_f}\left[\bbeta^{\top}(U_2U_2^{\top})\zb_{n+1}\zb_{n+1}^{\top}(U_2U_2^{\top})\bbeta|\mathcal{E}\right]=\bbeta^{\top}(U_2U_2^{\top})\EE_{\fs_f}[\zb_{n+1}\zb_{n+1}^{\top}|\mathcal{E}](U_2U_2^{\top})\bbeta=\bbeta^{\top}(U_2U_2^{\top})\Sigma_{\zb}(U_2U_2^{\top})\bbeta=0$.
\end{proof}

\subsection{Proof of Lemma \ref{lemma:approx}}
\label{subsec:AppendixLemma1}

\begin{proof} [Proof of Lemma \ref{lemma:approx}]
    
First, the transformer embed the sample $\fs=\{(\xb_{i},y_i)_{i=1}^n;\xb_{n+1}\}$ in such a way that
\begin{equation*} 
    ({\rm PE}+{\rm E})\left(\fs \right) = H = 
\begin{bmatrix}
    \xb_{1} & \cdots & \xb_{n} & \xb_{n+1}  \\
    y_1 & \cdots & y_n & 0 \\ 
    0 & \cdots & \cdots & 0 \\
    \vdots & \vdots & \vdots & \vdots & \\
    0 & \cdots & \cdots & 0 \\
    \mathcal{I}_1 & \cdots & \cdots & \cI_{n+1} \\
    1 & \cdots & \cdots & 1
\end{bmatrix}
\in \mathbb{R}^{d_{embed}\times \ell}.
\end{equation*}
Throughout the proof, we use $u=\|H\|_{\infty,\infty}$ to denote the largest entry in matrix $H$, and we use $h_i$ to denote the $i$-th column/token in $H$. We now present a step-by-step construction of 
$\cK(\fs)$, guided by our fundamental lemmas provided in Appendix \ref{sec:Appendixfundlemma}.

\vspace{5mm}
\noindent
$\bullet$ Implementation of $U^{\top}\xb_{n+1}$ in columns $1,\cdots,n$.
\\
\noindent
For each $1\leq i\leq n$, we define each attention head $A_i$, with $V_i=\left[\begin{array}{ccccccc} 
0 & \cdots & \cdots & \cdots & 0 & \cdots & 0 \\
\vdots &  &  & & \vdots &  & \vdots \\
	0 & \cdots & \cdots & \cdots & 0 & \cdots & 0 \\ 
    & & U^{\top} & & \vdots & & \vdots \\ 
     0 & \cdots & \cdots & \cdots & 0 & \cdots & 0 \\
	 \vdots & & & & \vdots & & \vdots  \\ 
    0 & \cdots & \cdots & \cdots & 0 & \cdots & 0
\end{array}\right]\in\mathbb{R}^{(2D+6)\times (2D+6)}$, where $U^{\top}\in\mathbb{R}^{r\times D}$ lives in $(D+2)$-th row to $(D+1+r)$-th row. The data kernels are in the form 
\begin{equation*}
Q^{data}_i = 
\left[\begin{array}{cccccc|ccc} 
	0 & & & &  & & 0 & 0 & 0 \\ 
     & \ddots & & & & & \vdots & \vdots & \vdots \\
	 & & \ddots & & & & 0 & 0 & 0 \\ 
     & & & 0 & & &  \vdots & \vdots & \vdots\\
     & & & & 0 & & 0 & 0 & 0 \\
      & & & &  & 0 & 0 & 0 & 1
\end{array}\right] 
\quad\quad\quad 
K^{data}_i = 
\left[\begin{array}{cccccc|ccc} 
	0 & & & &  & & 0 & 0 & 0 \\ 
     & \ddots & & & & & \vdots & \vdots & \vdots \\
	 & & \ddots & & & & 0 & 0 & 0 \\ 
     & & & 0 & & &  \vdots & \vdots & \vdots\\
     & & & & 0 & & 0 & 0 & 0 \\
      & & & &  & 0 & 0 & 0 & 1
\end{array}\right] 
\end{equation*}
where $Q^{data}_i,K^{data}_i\in\mathbb{R}^{(2D+3)\times (2D+6)}$. Both the $Q^{data}_i$ and $K^{data}_i$ have the last row and last column's element equal to $1$ for all $i=1,\cdots,n$. By Interaction Lemma \ref{lemmaIA}, we can construct $A_i$, $1\leq i\leq n$, such that $h_{i}$ interacts with $h_{n+1}$ only, i.e.,
\begin{equation*}
    [A_i(H)]_{i} = \sigma(\langle Q^{data}_ih_{i}, K^{data}_ih_{n+1}\rangle) V_ih_{n+1} = \sigma(1)
    \begin{bmatrix}     
    \mathbf{0} \\
    U^{\top}\xb_{n+1} \\
    \mathbf{0}
    \end{bmatrix} = \begin{bmatrix}  
    \mathbf{0} \\
    U^{\top}\xb_{n+1} \\
    \mathbf{0}
    \end{bmatrix},
\end{equation*}
and $[A_i(H)]_t=0$ when $t\neq i$.  Then the residual multi-head attention yields 

\begin{equation*}
    {\rm MHA}(H)+H = \begin{bmatrix}
    \xb_{1} & \cdots & \xb_{n} & \xb_{n+1}  \\
    y_1 & \cdots & y_n & 0 \\
    U^{\top}\xb_{n+1} & \cdots & U^{\top}\xb_{n+1} & 0\\
    0 & \cdots & \cdots & 0 \\
    \vdots & \vdots & \vdots & \vdots & \\
    0 & \cdots & \cdots & 0 \\
    \mathcal{I}_1 & \cdots & \cdots & \cI_{n+1} \\
    1 & \cdots & \cdots & 1
\end{bmatrix}.
\end{equation*}

\noindent
Let the feed-forward component implement the identity transformation. Then
\begin{equation*}
    H_1:={\rm B}_1(H) = \begin{bmatrix}
    \xb_{1} & \cdots & \xb_{n} & \xb_{n+1}  \\
    y_1 & \cdots & y_n & 0 \\
    U^{\top}\xb_{n+1} & \cdots & U^{\top}\xb_{n+1} & 0\\
    0 & \cdots & \cdots & 0 \\
    \vdots & \vdots & \vdots & \vdots & \\
    0 & \cdots & \cdots & 0 \\
    \mathcal{I}_1 & \cdots & \cdots & \cI_{n+1} \\
    1 & \cdots & \cdots & 1
\end{bmatrix}.
\end{equation*}

\noindent
In total, this process needs ${\rm B}_1\in\mathcal{B}(n,1,2D+6)$. The upper bound of the weights parameter in ${\rm B}_1$ is $\|\theta_{{\rm B}_1}\|_{\infty}\leq O(D^4\ell^2u^2b^2)=O(D^4 n^2u^2b^2)$. 

\vspace{10mm}
\noindent
$\bullet$ Implementation of $U^{\top}\xb_{n+1} - U^{\top}\xb_{i}$ for $1\leq i\leq n$.
\\
\noindent
For each $1\leq i\leq n$, we define each attention head $A_i$, with $V_i=\left[\begin{array}{ccccccc} 
0 & \cdots & \cdots & \cdots & 0 & \cdots & 0 \\
\vdots &  &  & & \vdots &  & \vdots \\
	0 & \cdots & \cdots & \cdots & 0 & \cdots & 0 \\ 
    & & -U^{\top} & & \vdots & & \vdots \\ 
     0 & \cdots & \cdots & \cdots & 0 & \cdots & 0 \\
	 \vdots & & & & \vdots & & \vdots  \\ 
    0 & \cdots & \cdots & \cdots & 0 & \cdots & 0
\end{array}\right]\in\mathbb{R}^{(2D+6)\times (2D+6)}$, where $-U^{\top}\in\mathbb{R}^{r\times D}$ lives in $(D+2)$-th row to $(D+1+r)$-th row. The data kernels are in the form 
\begin{equation*}
Q^{data}_i = 
\left[\begin{array}{cccccc|ccc} 
	0 & & & &  & & 0 & 0 & 0 \\ 
     & \ddots & & & & & \vdots & \vdots & \vdots \\
	 & & \ddots & & & & 0 & 0 & 0 \\ 
     & & & 0 & & &  \vdots & \vdots & \vdots\\
     & & & & 0 & & 0 & 0 & 0 \\
      & & & &  & 0 & 0 & 0 & 1
\end{array}\right] 
\quad\quad\quad 
K^{data}_i = 
\left[\begin{array}{cccccc|ccc} 
	0 & & & &  & & 0 & 0 & 0 \\ 
     & \ddots & & & & & \vdots & \vdots & \vdots \\
	 & & \ddots & & & & 0 & 0 & 0 \\ 
     & & & 0 & & &  \vdots & \vdots & \vdots\\
     & & & & 0 & & 0 & 0 & 0 \\
      & & & &  & 0 & 0 & 0 & 1
\end{array}\right] 
\end{equation*}
where $Q^{data}_i,K^{data}_i\in\mathbb{R}^{(2D+3)\times (2D+6)}$. Both the $Q^{data}_i$ and $K^{data}_i$ have the last row and last column's element equal to $1$ for all $i=1,\cdots,n$. By Interaction Lemma \ref{lemmaIA}, we can construct $A_i$, $1\leq i\leq n$, such that $h_{i}$ interacts with itself only, i.e.,
\begin{equation*}
    [A_i(H_1)]_{i} = \sigma(\langle Q^{data}_ih_{i}, K^{data}_ih_{i}\rangle) V_ih_{i} = \sigma(1)
    \begin{bmatrix}  
    \mathbf{0} \\
    -U^{\top}\xb_{i} \\
    \mathbf{0}
    \end{bmatrix} = \begin{bmatrix} 
    \mathbf{0} \\
    -U^{\top}\xb_{i} \\
    \mathbf{0}
    \end{bmatrix},
\end{equation*}
and $[A_i(H_1)]_t=0$ when $t\neq i$.  Then the residual multi-head attention yields

Then the residual multi-head attention yields 

\begin{equation*}
    {\rm MHA}(H_1)+H_1 = \begin{bmatrix}
    \xb_{1} & \cdots & \xb_{n} & \xb_{n+1}  \\
    y_1 & \cdots & y_n & 0 \\
    U^{\top}(\xb_{n+1}-\xb_1) & \cdots & U^{\top}(\xb_{n+1}-\xb_n) & 0\\
    0 & \cdots & \cdots & 0 \\
    \vdots & \vdots & \vdots & \vdots & \\
    0 & \cdots & \cdots & 0 \\
    \mathcal{I}_1 & \cdots & \cdots & \cI_{n+1} \\
    1 & \cdots & \cdots & 1
\end{bmatrix}.
\end{equation*}

\noindent
Let the feed-forward component implement the identity transformation.  Then
\begin{equation*}
    H_2:={\rm B}_2(H_1) = \begin{bmatrix}
    \xb_{1} & \cdots & \xb_{n} & \xb_{n+1}  \\
    y_1 & \cdots & y_n & 0 \\
    U^{\top}(\xb_{n+1}-\xb_1) & \cdots & U^{\top}(\xb_{n+1}-\xb_n) & 0\\
    0 & \cdots & \cdots & 0 \\
    \vdots & \vdots & \vdots & \vdots & \\
    0 & \cdots & \cdots & 0 \\
    \mathcal{I}_1 & \cdots & \cdots & \cI_{n+1} \\
    1 & \cdots & \cdots & 1
\end{bmatrix}.
\end{equation*}

\noindent
In total, this process needs ${\rm B}_2\in\mathcal{B}(n,1,2D+6)$. The upper bound of the weights parameter in ${\rm B}_2$ is $\|\theta_{{\rm B}_2}\|_{\infty}\leq O(D^4\ell^2u^2b^2)=O(D^4 n^2u^2b^2)$.

\vspace{10mm}
\noindent
$\bullet$ Implementation of $-\frac{\|U^{\top}(\xb_{n+1}-\xb_j)\|^2}{h^2}$ for $1\leq j\leq n$.

\noindent
Let us define each attention head $A_{j}$ with $V_j=\eb_{2D+2}\eb_{d_{embed}}^{\top}=\eb_{2D+2}\eb_{2D+6}^{\top}$, and data kernels in the form 

\begin{equation*}
Q^{data}_{j} = 
\left[\begin{array}{ccccccccc|ccc} 
	0 & & & & & & & & & 0 & 0 & 0
    \\
     & \ddots & & & & & & & & 0 & 0 & 0 \\
     &  & 0 & & & &  & & & 0 & 0 & 0 \\
     & && \frac{1}{h} & & & & & & \vdots & \vdots & \vdots \\
	 & & & & \ddots & & & & & 0 & 0 & 0 \\ 
     & & & & & \frac{1}{h} & & & &  \vdots & \vdots & \vdots\\
     & & & & & & 0 & & & 0 & 0 & 0 \\
     & & & & & &  & \ddots & & 0 & 0 & 0 \\
     & & & & & & &  & 0 & 0 & 0 & M
\end{array}\right]
\end{equation*}
and
\begin{equation*}
K^{data}_{j} = 
\left[\begin{array}{ccccccccc|ccc} 
	0 & & & & & & & & & 0 & 0 & 0
    \\
     & \ddots & & & & & & & & 0 & 0 & 0 \\
     &  & 0 & & & &  & & & 0 & 0 & 0 \\
     & && -\frac{1}{h} & & & & & & \vdots & \vdots & \vdots \\
	 & & & & \ddots & & & & & 0 & 0 & 0 \\ 
     & & & & & -\frac{1}{h} & & & &  \vdots & \vdots & \vdots\\
     & & & & & & 0 & & & 0 & 0 & 0 \\
     & & & & & &  & \ddots & & 0 & 0 & 0 \\
     & & & & & & &  & 0 & 0 & 0 & 1
\end{array}\right], 
\end{equation*}
where $Q^{data}_{j},K^{data}_{j}\in\mathbb{R}^{(2D+3)\times (2D+6)}$, and the elements $\frac{1}{h}$ live in $(D+2)$-th row to $(D+1+r)$-th row. By Interaction Lemma \ref{lemmaIA}, we can construct $A_{j}$ such that $h_{j}$ interacts with itself only, i.e.,
\begin{align*}
    [A_{j}(H_2)]_{j} &= \sigma(\langle Q^{data}_{j}h_{j}, K^{data}_{j}h_{j}\rangle) V_jh_{j} \\ 
    &= \sigma\left(-\frac{1}{h^2}\langle U^{\top}(\xb_{n+1}-\xb_j), U^{\top}(\xb_{n+1}-\xb_j)\rangle+M\right)\eb_{2D+2} \\
    &= \left(-\frac{\|U^{\top}(\xb_{n+1}-\xb_j)\|^2}{h^2}+M\right)\eb_{2D+2}
\end{align*}
and $[A_{j}(H_2)]_t=0$ for $t\neq j$, where $M\geq \max\left(\frac{4b^2D}{h^2},R\right)$. Then the residual multi-head 
attention yields
\begin{align*}  
    {\rm MHA}(H_2) + H_2 =\begin{bmatrix}
    \xb_{1} & \cdots & \xb_{n} & \xb_{n+1}  \\
    y_1 & \cdots & y_n & 0 \\
    U^{\top}(\xb_{n+1}-\xb_1) & \cdots & U^{\top}(\xb_{n+1}-\xb_n) & 0  \\
   0 & \cdots & 0 & 0  \\
     \vdots & \vdots & \vdots & \vdots  \\
    0 & \cdots & 0 & 0  \\
    -\frac{\|U^{\top}(\xb_{n+1}-\xb_1)\|^2}{h^2}+M & \cdots & -\frac{\|U^{\top}(\xb_{n+1}-\xb_n)\|^2}{h^2}+M & 0 \\ 
    0 & \cdots & 0 & 0 \\
    \mathcal{I}_1 & \cdots & \cI_{n} & \cI_{n+1} \\
    1 & \cdots & 1 & 1
\end{bmatrix}.
\end{align*}

Then we can apply Decrementing Lemma \ref{lemmadecrement} to subtract off the constant $M\geq \max\left(\frac{4b^2D}{h^2},R\right)$ and get 

\begin{equation*}
    H_3:={\rm B_3}(H_2) = \begin{bmatrix}
    \xb_{1} & \cdots & \xb_{n} & \xb_{n+1}  \\
    y_1 & \cdots & y_n & 0 \\
    U^{\top}(\xb_{n+1}-\xb_1) & \cdots & U^{\top}(\xb_{n+1}-\xb_n) & 0  \\
   0 & \cdots & 0 & 0  \\
     \vdots & \vdots & \vdots & \vdots  \\
    0 & \cdots & 0 & 0  \\
    -\frac{\|U^{\top}(\xb_{n+1}-\xb_1)\|^2}{h^2} & \cdots & -\frac{\|U^{\top}(\xb_{n+1}-\xb_n)\|^2}{h^2} & 0 \\ 
    0 & \cdots & 0 & 0 \\
    \mathcal{I}_1 & \cdots & \cI_{n} & \cI_{n+1} \\
    1 & \cdots & 1 & 1
\end{bmatrix}
\end{equation*}

\noindent
In total, this process needs ${\rm B}_3\in\mathcal{B}(n,6,2D+6)$. The upper bound of the weights parameter in ${\rm B}_3$ is $\|\theta_{{\rm B}_3}\|_{\infty}\leq O(D^4\ell^2M^2\cdot M^2)=O(D^4\ell^2M^4)$.

\vspace{10mm}
\noindent
$\bullet$ Implementation of $\cK(\fs)$ in the $(D+1)$-th row and $(n+1)$-th column.

\noindent
In the last layer of transformer, we use softmax activation (with a mask of size $n$) instead of ReLU. Then, we can construct $V=\eb_{D+1}\eb^{\top}_{D+1}$, and data kernels  
\begin{equation*}
Q^{data} = 
\left[\begin{array}{ccccccc|ccc} 
	0 & & & &  & & &  0 & 0 & 0 \\ 
     & \ddots & & & & & & \vdots & \vdots & \vdots \\
     & & 0 & & & & & 0 & 0 & 1 \\ 
	 & & & \ddots & & & & 0 & 0 & 0 \\ 
     & & &  & 0 & & & \vdots & \vdots & \vdots\\
     & & & & & 0 & & 0 & 0 & 0 \\
      & & & &  & & 0 & 0 & 0 & 0
\end{array}\right], \quad 
K^{data} = 
\left[\begin{array}{ccccccc|ccc} 
	0 & & & &  & & &  0 & 0 & 0 \\ 
     & \ddots & & & & & & \vdots & \vdots & \vdots \\
     & & 1 & & & & & 0 & 0 & 0 \\ 
	 & & & \ddots & & & & 0 & 0 & 0 \\ 
     & & &  & 0 & & & \vdots & \vdots & \vdots\\
     & & & & & 0 & & 0 & 0 & 0 \\
      & & & &  & & 0 & 0 & 0 & 0
\end{array}\right], 
\end{equation*}
where $Q^{data}\in\mathbb{R}^{(2D+3)\times(2D+6)}$ has $(2D+2)$-th position in the last column equals to $1$ and all the other entries are $0$, and $K^{data}\in\mathbb{R}^{(2D+3)\times(2D+6)}$ has $(2D+2,2D+2)$-th position equals to $1$ and all the other entries are $0$,
such that 
\begin{align*}
    [{\rm A}(H_3)]_{n+1} &= \sum_{j=1}^{n}{\rm softmax}\left(\langle Q^{data}h_{n+1}, K^{data}H_3\rangle\right)_{j} Vh_{j} \\ 
    &= \sum_{j=1}^n\frac{y_je^{-\|U^{\top}(\mathbf{x}_{n+1}-\mathbf{x}_j)\|^2/h^2}}{\sum_{j=1}^n e^{-\|U^{\top}(\mathbf{x}_{n+1}-\mathbf{x}_j)\|^2/h^2}}\cdot \eb_{D+1}=\cK(\{\xb_i,y_i\}_{i=1}^n;\xb_{n+1})\cdot \eb_{D+1}= \cK(\fs)\cdot \eb_{D+1},
\end{align*}
where $\eb_{j}$ is the one-hot vector with entry $1$ at the $j$-th component.
Therefore, the residual attention outputs
\begin{align*}
     {\rm A}(H_3)+H_3=    \begin{bmatrix}
    \xb_{1} & \cdots & \xb_{n} & \xb_{n+1}  \\
    y_1 & \cdots & y_n & \cK(\fs) \\
    U^{\top}(\xb_{n+1}-\xb_1) & \cdots & U^{\top}(\xb_{n+1}-\xb_n) & 0  \\
   0 & \cdots & 0 & 0  \\
     \vdots & \vdots & \vdots & \vdots  \\
    0 & \cdots & 0 & 0  \\
    -\frac{\|U^{\top}(\xb_{n+1}-\xb_1)\|^2}{h^2} & \cdots & -\frac{\|U^{\top}(\xb_{n+1}-\xb_n)\|^2}{h^2} & 0 \\ 
    0 & \cdots & 0 & 0 \\
    \mathcal{I}_1 & \cdots & \cI_{n} & \cI_{n+1} \\
    1 & \cdots & 1 & 1
\end{bmatrix}.
\end{align*}
Therefore, there exists ${\rm B}_4\in\mathcal{B}(1,1,2D+6)$ such that 
\begin{align*}
    H_4:={\rm B_4}(H_3) =     \begin{bmatrix}
    \xb_{1} & \cdots & \xb_{n} & \xb_{n+1}  \\
    y_1 & \cdots & y_n & \cK(\fs) \\
    U^{\top}(\xb_{n+1}-\xb_1) & \cdots & U^{\top}(\xb_{n+1}-\xb_n) & 0  \\
   0 & \cdots & 0 & 0  \\
     \vdots & \vdots & \vdots & \vdots  \\
    0 & \cdots & 0 & 0  \\
    -\frac{\|U^{\top}(\xb_{n+1}-\xb_1)\|^2}{h^2} & \cdots & -\frac{\|U^{\top}(\xb_{n+1}-\xb_n)\|^2}{h^2} & 0 \\ 
    0 & \cdots & 0 & 0 \\
    \mathcal{I}_1 & \cdots & \cI_{n} & \cI_{n+1} \\
    1 & \cdots & 1 & 1
\end{bmatrix}.
\end{align*}
The upper bound of the weights parameter in ${\rm B}_4$ is $\|\theta_{{\rm B}_4}\|_{\infty}\leq O(D^4\ell^2M^2\cdot 1)=O\left(D^4\ell^2M^2\right)$. 
Finally, we apply a decoding layer ${\rm DE}$ to output the element $\cK(\fs)$ as desired. The uniform upper bound for the weight parameters in ${\rm B}_4\circ {\rm B}_3\circ {\rm B}_2\circ{\rm B}_1$ is $\kappa\leq O\left(D^4 \ell^2 M^4\right)\leq O\left(\frac{D^8\ell^2b^8R^4}{h^8}\right)=O\left(\frac{D^8 n^2b^8R^4}{h^8}\right)$.
\end{proof}

\subsection{Proof of Lemma \ref{lemma:kernel}} \label{subsec:AppendixProposition1}

\begin{proof} [Proof of Lemma \ref{lemma:kernel}]

The main idea for proving Lemma \ref{lemma:kernel} is to decompose it into:
\begin{align} 
\label{eq:ApproxDecompose}
    \cR_{\fs_f}(\cK)=\left(\cK(\fs_f)-f(\xb_{n+1})\right)^2     
    &= \left(\cK(\fs_f)-\bar{\cK}(f;\xb_{n+1})+\bar{\cK}(f;\xb_{n+1})-f(\xb_{n+1})\right)^2 
    \\
    &
    \leq 2\underbrace{\left(\cK(\fs_f)-\bar{\cK}(f;\xb_{n+1})\right)^2}_{\rm A1}
    +
    2
    \underbrace{\left(\bar{\cK}(f;\xb_{n+1})-f(\xb_{n+1})\right)^2}_{\rm A2} \nonumber
\end{align}
where the error in ${\rm A1}$ measures the variance of kernel estimator, ${\rm A2}$ measures the bias of using kernel estimator to approximate the target function $f$. Here $\cK(\cdot)$ is given by \eqref{eq:NWkernelF7}, and $\bar{\cK}(\cdot)$ is the integral form of kernel estimator:
\begin{align} \label{barK}     \bar{\cK}(f;\xb_{n+1}):&=\frac{\mathbb{E}_{\xb,\epsilon}\left[K_h(U^{\top}\xb_{n+1}-U^{\top}\xb)(f(\xb)+\epsilon)\right]}{\mathbb{E}_{\xb}\left[ K_h(U^{\top}\xb_{n+1}-U^{\top}\xb)\right]} \\
\nonumber
& =\frac{\mathbb{E}_{\xb}\left[K_h(U^{\top}\xb_{n+1}-U^{\top}\xb)(f(\xb))\right]}{\mathbb{E}_{\xb}\left[ K_h(U^{\top}\xb_{n+1}-U^{\top}\xb)\right]}+\frac{\mathbb{E}_{\xb,\epsilon}\left[K_h(U^{\top}\xb_{n+1}-U^{\top}\xb)\epsilon\right]}{\mathbb{E}_{\xb}\left[ K_h(U^{\top}\xb_{n+1}-U^{\top}\xb)\right]} \\
& =\frac{\mathbb{E}_{\xb}\left[K_h(U^{\top}\xb_{n+1}-U^{\top}\xb)(f(\xb))\right]}{\mathbb{E}_{\xb}\left[ K_h(U^{\top}\xb_{n+1}-U^{\top}\xb)\right]}=\frac{\int K_h(U^{\top}\xb_{n+1}-U^{\top}\xb)g(U^{\top}\xb) d\xb}{\int K_h(U^{\top}\xb_{n+1}-U^{\top}\xb)d\xb},
\end{align}
where $K_h(\xb):=e^{-\frac{\|\xb\|^2}{h^2}}$. By abusing the notation, we use $d\xb$ and $d{\rho}_{\xb}$ interchangeably throughout this proof.
\vspace{2mm}
\\
\noindent
$\bullet$ {\bf Bounding the bias ${\rm A2}$} 
\\
\noindent
By \eqref{eq:ApproxDecompose} and \eqref{barK}, the bias term ${\rm A2}$ satisfies
\begin{align*}
\bar{\cK}_h(f;\xb_{n+1}) -f(\xb_{n+1})
&= \frac{\mathbb{E}\left[K_h(U^{\top}\xb_{n+1}-U^{\top}\xb)g(U^{\top}\xb)\right]}{\mathbb{E}\left[K_h(U^{\top}\xb_{n+1}-U^{\top}\xb)\right]} -g(U^{\top}\xb_{n+1}) \\
& =\frac{\int K_h(U^{\top}\xb_{n+1}-U^{\top}\xb)g(U^{\top}\xb)d\xb}{\int K_h(U^{\top}\xb_{n+1}-U^{\top}\xb)d\xb} - g(U^{\top}\xb_{n+1}) \\
& = \frac{\int K_h(U^{\top}\xb_{n+1}-U^{\top}\xb)(g(U^{\top}\xb)-g(U^{\top}\xb_{n+1}))d\xb}{\int K_h(U^{\top}\xb_{n+1}-U^{\top}\xb)d\xb}.
\end{align*}

To deal with the infinite support of the Gaussian kernel, we decompose the integral into regions close to $\xb_{n+1}$ and far away from $\xb_{n+1}$. For $\xb$ close to the center $\xb_{n+1}$, we use the Lipschitz property of $f$ to estimate the bias; For $\xb$ far from the center $\xb_{n+1}$, we use the Gaussian tail to bound the bias.

Let us consider the set of points in $\Omega$ which are at most $\tilde{h}$ distance to $\xb_{n+1}$ after projecting to the subspace $U$:
\[B_{\tilde{h}}(\xb_{n+1}):=\{\xb\in\Omega:\|U^{\top}\xb-U^{\top}\xb_{n+1}\|\leq \tilde{h}=Ch\}.\] 
We will specified choice of $C$ later in the proof.
Then, we have
\begin{align*}
     \bar{\cK}_h(f;\xb_{n+1}) -f(\xb_{n+1}) &= \frac{\int_{ B_{\tilde{h}}(\xb_{n+1})} K_h(U^{\top}\xb_{n+1}-U^{\top}\xb)(g(U^{\top}\xb)-g(U^{\top}\xb_{n+1}))d\xb}{\int K_h(U^{\top}\xb_{n+1}-U^{\top}\xb)d\xb} \\
     & + \frac{\int_{\Omega\setminus B_{\tilde{h}}(\xb_{n+1})} K_h(U^{\top}\xb_{n+1}-U^{\top}\xb)(g(U^{\top}\xb)-g(U^{\top}\xb_{n+1}))d\xb}{\int K_h(U^{\top}\xb_{n+1}-U^{\top}\xb)d\xb} \\
     &\leq \frac{\int_{ B_{\tilde{h}}(\xb_{n+1})} K_h(U^{\top}\xb_{n+1}-U^{\top}\xb)L(\tilde{h})^{\alpha}d\xb}{\int K_h(U^{\top}\xb_{n+1}-U^{\top}\xb)d\xb} + \frac{2R\int_{\Omega\setminus B_{\tilde{h}}(\xb_{n+1})} K_h(U^{\top}\xb_{n+1}-U^{\top}\xb)d\xb}{\int K_h(U^{\top}\xb_{n+1}-U^{\top}\xb)d\xb} \\
     &\leq L\tilde{h}^{\alpha}+\frac{2R\int_{\Omega\setminus B_{\tilde{h}}(\xb_{n+1})} K_h(U^{\top}\xb_{n+1}-U^{\top}\xb)d\xb}{\int K_h(U^{\top}\xb_{n+1}-U^{\top}\xb)d\xb},
\end{align*}
where the first inequality uses the Lipschitz property of $g$ inside the ball $B_{\tilde{h}}(\xb_{n+1})$.

Next, let us bound the integral outside  the ball $B_{\tilde{h}}(\xb_{n+1})$. Notice that
\begin{align*}
    \int_{\Omega} h^{-r}K_h(U^{\top}\xb_{n+1}-U^{\top}\xb)d\xb&=\int_{\Omega} h^{-r}e^{-\frac{\|U^{\top}\xb_{n+1}-U^{\top}\xb\|^2}{h^2}}d\xb\geq \int_{B_h(\xb_{n+1})} h^{-r}e^{-\frac{\|U^{\top}\xb_{n+1}-U^{\top}\xb\|^2}{h^2}}d\xb\\
    &\geq \int_{B_h(\xb_{n+1})}h^{-r}e^{-1}d\xb=h^{-r}e^{-1}{\rm Vol}(B_h(\xb_{n+1})) \\
    &=h^{-r}e^{-1}C_B h^r (2b)^{D-r}=e^{-1}C_B(2b)^{D-r},
\end{align*}
for some absolute constant $C_B$.

Therefore, the integral outside the ball $B_{\tilde{h}}(\xb_{n+1})$ can be bounded as follows:
\begin{align*}
\frac{2R\int_{\Omega\setminus B_{\tilde{h}}(\xb_{n+1})} K_h(U^{\top}\xb_{n+1}-U^{\top}\xb)d\xb}{\int_{\Omega} K_h(U^{\top}\xb_{n+1}-U^{\top}\xb)d\xb} 
   & = \frac{2R\int_{\Omega\setminus B_{\tilde{h}}(\xb_{n+1})} h^{-r}K_h(U^{\top}\xb_{n+1}-U^{\top}\xb)d\xb}{\int_{\Omega} h^{-r}K_h(U^{\top}\xb_{n+1}-U^{\top}\xb)d\xb} \\
   & \leq  2eRC_B^{-1}(2b)^{r-D}\int_{\Omega\setminus B_{\tilde{h}}(\xb_{n+1})} h^{-r}K_h(U^{\top}\xb_{n+1}-U^{\top}\xb)d\xb \\
   &\leq 2eRC_B^{-1}(2b)^{r-D}\cdot(\rho_{\xb}(\Omega)-\rho_{\xb}(B_{\tilde{h}}(\xb_{n+1})))\cdot h^{-r}e^{-\frac{\tilde{h}^2}{h^2}} \\
   & =2eRC_B^{-1}(2b)^{r-D}\cdot(\rho_{\xb}(\Omega)-\rho_{\xb}(B_{\tilde{h}}(\xb_{n+1})))\cdot h^{-r}e^{-C^2} \\
   & \leq 2eRC_B^{-1}(2b)^{r-D}\cdot h^{-r}e^{-C^2} \\
   & = O(h^{-r}e^{-C^2}),
\end{align*}
where $O(\cdot)$ hides the dependency on $b,r,D,R$.

Putting the bounds for inside and outside the ball $B_{\tilde{h}}(\xb_{n+1})$ together, we get 
\begin{align*}
    \bar{\cK}_h(f;\xb_{n+1}) -f(\xb_{n+1})\leq O(\tilde{h}^{\alpha})+O(h^{-r}e^{-C^2}). \leq O(C^\alpha h^\alpha)+O(h^{-r}e^{-C^2}).
\end{align*}
For $h\in (0,1)$, by choosing $C=\sqrt{(r+1)\log(h^{-1})}$, we have $h^{-r}e^{-C^2}=h$. Therefore,
\begin{align*}
   \bar{\cK}_h(f;\xb_{n+1}) - f(\xb_{n+1}) \leq O\left(h^{\alpha}\log\left(h^{-1}\right)\right)+O(h)= O\left(h^{\alpha}\log\left(h^{-1}\right)\right),
\end{align*}
where $O(\cdot)$ hides the constants depending on $L,b,r,D,R$. Hence ${\rm A2}\leq O(h^{2\alpha}[\log(h^{-1})]^2)$.
\vspace{5mm}
\\
\noindent
$\bullet$ {\bf Bounding the variance ${\rm A1}$} 
\\
\noindent
For the variance term ${\rm A1}$, let us first define 
\[
\hat{N}_{n,y} (\xb_{n+1}) := \frac{1}{n}\sum_{i=1}^n K_h\left(U^{\top}\xb_{n+1}-U^{\top}\xb_i\right) y_i, \quad N_y(\xb_{n+1}) := \mathbb{E}_{\xb,\epsilon}\left[K_h\left(U^{\top}\xb_{n+1}-U^{\top}\xb\right) y\right] 
\]
\[
\hat{N}_n (\xb_{n+1}) := \frac{1}{n}\sum_{i=1}^n K_h\left(U^{\top}\xb_{n+1}-U^{\top}\xb_i\right)f(\xb_i), \quad N(\xb_{n+1}) := \mathbb{E}_{\xb}\left[K_h\left(U^{\top}\xb_{n+1}-U^{\top}\xb\right) f(\xb)\right] 
\]
\[
\hat{D}_n (\xb_{n+1}) :=\frac{1}{n}\sum_{i=1}^n K_h\left(U^{\top}\xb_{n+1}-U^{\top}\xb_i\right), \quad D(\xb_{n+1}) := \mathbb{E}_\xb\left[K_h\left(U^{\top}\xb_{n+1}-U^{\top}\xb\right)\right].
\]
It is easy to see $N_y(\xb_{n+1})=N(\xb_{n+1})$.
We further define 
\[\hat{N}_{n,\epsilon} (\xb_{n+1}) :=\hat{N}_{n,y}(\xb_{n+1})-\hat{N}_n(\xb_{n+1})=\frac{1}{n}\sum_{i=1}^n K_h\left(U^{\top}\xb_{n+1}-U^{\top}\xb_i\right)\epsilon_i.\]
The variance can be decomposed as
\begin{align}
    & |\cK_h(\fs_f)-\bar{\cK}_h(f;\xb_{n+1})| \nonumber \\
    & \leq \frac{1}{|\hat{D}_n(\xb_{n+1})||D(\xb_{n+1})|} \left( |D(\xb_{n+1})||\hat{N}_{n,y}(\xb_{n+1}) - N_y(\xb_{n+1})| + |N_y(\xb_{n+1})||\hat{D}_n(\xb_{n+1}) - D(\xb_{n+1})| \right)\nonumber \\
    & =\frac{1}{|\hat{D}_{n}(\xb_{n+1})|} \left( |\hat{N}_{n,y}(\xb_{n+1}) - N_y(\xb_{n+1})| \right)+\frac{|N_y(\xb_{n+1})|}{|\hat{D}_n(\xb_{n+1})|\cdot |D(\xb_{n+1})|} \left(|\hat{D}_n(\xb_{n+1}) - D(\xb_{n+1})| \right).  \nonumber\\
    & =\frac{|\hat{N}_{n,\epsilon}(\xb_{n+1})|}{\hat{D}_n(\xb_{n+1})}+\frac{1}{|\hat{D}_{n}(\xb_{n+1})|} \left( |\hat{N}_{n}(\xb_{n+1}) - N(\xb_{n+1})| \right)+\frac{|N(\xb_{n+1})|}{|\hat{D}_n(\xb_{n+1})|\cdot |D(\xb_{n+1})|} \left(|\hat{D}_n(\xb_{n+1}) - D(\xb_{n+1})| \right).
    \label{eq:lemma3p1}
\end{align}
Let $a_i=K_h\left(U^{\top}\xb_{n+1}-U^{\top}\xb_i\right)$, then the first term $\frac{\hat{N}_{n,\epsilon}(\xb_{n+1})}{\hat{D}_n(\xb_{n+1})}=\frac{\sum_{i=1}^n a_i\epsilon_i}{\sum_{i=1}^n a_i}\sim \mathcal{N}\left(0,\frac{\sum_{i=1}^n a_i^2}{\left(\sum_{i=1}^n a_i\right)^2}\right)$.

By the concentration inequality, with probability at least $1-\delta$, it satisfies 
\begin{align*}
    \left|\frac{\sum_{i=1}^n a_i\epsilon_i}{\sum_{i=1}^n a_i}\right|\leq \frac{\sqrt{\sum_{i=1}^n a_i^2}}{\left(\sum_{i=1}^n a_i\right)^2} \cdot\sqrt{2\log(2/\delta)}\leq O\left(\sqrt{\frac{\log(2/\delta)}{n}}\right),
\end{align*}
where the last step uses the Reversed Cauchy-Schwarz inequality. The notation $O(\cdot)$ hides the dependency on $b$. 

Now let us bound the second and third terms in \eqref{eq:lemma3p1}.
We consider the ball 
\begin{equation} B_{\tilde{h}}:=B_{\tilde{h}}(\xb_{n+1}):=\{\xb\in\Omega:\|U^{\top}\xb-U^{\top}\xb_{n+1}\|\leq \tilde{h}=Ch\}
\label{eq:lemma3pball}
\end{equation}
Let us use $n_B$ to denote the number of samples in $B_{\tilde{h}}(\xb_{n+1})$.
By \citet[Lemma 30]{liao2019adaptive}, we can estimate $n_B$ as follows: 
\[\mathbb{P}\Big\{\left|\frac{n_B}{n}-\rho_{\xb}(B_{\tilde{h}}(\xb_{n+1}))\right|\geq \frac{1}{2}\rho_{\xb}(B_{\tilde{h}}(\xb_{n+1}))\Big\}\leq 2e^{-\frac{3n\cdot\rho_{\xb}(B_{\tilde{h}}(\xb_{n+1}))}{28}},\] 
where 
\begin{align*}
    C_{\rho,1}C^rh^r=C_{\rho,1}\tilde{h}^r\leq\rho_{\xb}(B_{\tilde{h}}(\xb_{n+1}))\leq C_{\rho,2}\tilde{h}^r=C_{\rho,2}C^rh^r, 
\end{align*}
for some uniform absolute constants $C_{\rho,1},C_{\rho,2}$ (defined in Assumption \ref{assumpx_multi_index}).
Therefore, with probability at least $1-2e^{-\frac{3n\cdot\rho_{\xb}(B_{\tilde{h}}(\xb_{n+1}))}{28}}$, it holds 
\[\frac{1}{2}C_{\rho,1}C^rh^r=\frac{1}{2}\rho_{\xb}(B_{\tilde{h}}(\xb_{n+1}))\leq \frac{n_B}{n}\leq \frac{3}{2}\rho_{\xb}(B_{\tilde{h}}(\xb_{n+1}))=\frac{3}{2}C_{\rho,2}C^rh^r.\]
Next, we can rewrite the denominator $\hat{D}_n(\xb_{n+1})$ as 
$$\hat{D}_n(\xb_{n+1})=\frac{n_B}{n}\cdot\frac{1}{n_B}\sum_{i=1}^nK_h(U^{\top}\xb_{n+1}-U^{\top}\xb_i) \ge \frac{n_B}{n}\cdot\frac{1}{n_B}\sum_{\xb_i \in B_{\tilde h}(\xb_{n+1})}K_h(U^{\top}\xb_{n+1}-U^{\top}\xb_i)$$ where the continuous counterpart of $\frac{1}{n_B}\sum_{\xb_i \in B_{\tilde h}(\xb_{n+1})} K_h(U^{\top}\xb_{n+1}-U^{\top}\xb_i)$ is 
\begin{align*}
\Phi&:=\frac{1}{\rho_{\xb}(B_{\tilde{h}})}\int_{B_{\tilde h}} K_h(U^{\top}\xb-U^{\top}\xb_{n+1})d\xb
 \ge \frac{1}{\rho_{\xb}(B_{\tilde h})} \int_{B_{ h}} K_h(U^{\top}\xb-U^{\top}\xb_{n+1})d\xb
 \\
 & \ge \frac{e^{-1}\rho_{\xb}(B_{ h})}{\rho_{\xb}(B_{\tilde h})} \ge C_{\Phi} C^{-r}
\end{align*}
for some absolute constant $C_\Phi$. By the Hoeffding's inequality \citep{hoeffding1994probability}, with probability $1-\delta$, it satisfies
\begin{align*}\hat{D}_n(\xb_{n+1})
&\geq \frac{n_B}{n}\cdot\left(\Phi-\sqrt{\frac{\log(2/\delta)}{n_B}}\right)
\geq \frac{n_B}{n}\left(C_{\Phi}C^{-r}-\sqrt{\frac{\log(2/\delta)}{n_B}}\right) .\end{align*}

The numerator $\hat{N}_n(\xb_{n+1}) - N(\xb_{n+1})$ of the second term can be decomposed as
\[\left| \hat{N}_n(\xb_{n+1}) - N(\xb_{n+1}) \right|\leq \left|\hat{N}^{(1)}_n(\xb_{n+1}) - N^{(1)}(\xb_{n+1})  \right|+\left|\hat{N}^{(2)}_n(\xb_{n+1}) - N^{(2)}(\xb_{n+1})  \right|,\]
where \[\hat{N}^{(1)}_n(\xb_{n+1})=\frac{1}{n}\sum_{\xb_i\in B_{\tilde{h}}(\xb_{n+1})} K_h\left(U^{\top}\xb_{n+1}-U^{\top}\xb_i\right) f(\xb_i)\]
and \[N^{(1)}(\xb_{n+1}) = \int_{\xb\in B_{\tilde{h}}(\xb_{n+1})} K_h(U^{\top}\xb_{n+1}-U^{\top}\xb)f(\xb)d\xb\]
and \[\hat{N}^{(2)}_n(\xb_{n+1})=\frac{1}{n}\sum_{\xb_i\in \Omega\setminus B_{\tilde{h}}(\xb_{n+1})} K_h\left(U^{\top}\xb_{n+1}-U^{\top}\xb_i\right) f(\xb_i)\]
and 
\[N^{(2)}(\xb_{n+1}) = \int_{\xb\in \Omega\setminus B_{\tilde{h}}(\xb_{n+1})} K_h(U^{\top}\xb_{n+1}-U^{\top}\xb)f(\xb)d\xb\]
Therefore, to bound the second term in \eqref{eq:lemma3p1}, we just need to bound
\begin{align}
   &\frac{\left|\hat{N}^{(1)}_n(\xb_{n+1}) - N^{(1)}(\xb_{n+1})  \right|+\left|\hat{N}^{(2)}_n(\xb_{n+1}) - N^{(2)}(\xb_{n+1})  \right|}{\hat{D}_n(\xb_{n+1})} 
   \nonumber
   \\
   &
 \leq \frac{\left|\hat{N}^{(1)}_n(\xb_{n+1}) - N^{(1)}(\xb_{n+1})  \right|}{\frac{n_B}{n}\left(C_{\Phi}C^{-r}-\sqrt{\frac{\log(2/\delta)}{n_B}}\right)}+\frac{\left|\hat{N}^{(2)}_n(\xb_{n+1}) - N^{(2)}(\xb_{n+1})  \right|}{\frac{n_B}{n}\left(C_{\Phi}C^{-r}-\sqrt{\frac{\log(2/\delta)}{n_B}}\right)}.
   \label{eq:lemma3p2}
\end{align}
The first term in \eqref{eq:lemma3p2} can be written as 
\begin{align*}
&\frac{\left|\hat{N}^{(1)}_n(\xb_{n+1}) - N^{(1)}(\xb_{n+1})  \right|}{\frac{n_B}{n}\left(C_{\Phi}C^{-r}-\sqrt{\frac{\log(2/\delta)}{n_B}}\right)}
\\
=&\frac{\left|\frac{1}{n_B}\sum_{\xb_i\in B_{\tilde{h}}(\xb_{n+1})} K_h\left(U^{\top}\xb_{n+1}-U^{\top}\xb_i\right) f(\xb_i)-\frac{n}{n_B}\int_{\xb\in B_{\tilde{h}}(\xb_{n+1})} K_h(U^{\top}\xb_{n+1}-U^{\top}\xb)f(\xb)d\xb\right|}{C_{\Phi}C^{-r}-\sqrt{\frac{\log(2/\delta)}{n_B}}} \\
 \leq &\frac{\left|\frac{1}{n_B}\sum_{\xb_i\in B_{\tilde{h}}(\xb_{n+1})} K_h\left(U^{\top}\xb_{n+1}-U^{\top}\xb_i\right) f(\xb_i)-\frac{1}{\rho_\xb(B_{\tilde{h}})}\int_{\xb\in B_{\tilde{h}}(\xb_{n+1})} K_h(U^{\top}\xb_{n+1}-U^{\top}\xb)f(\xb)d\xb\right|}{C_{\Phi}C^{-r}-\sqrt{\frac{\log(2/\delta)}{n_B}}} \\ 
& + \frac{\left|\frac{1}{\rho_\xb(B_{\tilde{h}})}\int_{\xb\in B_{\tilde{h}}(\xb_{n+1})} K_h(U^{\top}\xb_{n+1}-U^{\top}\xb)f(\xb) d\xb-\frac{n}{n_B}\int_{\xb\in B_{\tilde{h}}(\xb_{n+1})} K_h(U^{\top}\xb_{n+1}-U^{\top}\xb)f(\xb) d\xb\right|}{C_{\Phi}C^{-r}-\sqrt{\frac{\log(2/\delta)}{n_B}}}. 
\end{align*}
By Hoeffding's inequality, with probability at least $1-\delta$, it satisfies
\begin{align*}
    \left |\frac{1}{n_B}\sum_{\xb_i\in B_{\tilde{h}}(\xb_{n+1})} K_h\left(\xb_{n+1}-\xb_i\right) f(\xb_i)-\frac{1}{\rho_\xb(B_{\tilde{h}})}\int_{\xb\in B_{\tilde{h}}(\xb_{n+1})} K_h(\xb_{n+1}-\xb)f(\xb)d\xb\right | \leq R\sqrt{\frac{\log(2/\delta)}{n_B}}.
\end{align*}
By \citet[Lemma 30]{liao2019adaptive}, with probability $1-\delta$, it satisfies
\begin{align*}
    \left|\rho_{\xb}(B_{\tilde{h}}(\xb_{n+1}))-\frac{n_B}{n}\right|\leq O\left(\sqrt{\frac{\log(2/\delta)\rho_{\xb}(B_{\tilde{h}}(\xb_{n+1}))}{n}}\right),
\end{align*}
which gives
\begin{align*}
&\left|\frac{1} {\rho_\xb(B_{\tilde{h}})}\int_{\xb\in B_{\tilde{h}}(\xb_{n+1})} K_h(\xb_{n+1}-\xb)f(\xb) d\xb-\frac{n}{n_B}\int_{\xb\in B_{\tilde{h}}(\xb_{n+1})} K_h(\xb_{n+1}-\xb)f(\xb) d\xb\right|
\\
\le & \left|\frac{1} {\rho_\xb(B_{\tilde{h}})} - \frac{n}{n_B}\right| \cdot \left|\int_{\xb\in B_{\tilde{h}}(\xb_{n+1})} K_h(\xb_{n+1}-\xb)f(\xb) d\xb\right|
\le R  \rho_\xb(B_{\tilde{h}})\left|\frac{1} {\rho_\xb(B_{\tilde{h}})} - \frac{n}{n_B}\right|
\\
=& R\left|1-\frac{n}{n_B} \rho_\xb(B_{\tilde{h}}) \right| = R \frac{n}{n_B} \left|\frac{n_B}{n}-\rho_\xb(B_{\tilde{h}})\right| \le O\left(\frac{R}{\rho_{\xb}(B_{\tilde{h}}(\xb_{n+1}))}\sqrt{\frac{\log(2/\delta)\rho_{\xb}(B_{\tilde{h}}(\xb_{n+1}))}{n}}\right)
\\
& = O\left(R\sqrt{\frac{\log(2/\delta)}{n\rho_{\xb}(B_{\tilde{h}}(\xb_{n+1}))}}\right).
\end{align*}
Therefore, with probability at least $1-2\delta$, the first term in \eqref{eq:lemma3p2} satisfies
\begin{align}
    \frac{\left|\hat{N}^{(1)}_n(\xb_{n+1}) - N^{(1)}(\xb_{n+1})  \right|}{\hat{D}_n(\xb_{n+1})} 
    &\leq O\left(RC^r\sqrt{\frac{\log(2/\delta)}{n_B}}\right) + O\left(RC^r\sqrt{\frac{\log(2/\delta)}{n\rho_{\xb}(B_{\tilde{h}}(\xb_{n+1}))}}\right) \nonumber \\
    & =O\left(RC^r\sqrt{\frac{\log(2/\delta)}{n\rho_{\xb}(B_{\tilde{h}}(\xb_{n+1}))}}\right)=O\left(\frac{RC^{r/2}}{\sqrt{C_{\rho,1}}}\sqrt{\frac{\log(2/\delta)}{nh^r}}\right)\nonumber 
    \\
    & = O\left(\frac{RC^{r/2}}{\sqrt{C_{\rho,1}}}\sqrt{\frac{\log(2/\delta)}{nh^r}}\right),
    \label{eq:lemma3eq2inball}
\end{align}
for some absolute constant $C_B$.

For the second term in \eqref{eq:lemma3p2}, it satisfies
\begin{align*}
    &\frac{\left|\hat{N}^{(2)}_n(\xb_{n+1}) - N^{(2)}(\xb_{n+1})  \right|}{\hat{D}_n(\xb_{n+1})}
    \\
    \leq
    &\frac{\left|\frac{1}{n}\sum_{\xb_i\in \Omega\setminus B_{\tilde{h}}(\xb_{n+1})} K_h\left(U^{\top}\xb_{n+1}-U^{\top}\xb_i\right) f(\xb_i)-\int_{\xb\in \Omega\setminus B_{\tilde{h}}(\xb_{n+1})} K_h(U^{\top}\xb_{n+1}-U^{\top}\xb)f(\xb) d\xb\right|}{\frac{n_B}{n}\left(C_{\Phi}C^{-r}-\sqrt{\frac{\log(2/\delta)}{n_B}}\right)}
\end{align*}
By Hoeffding's inequality, with probability at least $1-\delta$, it satisfies 
\begin{align*}
    \left |\frac{1}{n}\sum_{\xb_i\in\Omega\setminus B_{\tilde{h}}(\xb_{n+1})} K_h\left(U^{\top}\xb_{n+1}-U^{\top}\xb_i\right) f(\xb_i)-\int_{\xb\in\Omega\setminus B_{\tilde{h}}(\xb_{n+1})} K_h(\xb_{n+1}-\xb)f(\xb)d\xb\right | \leq Re^{-C^2}\sqrt{\frac{\log(2/\delta)}{n}},
\end{align*}
where we have used the fact $|K_h(\xb_{n+1}-\xb)f(\xb)| \le Re^{-C^2}$ for all $\xb\in \Omega\setminus B_{\tilde h}(\xb_{n+1})$. Therefore, the second term in \eqref{eq:lemma3p2} can be bounded by
\begin{align}
    \frac{\left|\hat{N}^{(2)}_n(\xb_{n+1}) - N^{(2)}(\xb_{n+1})  \right|}{\hat{D}_n(\xb_{n+1})}
    &\leq\frac{Re^{-C^2}\sqrt{\frac{\log(2/\delta)}{n}}}{\frac{n_B}{n}\left(C_{\Phi}C^{-r}-\sqrt{\frac{\log(2/\delta)}{n_B}}\right)}
    \le O\left(C^r e^{-C^2}\frac{n}{n_B}\sqrt{\frac{\log(2/\delta)}{n}} \right)
    \nonumber \\
    & = O\left(\frac{ e^{-C^2}}{C_{\rho,1}h^{r/2}}\sqrt{\frac{\log(2/\delta)}{nh^r}} \right)= O\left(\frac{ e^{-C^2}}{h^{r/2}}\sqrt{\frac{\log(2/\delta)}{nh^r}} \right).
    \label{eq:lemma3eq2outball}
\end{align}
Combining \eqref{eq:lemma3eq2inball} and \eqref{eq:lemma3eq2outball} gives rise to
\begin{align*}
    \frac{ |\hat{N}_n(\xb_{n+1}) - N(\xb_{n+1})| }{|\hat{D}_n(\xb_{n+1})|} &\leq O\left(\frac{RC^{r/2}}{\sqrt{C_{\rho,1}}}\sqrt{\frac{\log(2/\delta)}{nh^r}}\right)
+O\left(\frac{ e^{-C^2}}{h^{r/2}}\sqrt{\frac{\log(2/\delta)}{nh^r}} \right) \\
&=O\left(\left[\log\left(h^{-1}\right)\right]^{\frac{r}{4}}\sqrt{\frac{\log(2/\delta)}{nh^{r}}}\right).
\end{align*}
The last equality is obtained by choosing \begin{equation}
C=\sqrt{r\log(1/h)},
\label{eq:lemma3pc}
\end{equation}
so that $e^{-C^2}=h^r<h^{r/2}$ when $h\in (0,1)$. 

The third term in \eqref{eq:lemma3p1} can be bounded by using similarly strategy, with an additional estimate on $\frac{N(\xb_{n+1})}{D(\xb_{n+1})}$. We define the ball $B_h(\xb_{n+1})$ and $B_{\tilde h}(\xb_{n+1})$ as in \eqref{eq:lemma3pball} with $\tilde h = Ch$. Then,
\begin{align*}
\frac{|N(\xb_{n+1})|}{|D(\xb_{n+1})|}
& =\frac{\left|\int_{\Omega}K_h(U^{\top}\xb_{n+1}-U^{\top}\xb)f(\xb)d\xb\right|}{\left|\int_{\Omega}K_h(U^{\top}\xb_{n+1}-U^{\top}\xb)d\xb\right|}
\\
&\le \frac{\left|\int_{B_{\tilde h}(\xb_{n+1})}K_h(U^{\top}\xb_{n+1}-U^{\top}\xb)f(\xb)d\xb\right|+\left|\int_{\Omega\setminus B_{\tilde h}(\xb_{n+1})}K_h(U^{\top}\xb_{n+1}-U^{\top}\xb)f(\xb)d\xb\right|}{\left|\int_{B_h(\xb_{n+1})}K_h(U^{\top}\xb_{n+1}-U^{\top}\xb)d\xb\right|}
\\
& \le  \frac{R\rho_{\xb}(B_{\tilde h}(\xb_{n+1}))+Re^{-C^2}\rho_{\xb}(\Omega)}{e^{-1}\rho_{\xb}(B_h(\xb_{n+1}))} \le O(RC^r),
\end{align*}
where the last inequality holds with $C$ chosen according to \eqref{eq:lemma3pc} and when $h$ is sufficiently small.
Applying a similar argument above, the second term in \eqref{eq:lemma3p1} can be bounded by 
\begin{align*}
\frac{|N(\xb_{n+1})|}{|\hat{D}_n(\xb_{n+1})|\cdot |D(\xb_{n+1})|} \left(|\hat{D}_n(\xb_{n+1}) - D(\xb_{n+1})| \right) \le O\left(\log^{3r/4}\left(h^{-1}\right)\sqrt{\frac{\log(2/\delta)}{nh^{r}}}\right).
\end{align*}

Combining all the three terms in \eqref{eq:lemma3p1} together, we get with probability at least $1-3\delta$, it satisfies
\begin{align*}
    |\cK_h(\fs_f)-\bar{\cK}_h(f;\xb_{n+1})|&\leq O\left(\sqrt{\frac{\log(2/\delta)}{n}}\right)+O\left(\left[\log\left(h^{-1}\right)\right]^{\frac{r}{4}}\sqrt{\frac{\log(2/\delta)}{nh^{r}}}\right)+O\left([\log(h^{-1})]^{3r/4}\sqrt{\frac{\log(2/\delta)}{nh^{r}}}\right) \\
    &=O\left([\log(h^{-1})]^{3r/4}\sqrt{\frac{\log(2/\delta)}{nh^{r}}}\right).
\end{align*}
The notation $O(\cdot)$ hides the dependency on $r,b,D,R,C_{\rho,1}$. 

By abusing notation, rewrite $3\delta$ as $\delta$, we get with probability at least $1-\delta$,
\begin{equation*}
    |\cK_h(\fs_f)-\bar{\cK}_h(f;\xb_{n+1})|\leq O\left([\log(h^{-1})]^{3r/4}\sqrt{\frac{\log(6/\delta)}{nh^{r}}}\right).
\end{equation*}
Therefore,
\begin{equation*}
    {\rm A1}\leq O\left(\left[\log\left(h^{-1}\right)\right]^{3r/2}{\frac{\log(6/\delta)}{nh^{r}}}\right)
\end{equation*}
with probability at least $1-\delta$. 
\vspace{5mm}
\\
\noindent
$\bullet$ {\bf Combining ${\rm A1}$ and ${\rm A2}$} 
\\
\noindent 
Putting ${\rm A1}$ and ${\rm A2}$ together and take the expectation, we get
\begin{align*}
    \EE_{\fs_f}\left[\cR_{\fs_f}(\cK)\right]
    &\leq (1-\delta)\cdot O\left(\left[\log\left(h^{-1}\right)\right]^{3r/2}{\frac{\log(6/\delta)}{nh^{r}}}\right)+\delta\cdot(2R)^2+O(h^{2\alpha}[\log(h^{-1})]^2) \\
    & \overset{\text{let}\text{ }\delta=6h^2}{=\joinrel=\joinrel=\joinrel=\joinrel=\joinrel=}
    (1-6h^2)\cdot O\left(\left[\log\left(h^{-1}\right)\right]^{3r/2}{\frac{\log(h^{-1})}{nh^{r}}}\right)+24h^2R^2+O(h^{2\alpha}[\log(h^{-1})]^2) \\
    &\leq O\left(\frac{\left[\log\left(h^{-1}\right)\right]^{1+3r/2}}{nh^{r}}\right)+O(h^{2\alpha}[\log(h^{-1})]^2).
\end{align*}
The last inequality is due to $h\in (0,1)$. The notation $O(\cdot)$ hides the constants dependency on $r,b,L,R,D,C_{\rho,1},C_{\rho,2}$.

\end{proof}

\section{Other Definitions and Lemmas} \label{sec:Appendixfundlemma}

We state the fundamental lemmas being used in this paper, with their proofs provided in \citet{shen2025understanding}.

\begin{lemma} [Interaction Lemma] \label{lemmaIA}
   Let $H=[h_t]_{1\leq t\leq\ell}\in\mathbb{R}^{d_{embed}\times\ell}$ be an embedding matrix such that $h_t^{(d_{embed}-2):(d_{embed}-1)}=\mathcal{I}_t$ and $h^{d_{embed}}_t=1$. Fix $1\leq t_1,t_2\leq \ell, 1\leq i\leq d_{embed}$, and $\ell\in\mathbb{N}$. Suppose $d_{embed}\geq 5$ and $\|H\|_{\infty,\infty}<U$ for some $U>0$, and the data kernels $Q^{data} \in \RR^{(d_{embed}-3) \times d_{embed}}$ (the first $(d_{embed}-3)$ rows in the query matrix $Q$) and $K^{data}\in \RR^{(d_{embed}-3) \times d_{embed}}$ (the first $(d_{embed}-3)$ rows in the key matrix $K$) satisfy $\max\{\|Q^{data}\|_{\infty,\infty},\|K^{data}\|_{\infty,\infty}\}\leq\kappa$. Then one can construct an attention head $A$ with ReLU activation ($\sigma={\rm ReLU}$)  such that 
   \[ [A(H)]_t = 
   \begin{cases} 
    \sigma(\langle Q^{data}h_t, K^{data}h_{t_2}\rangle)\eb_i & \text{if}\text{  } t=t_1, \\
      0 & \text{otherwise}. \\
   \end{cases}
\]
The weight parameters of this attention head satisfies $\|\theta_A\|_{\infty}=O(d_{embed}^4\kappa^2\ell^2U^2)$.
\end{lemma}

\begin{lemma} [Gating Lemma] \label{lemmagate}
Let $d_{embed}\geq 5$ and $H=[h_t]_{1\leq t\leq\ell}\in\mathbb{R}^{d_{embed}\times\ell}$, be an embedding matrix such that $h_t^{(d_{embed}-2):(d_{embed}-1)}=(\mathcal{I}_t^1,\mathcal{I}_t^2)=\mathcal{I}_t$ and $h^{d_{embed}}_t=1$. Then for any $r_1$ and $r_2$ with $1\leq r_1\leq r_2\leq d_{embed}-3$ and any $k_1,k_2$ with $1\leq k_1,k_2\leq \ell$, there exist both two-layer feed-forward networks $({\rm FFN})$ such that
\begin{equation} \label{eq:gate1}
     \text{\rm FFN}_1(h_t) = 
   \begin{cases} 
    h_t & \text{if}\text{  } t\in\{1,\cdots,k_1\} \\
      \begin{bmatrix}

      (h_t)_{1} \\
          \vdots \\
          (h_t)_{r_1-1} \\ 
          \mathbf{0} \\
          (h_t)_{r_2+1} \\
          \vdots \\
          (h_t)_{d_{embed}-3} \\ 
          \mathcal{I}_t^1 \\
          \mathcal{I}_t^2 \\
          1
      \end{bmatrix} & \text{otherwise}\text{ }  \\
   \end{cases}
\end{equation}
and
\begin{equation} \label{eq:gate2}
    \text{\rm FFN}_2(h_t) = 
   \begin{cases} 
 h_t & \text{if}\text{  } t\in\{k_2,\cdots,\ell\} \\
      \begin{bmatrix}
      (h_t)_{1} \\
          \vdots \\
          (h_t)_{r_1-1} \\
          \mathbf{0} \\
          (h_t)_{r_2+1} \\
          \vdots \\
          (h_t)_{d_{embed}-3} \\
          \mathcal{I}_t^1 \\
          \mathcal{I}_t^2 \\
          1
      \end{bmatrix} & \text{otherwise}\text{ }  \\
   \end{cases}
\end{equation}
Additionally, we have $\|\theta_{\rm FFN}\|_{\infty}\leq O(\ell\|H\|_{\infty,\infty})$.
\end{lemma}

\begin{lemma} [Decrementing Lemma] \label{lemmadecrement}
Let $d_{embed}\geq 5$ and $H=[h_t]_{1\leq t\leq\ell}\in\mathbb{R}^{d_{embed}\times\ell}$, be an embedding matrix such that $h_t^{(d_{embed}-2):(d_{embed}-1)}=(\mathcal{I}_t^1,\mathcal{I}_t^2)=\mathcal{I}_t$ and $h^{d_{embed}}_t=1$. Then for any $r_1,r_2$ with  $1\leq r_1\leq r_2\leq d_{embed}-3$ and any $k_1,k_2$ with $1\leq k_1,k_2\leq \ell$ and any $M>0$, there exists a six-layer residual feed-forward network $({\rm FFN})$ such that
\[ \text{\rm FFN}(h_t) + h_t = 
   \begin{cases} 
    h_t & \text{if}\text{  } t\in\{1,\cdots,k_1\}\cup\{k_2,\cdots,\ell\} \\
      \begin{bmatrix}
      (h_t)_{1} \\
          \vdots \\
          (h_t)_{r_1-1} \\
          (h_t)_{r_1}-M \\
          \vdots \\
          (h_t)_{r_2}-M \\
          (h_t)_{r_2+1} \\
          \vdots \\
          (h_t)_{d_{embed}-3} \\
          \mathcal{I}_t \\
          1
      \end{bmatrix} & \text{otherwise}\text{ }  \\
   \end{cases}
\]
Additionally, we have $\|\theta_{\rm FFN}\|_{\infty}\leq O(\ell M)$.
\end{lemma}

\begin{lemma}[Dudley Entropy Integral (c.f. \citep{dudley1967sizes})] \label{lemmaRAD}
Let $\{\xi_i\}_{i=1}^n\in\{-1,+1\}$ be i.i.d Rademacher random variables. Suppose $\sup_{f\in\mathcal{F}}\|f\|_{L^\infty}\leq M$ for some $M>0$ and some function class $\mathcal{F}$. Then
\begin{equation}
    \EE_{\{x_i\}_{i=1}^n,\{\xi_i\}_{i=1}^n}\left[\sup_{f\in\mathcal{F}}\frac{1}{n}\sum_{i=1}^n \xi_if(x_i)\right]\leq \inf_{\epsilon>0} \left(2\epsilon+\frac{12}{\sqrt{n}}\int_{\epsilon}^{M}\sqrt{\log\mathcal{N}(\delta,\mathcal{F},\|\cdot\|_{L^\infty})}\text{ }d\delta\right).
\end{equation}
where $\mathcal{N}(\delta,\mathcal{T},\|\cdot\|_{L^\infty})$ is the $\delta$-covering number of $\mathcal{F}$ with respect to $L^{\infty}$ norm.  
\end{lemma}

\begin{definition}
   Let $\{\epsilon_i\}_{i=1}^n$ be i.i.d Gaussian random variables. The Gaussian complexity of a function class $\cH$ over a set $Z$ is defined as
\begin{equation}
\cG_{Z}(\cH):=\frac{1}{n} \EE_{\{\epsilon_i\}_{i=1}^n}\sup_{h\in\cH}\left[ \sum_{i=1}^n \epsilon_ih(z_i)\right].
\end{equation} 
\end{definition}

Our next lemma bounds the Gaussian complexity by Rademacher complexity.

\begin{lemma} \label{lemma:Gaussian_complexity}
Let $\cH$ be a hypothesis space and $Z=\{z_1,\cdots,z_m\}$.  Then
\begin{equation}
     \cG_{Z}(\cH)\leq 2\sqrt{\log(m)}\Rad_{Z}(\cH)
\end{equation}
\end{lemma}

\begin{proof}
Suppose $\bepsilon=(\epsilon_1,\cdots,\epsilon_m)$ where $\epsilon_i$ are i.i.d drawn from standard Gaussian distribution, and $\bxi=(\xi_1,\cdots,\xi_m)$ where $\xi_i$ are i.i.d drawn from Rademacher distribution. Then
    \begin{align*}
        \cG_{Z}(\cH):&=\frac{1}{m} \EE_{\bepsilon} \left[\sup_{h\in\cH} \sum_{i=1}^m \epsilon_i h(z_i)\right]=
        \frac{1}{m} \EE_{\bxi,\bepsilon} \left[\sup_{h\in\cH} \sum_{i=1}^m |\epsilon_i|\xi_i h(z_i)\right] \\
        & =  \frac{1}{m}\EE_{\bepsilon} \left[\EE_{\bxi} \left[\sup_{h\in\cH} \sum_{i=1}^m |\epsilon_i|\xi_i h(z_i)\text{ }\vline\text{ }\bepsilon\right]\right]
        \leq  \frac{1}{m} \EE_{\bepsilon}\|\bepsilon\|_{\infty}\EE_{\bxi} \left[\sup_{h\in\cH} \sum_{i=1}^m \xi_i h(z_i)\right] \\
        &\leq \frac{1}{m}\cdot 2\sqrt{\log m}\cdot \EE_{\bxi} \left[\sup_{h\in\cH} \sum_{i=1}^m \xi_i h(z_i)\right]=2\sqrt{\log m}\cdot\Rad_Z(\cH).
    \end{align*}
    where we have used the fact that, for large $m$, $\EE_{\bepsilon}\|\bepsilon\|_{\infty}\leq 2\sqrt{\log m}$ in the last inequality.
\end{proof}

\begin{lemma}[Davis-Kahan Theorem \citep{davis1970rotation,yu2015useful}] \label{lemma:DavisKahan}
    Let $A, \tilde{A} \in \mathbb{R}^{n \times n}$ be symmetric matrices with spectral decompositions
\[
A = \sum_{i=1}^n \lambda_i u_i u_i^\top, \qquad \tilde{A} = \sum_{i=1}^n \tilde\lambda_i \tilde u_i \tilde u_i^\top,
\]
where $\lambda_1 \geq \lambda_2 \geq \cdots \geq \lambda_n$ are the eigenvalues of $A$ with orthonormal eigenvectors $u_1, \dots, u_n$, and $\tilde\lambda_1 \geq \cdots \geq \tilde\lambda_n$ are the eigenvalues of $\tilde A$ with orthonormal eigenvectors $\tilde u_1, \dots, \tilde u_n$. Let
\[
E = \tilde{A} - A
\]
denote the perturbation. Fix an integer $1 \le k \le n$, and let
\[
U_1 = [u_1, \dots, u_k] \in \mathbb{R}^{n\times k}, \qquad \tilde U_1 = [\tilde u_1, \dots, \tilde u_k] \in \mathbb{R}^{n \times k}
\]
be the matrices whose columns span, respectively, the top-$k$ eigenspace of $A$ (associated with eigenvalues $\lambda_1, \dots, \lambda_k$) and the top-$k$ eigenspace of $\tilde A$ (associated with eigenvalues $\tilde\lambda_1, \dots, \tilde\lambda_k$). Suppose $\lambda_k-\tilde{\lambda}_{k+1}>0$. Then
\[
\left\|  U_1U_1^{\top}-\tilde{U}_1\tilde{U}_1^{\top}\right\|_{\mathrm{op}} \leq \frac{\| E \|_{\mathrm{op}}}{\lambda_k-\tilde{\lambda}_{k+1}}.
\]
\end{lemma}

\begin{lemma}[Matrix Bernstein \citep{tropp2015matrix}] \label{lemma:bernstein}

Let $S_1, \ldots, S_n$ be independent, centered random matrices with common dimension $d_1 \times d_2$, and assume that:
\[
\mathbb{E} [S_k]=0 \text { and }\left\|S_k\right\|_2 \leq L \text { for each } k=1, \ldots, n .
\]
Introduce the random matrix
\[
Y=\sum_{k=1}^n S_k
\]
and let $v(Y)$ denote the matrix variance statistic of the sum:
\[
\begin{aligned}
v(Y) & =\max \left\{\left\|\mathbb{E}\left[Y Y^*\right]\right\|_2,\left\|\mathbb{E}\left[Y^* Y\right]\right\|_2\right\} \\
& =\max \left\{\left\|\sum_{k=1}^n \mathbb{E}\left[S_k S_k^*\right]\right\|_2,\left\|\sum_{k=1}^n \mathbb{E}\left[S_k^* S_k\right]\right\|_2\right\} .
\end{aligned}
\]
Then for all $t\geq 0$,
\[
\mathbb{P}\{\|Y\|_2 \geq t\} \leq\left(d_1+d_2\right) \cdot \exp \left(\frac{-t^2 / 2}{v(Y)+L t / 3}\right).
\]
Furthermore, 
\[
\mathbb{E}\|Y\|_2 \leq \sqrt{2 v(Y) \log \left(d_1+d_2\right)}+\frac{1}{3} L \log \left(d_1+d_2\right).
\]
\end{lemma}

\begin{lemma}[Intrinsic Matrix Bernstein \citep{tropp2015matrix}] \label{lemma:intrinsic_bernstein}

Let $S_1, \ldots, S_n$ be independent, centered random matrices with common dimension, and assume that:
\[
\mathbb{E} [S_k]=0 \text { and }\left\|S_k\right\|_2 \leq L \text { for each } k=1, \ldots, n .
\]
Introduce the random matrix
\[
Y=\sum_{k=1}^n S_k
\]
Let $V_1,V_2$ be semidefinite upper bounds for the matrix-valued variances ${\rm Var}_1(Y)$ and ${\rm Var}_2(Y)$:
\begin{equation*}
V_1\succcurlyeq {\rm Var}_1(Y)=\EE[YY^*]=\sum_{k=1}^n[S_k S_k^*]
    \quad\text{and}\quad
     V_2\succcurlyeq {\rm Var}_2(Y)=\EE[Y^*Y]=\sum_{k=1}^n[S_k^*S_k].     
\end{equation*}
Define an intrinsic dimension bound and a variance bound
\begin{equation*}
    d={\rm intdim}\begin{bmatrix}
        V_1 & 0 \\
        0 & V_2
    \end{bmatrix}
    :=\frac{{\rm tr}(V_1)+{\rm tr}(V_2)}{v} \quad\text{and}\quad v:=\max\{\|V_1\|_2, \|V_2\|_2\}.
\end{equation*}
Then, for all $t\geq \sqrt{v}+L/3$,
\[
\mathbb{P}\{\|Y\|_2 \geq t\} \leq 4d\cdot \exp \left(\frac{-t^2 / 2}{v+L t / 3}\right).
\]
\end{lemma}

The next lemma states that, in the symmetric case, the perturbation of pseudo-inverse can be bounded by the perturbation of the original matrix.

\begin{lemma} \label{lemma:pesudo-inverse}
Let $A,B\in\mathbb{R}^{r\times r}$ be rank $k<r$ symmetric matrices. Then
\begin{equation*}
    \|B^{\dagger}-A^{\dagger}\|_2\leq 3\max\{\|A^{\dagger}\|^2_2,\|B^{\dagger}\|_2^2\}\|B-A\|_2
\end{equation*}
\end{lemma}

\begin{proof}
     Write $E:=B-A$, $\alpha=\lambda_k(A), \beta=\lambda_k(B)$. Let $P_A = AA^{\dagger},P_B=BB^{\dagger}$. Notice that since $A,B$ are symmetric, we have $AA^{\dagger}=A^{\dagger}A$ and $BB^{\dagger}=B^{\dagger}B$. Next, we make use of the identity
     \begin{equation} \label{pesudo_identity}
         B^{\dagger}-A^{\dagger}=-B^{\dagger}EA^{\dagger}+(B^{\dagger})^2E(I-P_A)+(I-P_B)E(A^{\dagger})^2,
     \end{equation}
     which can be verified in the following way:
     
     The first term on the right-hand-side equals to 
     \begin{equation*}
         -B^{\dagger}EA^{\dagger}=-B^{\dagger}(B-A)A^{\dagger}=-B^{\dagger}BA^{\dagger}+B^{\dagger}AA^{\dagger}=-P_BA^{\dagger}+B^{\dagger}P_A.
     \end{equation*}
     Notice that $A(I-P_A)=A-AA^{\dagger}A=0$. The second term on the right-hand-side equals to 
     \begin{equation*}
         (B^{\dagger})^2E(I-P_A)=(B^{\dagger})^2(B-A)(I-P_A)=(B^{\dagger})^2B(I-P_A)=B^{\dagger}(I-P_A).
     \end{equation*}
     Similarly, the third term on the right-hand-side equals to
     \begin{equation*}
         (I-P_B)E(A^{\dagger})^2=(I-P_B)(B-A)(A^{\dagger})^2=-(I-P_B)A(A^{\dagger})^2=-(I-P_B)A^{\dagger}.
     \end{equation*}
     Adding all three terms together gives $B^{\dagger}-A^{\dagger}$.
     
     Taking the operator norm on both sides of \eqref{pesudo_identity}, and by using the fact $\|I-P_A\|_2=\|I-P_B\|_2=1$, it gives
     \begin{align*}
         \|B^{\dagger}-A^{\dagger}\|_2&\leq \|B^{\dagger}\|_2\|E\|_2\|A^{\dagger}\|_2+\|B^{\dagger}\|^2_2\|E\|_2+\|E\|_2\|A^{\dagger}\|^2_2 \\
         &=(\|B^{\dagger}\|_2\|A^{\dagger}\|_2+\|B^{\dagger}\|^2_2+\|A^{\dagger}\|^2_2)\|E\|_2 \\
         &\leq 3\max\{\|A^{\dagger}\|^2_2,\|B^{\dagger}\|_2^2\}\|B-A\|_2
     \end{align*}
as desired.     
\end{proof}

\end{document}